\documentclass[sigconf]{aamas}

\usepackage{balance} 

\usepackage[utf8]{inputenc}
\usepackage{amsmath}
\usepackage{amsthm}
\usepackage{amsfonts}
\usepackage{mathtools}
\usepackage{hyperref}
\usepackage{verbatim}
\usepackage{caption}
\usepackage{subcaption}
\usepackage{bm}
\usepackage[bb=dsserif]{mathalpha} 
\usepackage [mathscr] {eucal}
\usepackage{float}
\usepackage{xcolor}
\usepackage[colorinlistoftodos]{todonotes}  
\usepackage{color}                          
\usepackage{totcount}                       
\usepackage{tikz}
\usepackage{tikz-3dplot}
\usetikzlibrary{calc,shapes,positioning}
\usetikzlibrary{intersections}
\usetikzlibrary{patterns} 
\usetikzlibrary{patterns.meta}
\usepackage{algorithm}
\usepackage{algpseudocode}

\newcommand{\pill}[1]{\textbf{#1}}
\definecolor{Red}{RGB}{210, 0 , 30}
\definecolor{Blue}{RGB}{0, 30, 210}
\definecolor{Green}{RGB}{15, 210, 15}
\definecolor{Cyan}{RGB}{0, 190, 180}
\definecolor{Beryl}{RGB}{80, 170, 200}
\definecolor{Cerulean}{RGB}{42, 82, 190}
\definecolor{Azure}{RGB}{0, 127, 255}
\definecolor{Indigo}{RGB}{70, 0, 130}
\definecolor{Navy}{RGB}{0, 0, 128}
\definecolor{Sapphire}{RGB}{15, 82, 186}
\definecolor{Turquoise}{RGB}{64, 200, 210}
\definecolor{Ultramarine}{RGB}{20, 10, 150}
\definecolor{Bordeaux}{RGB}{120, 0, 5}

\reversemarginpar
\newtotcounter{notecount}
\newcommand{\notewarning}{%
\ifnum\totvalue{notecount}>0%
 \vspace{1ex}
\begin{center}
 \begin{tikzpicture}[baseline=(A.south)]
    \node (A) [] at (0,0){};
    \node [rounded corners=1pt,rectangle, draw=red, fill=red!20,text=black](B) at (0.1ex,0ex){
        \Large \raggedright {\bf Warning:} There are still some notes left!
    };
 \end{tikzpicture}
\end{center}
 \vspace{1ex}
\fi
}
\makeatletter
\def\myaddcontentsline#1#2#3{%
  \addtocontents{#1}{\protect\contentsline{#2}{#3}{Section \thesubsection\ at p. \thepage}{}}}
\renewcommand{\@todonotes@addElementToListOfTodos}{%
    \if@todonotes@colorinlistoftodos%
        \myaddcontentsline{tdo}{todo}{{%
            \colorbox{\@todonotes@currentbackgroundcolor}%
                {\textcolor{\@todonotes@currentbackgroundcolor}{o}}%
            \ \@todonotes@caption}}%
    \else%
        \myaddcontentsline{tdo}{todo}{{\@todonotes@caption}}%
   \fi}%
\newcommand*\mylistoftodos{%
  \begingroup
       \setbox\@tempboxa\hbox{Section 9.9 at p. 99}%
       \renewcommand*\@tocrmarg{\the\wd\@tempboxa}%
       \renewcommand*\@pnumwidth{\the\wd\@tempboxa}%
       \listoftodos%
  \endgroup
}
\makeatother
\definecolor{lightgreen}{rgb}{0.86, 0.93, 0.78}
\definecolor{bordergreen}{rgb}{0.55, 0.76, 0.74}
\definecolor{lightblue}{rgb}{0.70, 0.90, 0.99}
\definecolor{borderblue}{rgb}{0.01, 0.66, 0.96}
\definecolor{lightamber}{rgb}{1, 0.93, 0.70}
\definecolor{borderamber}{rgb}{1, 0.76, 0.03}

\DeclareMathOperator*{\argmin}{\arg\!\min}

\newcommand\restr[2]{{
  \left.\kern-\nulldelimiterspace 
  #1 
  \littletaller 
  \right|_{#2} 
  }}

\newcommand{\littletaller}{\mathchoice{\vphantom{\big|}}{}{}{}}

\newtheorem{definition}{Definition}
\newtheorem{theorem}{Theorem}
\newtheorem{lemma}{Lemma}

\newtheorem{axiom}{Axiom}
\newtheorem*{example}{Example}

\newtheorem{rthm}{Theorem}
\newenvironment{reptheorem}[1]
  {\renewcommand\therthm{\ref*{#1}}\rthm}
  {\endrthm}

\doi{MJWJ6521}

\makeatletter
\gdef\@copyrightpermission{
  \begin{minipage}{0.2\columnwidth}
   \href{https://creativecommons.org/licenses/by/4.0/}{\includegraphics[width=0.90\textwidth]{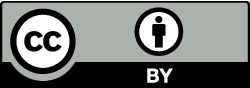}}
  \end{minipage}\hfill
  \begin{minipage}{0.8\columnwidth}
   \href{https://creativecommons.org/licenses/by/4.0/}{This work is licensed under a Creative Commons Attribution International 4.0 License.}
  \end{minipage}
  \vspace{5pt}
}
\makeatother

\setcopyright{ifaamas}
\acmConference[AAMAS '26]{Proc.\@ of the 25th International Conference
on Autonomous Agents and Multiagent Systems (AAMAS 2026)}{May 25 -- 29, 2026}
{Paphos, Cyprus}{C.~Amato, L.~Dennis, V.~Mascardi, J.~Thangarajah (eds.)}
\copyrightyear{2026}
\acmYear{2026}
\acmDOI{}
\acmPrice{}
\acmISBN{}

\acmSubmissionID{1240}

\title{Clone-Robust Weights in Metric Spaces: Handling Redundancy Bias in Benchmark Aggregation}

\author{Damien Berriaud}
\affiliation{
  \institution{ETHZ}
  \city{Z\"urich}
  \country{Switzerland}}
\email{dberriaud@ethz.ch}

\author{Roger Wattenhofer}
\affiliation{
  \institution{ETHZ }
  \city{Z\"urich}
  \country{Switzerland}}
\email{wattenhofer@ethz.ch}

\begin{abstract}
    We are given a set of elements in a metric space. The distribution of the elements is arbitrary, possibly adversarial. Can we weigh the elements in a way that is resistant to such (adversarial) manipulations?
    This problem arises in various contexts. For instance, the elements could represent data points, requiring robust domain adaptation. Alternatively, they might represent tasks to be aggregated into a benchmark; or questions about personal political opinions in voting advice applications.
    This article introduces a theoretical framework for dealing with such problems. We propose \textit{clone-proof weighting functions} as a solution concept. These functions distribute importance across elements of a set such that similar objects (``clones'') share (some of) their weights, thus avoiding a potential bias introduced by their multiplicity.
    Our framework extends the maximum uncertainty principle to accommodate general metric spaces and includes a set of axioms---symmetry, continuity, and clone-proofness---that guide the construction of weighting functions.
    Finally, we address the existence of weighting functions satisfying our axioms in the significant case of Euclidean spaces and propose a general method for their construction.

\end{abstract}


\newcommand{\BibTeX}{\rm B\kern-.05em{\sc i\kern-.025em b}\kern-.08em\TeX}

\begin{document}


\pagestyle{fancy}
\fancyhead{}


\maketitle 




\section{Introduction}

Morpheus: ``You take the \textcolor{Blue}{\pill{blue}} pill and the story ends. You wake up in your bed and believe whatever you want to believe. 
You take the \textcolor{Red}{\pill{red}} pill, you stay in Wonderland and I show you how deep the rabbit hole goes.''
Before Neo can make his choice, Morpheus continues: ``Or you can take this \textcolor{Indigo}{\pill{indigo}} pill, and wake up in Wonderland with \$100 in your pocket. Or this \textcolor{Navy}{\pill{navy}} pills---with a different hair color.''
Why would Morpheus present these insignificant shades of blue? 
Neo already feels manipulated, but Morpheus continues excitedly and adds 
more pills colored in \textcolor{Bordeaux}{\pill{bordeaux}}, \textcolor{Cyan}{\pill{cyan}}, and \textcolor{Green}{\pill{green}}.

\begin{figure}[!tbh]
    \centering
    \begin{tikzpicture}[scale=7]
    \definecolor{PaperBlue}{HTML}{1f77b4}
    \definecolor{PaperOrange}{HTML}{ff7f0e}
    \definecolor{PaperGreen}{HTML}{2ca02c}
    \definecolor{PaperRed}{HTML}{d62728}
    \definecolor{PaperGrey}{HTML}{7f7f7f}
    \definecolor{PaperYellow}{HTML}{ffdd00}
    
    \coordinate (R) at (1, 0); 
    \coordinate (G) at (0, 0); 
    \coordinate (B) at (0.5, {sqrt(3)/2}); 
    
    \fill[PaperOrange, opacity=0.05] (R) -- (G) -- (B) -- cycle;
    
    \begin{scope}
        \clip (R) -- (G) -- (B) -- cycle; 
        \fill[PaperRed, opacity=0.1] (0.7, 0) arc[start angle=180, end angle=120, radius=0.3];
    \end{scope}
    \fill[PaperRed, opacity=0.1] (1, 0) -- (0.85, 0.2598) -- (0.7, 0) -- cycle;

    \begin{scope}
        \clip (R) -- (G) -- (B) -- cycle; 
        \fill[PaperGreen, opacity=0.1] (0.3, 0) arc[start angle=0, end angle=60, radius=0.3];
    \end{scope}
    \fill[PaperGreen, opacity=0.1] (0, 0) -- (0.15, 0.2598) -- (0.3, 0) -- cycle;
    
    \begin{scope}
        \clip (R) -- (G) -- (B) -- cycle; 
        \fill[PaperBlue, opacity=0.1] (0.5 - 0.15, {sqrt(3)/2 - 0.259}) arc[start angle=-120, end angle=-60, radius=0.3];
    \end{scope}
    \fill[PaperBlue, opacity=0.1] (0.5, {sqrt(3)/2}) -- (0.35, {sqrt(3)/2 - 0.259}) -- (0.65, {sqrt(3)/2 - 0.259}) -- cycle;
    
    \draw[dashed, PaperRed] (0.7, 0) arc[start angle=180, end angle=120, radius=0.3];
    \draw[dashed, PaperGreen] (0.3, 0) arc[start angle=0, end angle=60, radius=0.3];
    \draw[dashed, PaperBlue] (0.5 - 0.15, {sqrt(3)/2 - 0.259}) arc[start angle=-120, end angle=-60, radius=0.3];
    
    \newcommand{\plotcolor}[4]{
        \pgfmathsetmacro{\Sum}{#2 + #3 + #4}
        \pgfmathsetmacro{\RNorm}{#2 / \Sum}
        \pgfmathsetmacro{\GNorm}{#3 / \Sum}
        \pgfmathsetmacro{\BNorm}{#4 / \Sum}
        \pgfmathsetmacro{\x}{\RNorm * 1 + \GNorm * 0 + \BNorm * 0.5}
        \pgfmathsetmacro{\y}{\RNorm * 0 + \GNorm * 0 + \BNorm * sqrt(3)/2}
        \fill[#1] (\x, \y) circle (0.01) node[below right] {#1};
    }
    
    \plotcolor{Red}{210}{0}{30}
    \plotcolor{Bordeaux}{120}{10}{3}
    \plotcolor{Blue}{10}{30}{210}
    \plotcolor{Cyan}{0}{190}{180}
    \plotcolor{Indigo}{70}{0}{170}
    \plotcolor{Navy}{0}{0}{128}
    \plotcolor{Green}{5}{120}{3}
    
    \end{tikzpicture}

    \caption{Which weight should we give to each individual point?
    By symmetry, one would expect  
    the areas in \textcolor{Blue}{blue}, \textcolor{Red}{red} and \textcolor{Green}{green} to sum up to the same value, even though they contain different numbers of points. 
    How to deal with the addition of \textcolor{Cyan}{cyan} though?
    }
    \Description{The equilateral triangle represents the 2-simplex of normalized triples (R,G,B), where each point corresponds to a color obtained after normalization. The vertices R, G and B represent pure red, green, and blue, respectively.
    The shaded circular sectors at the vertices indicate symmetric neighborhoods of equal angular aperture, highlighting regions that are geometrically equivalent by symmetry. Although these regions contain different numbers of sampled colors, symmetry suggests they should contribute equal total weight.
    Interior points (e.g., cyan, indigo, bordeaux, navy) correspond to mixed colors positioned according to their barycentric coordinates. The figure illustrates the tension between geometric symmetry of the simplex and the non-uniform distribution of discrete samples within it, raising the question of how weights should be assigned, particularly when additional colors such as cyan lie between symmetric sectors.}
    \label{fig:intro_RGB_triangle}
\end{figure}

Is there an objective way to make a choice without being bamboozled 
(see Figure~\ref{fig:intro_RGB_triangle})?
This problem is at the heart of machine learning, but it can also be applied to various other areas, e.g., distributed systems or social choice. 
Let us set aside these related use cases and concentrate on our primary application: multi-task benchmark aggregation.

Consider a multi-task benchmark $\mathcal{B}$, e.g., GLUE~\cite{wang_glue_2019}.
The evaluation process of such a benchmark typically unfolds in two stages:
First, a set of benchmark tasks $T = \{T_i\}_{i \in [n]}$ is defined, where each task $T_i$ maps a model $m \in \mathcal{M}$ to a score $T_i(m) \in \mathbb{R}$;
Then, an aggregation rule $A$ combines the task-wise scores $(T_i(m))_{i \in [n]}$ into a single score $A(T(m)) \in \mathbb{R}.$ 

As discussed in prior works~\cite{colombo_what_2022,rofin_votenrank_2023,zhang_inherent_2024}, the choice of an aggregation rule $A$ is perhaps best understood through the lens of social choice theory, where each task acts as a voter. 
However, multi-task benchmarking diverges from traditional voting scenarios in a key way: \emph{anonymity}, or the equal treatment of votes, is not inherently required.
In fact, it is well recognized that some tasks may exhibit significant similarity, such as CoLA~\cite{warstadt_neural_2019} and SST-2~\cite{socher_recursive_2013} in the case of GLUE. 
Perhaps tasks, similarly to colors in Figure~\ref{fig:intro_RGB_triangle}, should share some of their weight? 
Indeed, the benchmark's outcome ought to remain unaffected by the inclusion of numerous highly similar tasks, as this could unfairly favor models that perform well on the original task over those excelling in other areas.

Although weighted aggregation rules have been proposed in the literature~\cite{rofin_votenrank_2023}, the weights are typically chosen arbitrarily by the benchmark designer. In contrast, we propose a principled method for determining tasks' weights. 
Specifically, we suggest that the designer: (i) settles on a relevant measure of similarity between tasks, and embeds them in a metric space $(E, d)$; and (ii) calculates the tasks' weights using a weighting function with desirable axiomatic properties.
In this article, we propose a set of axioms and weighting functions designed to be robust to noise and approximate clones, ensuring practical applicability in real-world settings. Using these weighting functions to assign task weights enables automatic scaling of the benchmark, i.e., the evaluation always benefits from adding new tasks, even if they are somewhat similar to previous tasks.

\paragraph{Contribution \& Outline}

This work proposes a mathematical framework for handling redundancy in a metric space. 
Specifically, we tackle the problem of determining the relative importance of elements in a finite set such that close-by elements, or ``clones,'' share some of their weight.
To extend the well-understood case of discrete metrics, where elements are either similar or equally dissimilar, we introduce in Section~\ref{sec:axioms} the concept of weighting function and propose a set of axioms that such functions should satisfy in general metric spaces. 
These properties can be broadly categorized into three key principles: symmetry, continuity, and clone-proofness.
Building on these foundations, we address in Section~\ref{sec:local_voting} the challenges of constructing functions that adhere to these axioms.
For the specific case of Euclidean spaces, we resolve the question of existence and construct, in Theorems~\ref{thm:local_vote_rep_func} and~\ref{thm:cont_convex_combi_gr}, a family of desirable weighting functions.
We discuss in Section~\ref{sec:computational_cons} the computational hurdles associated with their exact evaluation, and introduce Monte Carlo algorithms for their practical approximation.
We finally explore in Section~\ref{sec:discu} possible extensions of our construction to more general spaces.

\section{Related Work}

In a recent line of work \cite{colombo_what_2022,himmi_towards_2023,rofin_votenrank_2023,tatiana_how_2021,zhang_inherent_2024}, multi-tasks benchmarking practices have been scrutinized through the lenses of social choice theory. In particular, these works question the usage of the arithmetical mean to aggregate scores of different tasks in popular benchmarks \cite{koh_wilds_2021,wang_superglue_2020}
and investigate different aggregates such as the Pythagorean means \cite{tatiana_how_2021}, the Bradley-Terry model \cite{peyrard_better_2021}, or classical voting rules \cite{colombo_what_2022,himmi_towards_2023,rofin_votenrank_2023}.

Contrary to usual voting scenarios where the equal treatment of voters is of utmost importance, there is apriori no requirement to treat each task equally in benchmark aggregation scenarios, and we may want to consider voting schemes with different weights, e.g., chosen arbitrarily by the benchmark creator.
These weights may however carry more information than arbitrary preference. In \cite{balduzzi_re-evaluating_2018}, researchers proposed to model the evaluation of agents on different tasks through a zero-sum meta-game played between an ``agent'' and a ``task'' player, each choosing a probability distribution over the corresponding set. Scores on different tasks are then aggregated with a weighted average, where the weights correspond to the probability of playing each task in the entropy-maximizing Nash Equilibrium. One of the desirable properties of this technique is that it is invariant under the addition of \emph{exact} copies of agents, a property which has been studied under the appellation \emph{false-name-proofness} in social choice theory \cite{conitzer_using_2010,nehama_manipulation-resistant_2022,todo_characterizing_2009}. 

Similarly, independence of clones and its stronger form, composition consistency \cite{tideman_independence_1987,laffond_composition-consistent_1996,brandl_consistent_2016}, have been proposed as desirable principles for handling the cloning of alternatives. 
Our approach differs in that, without an internal order to distinguish between clones, we choose to treat them all symmetrically. 
Furthermore, all the above properties are very brittle and offer no guarantees whenever minimal noise is added to a clone.

Shortly before our preprint was posted,   \citet{procaccia_clone-robust_2025} proposed robustness to approximate clones as a desirable property for preference aggregation in reinforcement learning with human feedback. They propose a clone-robust weighing of the alternatives based on Voronoi diagrams. 
However, we show in Appendix~\ref{sec:voronoi} that this weighting function has undesirable properties;
in particular, it is discontinuous in each perfect clone, and close-by elements may receive entirely dissimilar weights.

Importantly, our framework assumes that practitioners provide a suitable distance metric, as its selection lies beyond the scope of this work.
Identifying an appropriate distance metric between tasks or datasets is a critical prerequisite for applying our approach effectively. 
Fortunately, this challenge has been extensively explored 
\cite{alvarez-melis_geometric_2020,gretton_kernel_2012,liu_wasserstein_2022}, particularly within the transfer learning literature \cite{achille_task2vec_2019,peng_domain2vec_2020}. 
This existing body of research complements our work and provides valuable guidance for practitioners seeking to use our framework for benchmark aggregation.

We further discuss in Appendix~\ref{sec:further_related_works}
potential applications of our framework, and highlight additional connections to related areas.

\section{Weighting Functions and Desirable Axioms}\label{sec:axioms}

In this section, we formally introduce weighting functions and propose a set of axioms that we consider essential for generalizing the well-understood case of discrete metrics.
Notations are introduced as they appear, a summary is however provided in Appendix~\ref{sec:def}.

Consider a metric space $(E,d)$, that is a set $E$ equipped with a notion of distance in the form of an operator $d:E\times E \mapsto \mathbb{R}_{\geq 0}$ satisfying \emph{separability}, \emph{symmetry} and \emph{triangular inequality}. 
We now formally define the object of interest of this work, called \emph{weighting functions of $(E, d)$}.

\begin{definition}[Weighting functions of $(E, d)$]
    A weighting function of $(E, d)$ is a function $f$ that maps finite sets of $E$ to probability distributions over their elements, i.e.,
    \begin{equation*}
        \begin{aligned}
            f: S \in \mathcal{P}(E) &\mapsto p_S \in \Delta(S),
        \end{aligned}
    \end{equation*}
    where $\mathcal{P}(E)$
    denotes the set containing all finite subsets of $E$ (outside the empty set),
    and $\Delta(S) = \big\{\ p_S : S \mapsto [0,1] \mid \sum_{x \in S} p
    _S(x) =1 \big\}$ denotes the  
    simplex over the elements of $S.$ 
    We moreover refer to the probability distribution $f(S): S \mapsto [0,1]$ as the weighting of $S.$
\end{definition}

Note that this definition encompasses the \emph{uniform distribution} as a particular case of weighting function.
Indeed, consider the discrete metric space $(E,\rho)$, where $\rho(x,y)$ is equal to one if $x \neq y$ and zero otherwise. Then the maximum entropy principle compels us to use the \emph{uniform weighting function} $\mathcal{U} : S \in \mathcal{P}(E) \mapsto \mathbb{1}_S(\cdot) / |S| \in \Delta(S).$

Drawing inspiration from the properties of this particular weighting function, we next introduce a few axioms 
that we argue are desirable for a general metric space $(E,d)$ and weighting function $f$ thereof.
The first desirable property that the uniform weighting $\mathcal{U}$ verifies is rather simple: it ensures that all elements of a finite set are represented with positive probability. This means that it never hurts to add new elements to a set as the support of the probability distribution given by the weighting function only increases.

\begin{axiom}[Positivity]\label{axi:pos}
Every element of a finite set is represented with positive probability, i.e.,
for all finite subset $S \in  \mathcal{P}(E)$ and element $x$ in  $S$, we have $f(S)(x) >0.$
\end{axiom}

The second property of $\mathcal{U}$ that we would want to extend to a generic $f$ is that of symmetry: 
when the distance is uninformative and some elements are isomorphic with respect to a distance preserving permutation $\sigma_S: S \mapsto S$, they receive similar weights. In particular, if all elements of a finite subset $S$ are equidistant, then $f(S)$ should be uniform over $S.$

\begin{axiom}[Symmetry]\label{axi:sym}
    Elements of a set that are symmetric with respect to the metric are equally represented, i.e., 
    for all finite subset $ S \in  \mathcal{P}(E)$ and self-isometry $\sigma_S: S \mapsto S$, it holds for all $x \in S$ that $f(S)(x) = f(S)(\sigma_S(x)). $
\end{axiom}

Importantly,  $\sigma_S$ need not be extendable to a full isometry on $E$ (c.f. Appendix~\ref{sec:def}). Moreover, determining the automorphism group of a set $S$ is an instance of the \emph{graph automorphism problem}, which is known to be 
solvable in quasi-polynomial time \cite{helfgott_graph_2017}, but is neither known to be in P nor to be NP-complete. 
Luckily, two symmetric elements $x$ and $\sigma_S(x)$ possess the same multi-set of distances $\{\{d(x,y)\}\}_{y\in S}$ and we only need to make sure that similar multi-sets lead to similar weightings.

Since perfect clones at distance precisely zero are always isomorphic with one another, Axiom~\ref{axi:sym} requires in particular that they get equal weights, and can hence be thought of as requiring \emph{fairness among perfect clones.} 
We may want to extend this fairness requirement beyond perfect clones to include approximate ones as well. Unequal treatment between the two could undermine robustness -- particularly in adversarial settings like data poisoning, where strategically crafted approximate clones could divert all the weight away from the original elements.

\begin{axiom}[Uniform Clone Fairness]\label{axi:clone_fair_uni}
Weighting is fair among approximate clones, i.e.,
for all $\varepsilon>0$, there exists $\delta>0$ such that, for all finite subset $S\in \mathcal{P}(E)$ and $x,y$ in $S$ satisfying $d(x,y)\leq \delta$, it holds that $\vert f(S)(x) - f(S)(y) \vert \leq \varepsilon.$
\end{axiom}

Finally, the third property that $\mathcal{U}$ trivially satisfies is that of continuity, since the topology induced by the metric $\rho$ is the discrete one. Intuitively, we would want to ensure that slightly perturbing each element of a subset $X$ does not result in large variations in weighting. 
Formally, we define, for two finite subsets $X$ and $Y$ in $\mathcal{P}(E)$ of cardinality $k\in \mathbb{N}$, the \emph{transport distance}  $d_\Pi(X, Y) = d_\Pi(Y,X) = \min_{\pi \in \mathrm{Bij}(Y, X)}\max_{y \in Y} d(y, \pi(y))$, where $\mathrm{Bij}(Y, X)$ denotes the set of bijections from $Y$ to $X$. We similarly define the set of \emph{minimal transport maps} from  $Y$ to $X$ as $\Pi(Y,X) = \argmin_{\pi \in\mathrm{Bij}(Y, X)} \max_{y \in Y} d(y, \pi(y)).$
Using $\pi \in \Pi(Y, X)$ to identify element $x$ in $X$ with a $\pi^{-1}(x)$ in $ Y$, we then require that both get similar weights in their respective sets.

\begin{axiom}[Uniform Individual Continuity]\label{axi:indiv_cont}
Weighting is \\element-wise continuous, i.e., for all $\varepsilon>0$ and $k\in \mathbb{N}$, there exists  $\delta>0$ such that, for all finite subsets $X, Y \in \mathcal{P}(E)$ of cardinality $\vert X \vert = \vert Y \vert = k$ such that $d_\Pi(X,Y) \leq \delta$, 
we have $\max_{x\in X} \vert f(X)(x) - f(Y)(\pi^{-1}(x)) \vert \leq \varepsilon$, where $\pi \in \Pi(Y,X).$
\end{axiom}

We show in Appendix~\ref{sec:metric_cont&axioms_disc} that Axiom~\ref{axi:indiv_cont} implies continuity with the Wasserstein metric on the codomain of $f.$
Though similar in formulation, note that Axiom~\ref{axi:clone_fair_uni} cannot be derived from Axiom~\ref{axi:indiv_cont} by simply plugging in $Y = X$, since the identity is always the unique minimal transport map (in the absence of perfect clones).

One might wonder why we restrict our continuity requirement to sets of the same cardinality. Indeed, for two finite subsets $X$ and $Y$ with cardinality $\vert Y \vert \geq \vert X \vert$, it is possible to extend the definition of $d_\Pi(Y,X)= d_\Pi(X,Y) = \min_{\pi \in\mathrm{Surj}(Y, X)} \max_{y \in Y} d(y, \pi(y))$ by requiring that $\pi \in \mathrm{Surj}(Y, X)$ is only a surjection (we similarly extend the definition of $\Pi(Y,X)$). We then show in Appendix~\ref{sec:metric_cont&axioms_disc} that $d_\Pi$ constitutes a metric on the whole domain $\mathcal{P}(E).$ 
Importantly, note that only sets $Y$ of cardinality greater or equal to that of $X$ satisfy $d_\Pi(X,Y) \leq \delta$ for small enough $\delta$; indeed, for $\delta$ smaller than $\underline{d}(X) = \min_{x \neq x' \in X} d(x,x') /2$, no element $y $ in $E$ can be simultaneously $\delta$-close to distinct $x$ and $x'$ in $X$ (c.f. Figure~\ref{fig:d_Pi}). 
For such $Y$, a surjection $\pi:Y\mapsto X \in\Pi(X,Y)$ still offers a natural way to identify elements of $Y$ with those of $X$, and we could think of each $\pi^{-1}(x) = \{ y\in Y \mid \pi(y) =x \}$ as a \emph{class of clones} since all $y, y'$ in $\pi^{-1}(x)$ are at distance at most $2\delta$ by the triangle inequality.

Summing weights locally over each class of clones,
we could then collapse some of the dimensions of the codomain $\Delta(Y)$ and identify it with $\Delta(X).$ 
With this intuition, we then define \emph{class continuity} as follows.

\begin{axiom}[Class Continuity]\label{axi:class_cont}
Weights are class-wise continuous, i.e.,
for a finite subset $X\in \mathcal{P}(E)$ and $\varepsilon>0$, there exists $ \delta$ verifying $\min_{x \neq x' \in X} d(x,x') /2 > \delta >0$ such that, for each finite subset $Y \in \mathcal{P}(E)$ satisfying $d_\Pi(X,Y) \leq \delta$, we have $\max_{x\in X} \big\vert f(X)(x) - \sum_{y \in \pi^{-1}(x)} f(Y)(y) \big\vert \leq \varepsilon$, where $\pi \in \Pi(X,Y).$ 
\end{axiom}

Note that Axiom~\ref{axi:class_cont} ensures a form of \emph{cloneproofness}, i.e., robustness of weighting under the addition of clones. 
Intuitively, a $\delta$-neighboring set $Y$ of $X$ with greater cardinality contains many $\delta$-clones, and Axiom~\ref{axi:class_cont}  ensures that both sets get ``similar weightings'' when summing probabilities locally over the redundancies in $Y.$

However, we argue that the local summation in Axiom~\ref{axi:class_cont}, although intuitive, is too strong of a requirement.
Indeed, we show hereafter that a weighting function $f$ 
verifying Axioms~\ref{axi:sym} and Axiom~\ref{axi:class_cont}
gives very different individual weights to points in nearby sets, and breaks Axiom~\ref{axi:clone_fair_uni}. 
Note that similar arguments explain the choice of $d_\Pi$ over a perhaps more standard \emph{Hausdorff distance} (c.f.  Appendix~\ref{sec:metric_id}).

\begin{figure*}[!htb]
     \centering
     \begin{subfigure}[t]{0.31\textwidth}
         \centering
         \begin{tikzpicture}[scale=0.6]
            \def\radius{0.7}
            \definecolor{PaperBlue}{HTML}{1f77b4}
            \definecolor{PaperHatch}{HTML}{d3d3d3}
            \definecolor{PaperOrange}{HTML}{ff7f0e}
            \definecolor{PaperGreen}{HTML}{2ca02c}
            \definecolor{PaperRed}{HTML}{d62728}
        
          \coordinate (x1) at (0, 0);
          \coordinate (x2) at (-0.75, 4);
          \coordinate (x3) at (0.75, 4);
        
          \coordinate (y1) at (0.2, -0.3);
          \coordinate (y2) at (-0.9, 3.7);
          \coordinate (y3) at (1.3, 4.2);
          \coordinate (y4) at (0.9, 3.5);

          \draw[-<, thin, black] (0, 0) -- (-0.55, -0.55);
          \draw[thin, black] (0, 0) -- (-0.75, -0.75);
          \node at (-1, -0.5) {\(\delta\)};

          \node[PaperBlue] at (-2, 1.2) {\(X = \{ x_i\}_{i \in [3]}\)};
          \node[PaperRed] at (2, 1.2) {\(Y = \{ y_i\}_{i \in [4]}\)};

          \draw[black, dashed] (x1) circle (\radius);
          \draw[black, dashed] (x2) circle (\radius);
          \draw[black, dashed] (x3) circle (\radius);
          \fill[PaperBlue, opacity =0.1] (x1) circle (\radius);
          \fill[PaperBlue, opacity =0.1] (x2) circle (\radius);
          \fill[PaperBlue, opacity =0.1] (x3) circle (\radius);

          \fill[PaperBlue] (x1) circle (2pt) node[above ] {$x_1$};
          \fill[PaperBlue] (x2) circle (2pt) node[above ] {$x_2$};
          \fill[PaperBlue] (x3) circle (2pt) node[above ] {$x_3$};
        
          \fill[PaperRed] (y1) circle (2pt) node[below right] {$y_1$};
          \fill[PaperRed] (y2) circle (2pt) node[below left] {$y_2$};
          \fill[PaperRed] (y3) circle (2pt) node[above right] {$y_3$};
          \fill[PaperRed] (y4) circle (2pt) node[below right] {$y_4$};

          \draw[<->, thin] (-0.75, 2.5) -- (0.75, 2.5) node[midway, below] {\(\underline{d}(X)\)};
          \draw[thin, dotted] (-0.75, 2.5) -- (x2);
          \draw[thin, dotted] (0.75, 2.5) -- (x3);

        \end{tikzpicture}
        \caption{ Sets $X$ and $Y$ at distance $d_\Pi(X,Y) \leq \delta < \underline{d}(X) /2.$}
        \label{fig:d_Pi}
     \end{subfigure}
     \hfill
     \begin{subfigure}[t]{0.294\textwidth}
         \centering
         \begin{tikzpicture}[scale=2.4]
            \tdplotsetmaincoords{60}{135}  
            \begin{scope}[tdplot_main_coords]
                \def\alphaval{0.4}
                \def\gammaval{0.1}
                \def\axislen{0.7}
                \def\zaxislen{1.3}
                \definecolor{PaperBlue}{HTML}{1f77b4}
                \definecolor{PaperHatch}{HTML}{d3d3d3}
                \definecolor{PaperOrange}{HTML}{ff7f0e}
                \definecolor{PaperGreen}{HTML}{2ca02c}
                \definecolor{PaperRed}{HTML}{d62728}
                
                \coordinate (O) at (0,0,0);
                \coordinate (U) at (\alphaval,0,1);
                \coordinate (Vplus) at (-\alphaval,\gammaval,1);   
                \coordinate (Vminus) at (-\alphaval,-\gammaval,1);
                
                \draw[PaperBlue, line width = 2pt, opacity=0.2] (Vplus) -- (Vminus) ;
                \draw[PaperBlue, line width = 2pt, opacity=0.2] (U) -- (-\alphaval,0,1) ;
                
                \fill[PaperBlue] (O) circle (0.7pt) node[tdplot_screen_coords, anchor=east, xshift=-1pt] {$o$};
                \fill[PaperBlue] (U) circle (0.7pt) node[tdplot_screen_coords, anchor=north east, yshift=1pt] { $u_\alpha$};
                \fill[PaperBlue] (Vplus) circle (0.7pt) node[tdplot_screen_coords, anchor=north west ] { $v^{+}_{\alpha,\gamma}$};
                \fill[PaperBlue] (Vminus) circle (0.7pt) node[tdplot_screen_coords, anchor=south , xshift=-2pt, yshift = 2pt] { $v^{-}_{\alpha,\gamma}$};
                
                \draw[dotted,very thin] (-\axislen,0,0) -- (\axislen,0,0); 
                \draw[dotted,very thin] (0,-\axislen,0) -- (0,\axislen,0); 
                \draw[dotted,very thin] (0,0,-0.13*\zaxislen) -- (0,0,\zaxislen); 
                
                \node at (\axislen + 0.1,0,0) {$x$};
                \node at (0,\axislen + 0.1,0) {$y$};
                \node at (0,0,\zaxislen + 0.1) {$z$};

                \node[align=center] at ( { 0.5}, {0 }, {0.6} )
                { $p\big(u_\alpha\big)\approx\frac{1}{4}$};
                \node[align=center] at ( { -0.8}, {0 }, {0.4} )
                { $p\Big(v^{+}_{\alpha,\gamma}\Big)=p\Big(v^{-}_{\alpha,\gamma}\Big)\approx\frac{1}{8}$};
            \end{scope}
            \end{tikzpicture}
         \caption{$S_{\alpha, \alpha, \gamma}$ for $1 \gg \alpha \gg \gamma >0.$}
         \label{fig:visualization_power_two}
     \end{subfigure}
     \hfill
     \begin{subfigure}[t]{0.32\textwidth}
        \centering
        \begin{tikzpicture}[scale=2.4]
            \tdplotsetmaincoords{60}{135} 
            \begin{scope}[tdplot_main_coords]
                \def\alphaval{0.4}
                \def\axislen{0.7}
                \def\zaxislen{1.3}
                \definecolor{PaperBlue}{HTML}{1f77b4}
                \definecolor{PaperHatch}{HTML}{d3d3d3}
                \definecolor{PaperOrange}{HTML}{ff7f0e}
                \definecolor{PaperGreen}{HTML}{2ca02c}
                \definecolor{PaperRed}{HTML}{d62728}
                
                \coordinate (O) at (0,0,0);
                \coordinate (U) at (\alphaval,0,1);
                \coordinate (Vplus) at (-\alphaval/2,{0.86602540378*\alphaval},1);   
                \coordinate (Vminus) at (-\alphaval/2,{-0.86602540378*\alphaval},1);
                
                \fill[PaperBlue,opacity=0.18] (U) -- (Vplus) -- (Vminus) -- cycle;
            
                \fill[PaperBlue] (O) circle (0.7pt) node[tdplot_screen_coords, anchor=east, xshift=-1pt] {$o$};
                \fill[PaperBlue] (U) circle (0.7pt) node[tdplot_screen_coords, anchor=north east, yshift=1pt] { $u_\alpha$};
                \fill[PaperBlue] (Vplus) circle (0.7pt) node[tdplot_screen_coords, anchor=north west, xshift=0pt] { $v^{+}_{\alpha/2,\sqrt{3}\alpha/2}$};
                \fill[PaperBlue] (Vminus) circle (0.7pt) node[tdplot_screen_coords, anchor=south east, xshift=0pt] { $v^{-}_{\alpha/2,\sqrt{3}\alpha/2}$};
                
                \draw[dotted,very thin] (-\axislen,0,0) -- (\axislen,0,0); 
                \draw[dotted,very thin] (0,-\axislen,0) -- (0,\axislen,0); 
                \draw[dotted,very thin] (0,0,-0.13*\zaxislen) -- (0,0,\zaxislen); 
                
                \node at (\axislen + 0.1,0,0) {$x$};
                \node at (0,\axislen + 0.1,0) {$y$};
                \node at (0,0,\zaxislen + 0.1) {$z$};
                
                \node[align=center] at ( { 0}, {0 }, {0.4} )
                { $p\big(u_\alpha\big)=p\Big(v^{+}_{\alpha/2,\sqrt{3}\alpha/2}\Big)=p\Big(v^{-}_{\alpha/2,\sqrt{3}\alpha/2}\Big)\approx\frac{1}{6}$};
            \end{scope}
            \end{tikzpicture}
        \caption{$S_{\alpha, \alpha/2, \sqrt{3} \alpha/ 2 }$ for $1\gg \alpha>0$}
        \label{fig:visualization_equilateral}
     \end{subfigure}
     
        \caption{Visualization of the neighborhoods of $d_\Pi$, and of the divergence of individual weightings under Axioms~\ref{axi:sym} and~\ref{axi:class_cont}. The edges in
        ~\ref{fig:visualization_power_two} highlight the symmetries of $S_{\alpha, \alpha, \gamma}$ in the limit $\gamma \to 0$; the equilateral triangle in
        ~\ref{fig:visualization_equilateral} displays the symmetry of $S_{\alpha, \alpha/2, \sqrt{3} \alpha/ 2 }.$}
        \Description{Subfigure a shows illustrate that there are only sets of higher cardinality in a small enough vicinity of a finite set X. Subfigures b and c illustrate the construction for diverging individual weights in the example.}
        \label{fig:visualization_param_families}
\end{figure*}

\begin{example}[Diverging Individual Weights.]

Let $f$ be a weighting function on the three-dimensional Euclidean space $(\mathbb{R}^3, d_2)$ satisfying both Axioms~\ref{axi:sym} and~\ref{axi:class_cont}, and define the parametric family $S_{\alpha,\beta, \gamma} = \big\{o, u_\alpha, v^+_{\beta,\gamma}, v^-_{\beta,\gamma} \big\}$, where $o=(0,0,0)$, $u_\alpha = (1,\alpha, 0)$, $v^+_{\beta,\gamma} = (1,-\beta, \gamma)$ and $v^-_{\beta,\gamma} = v^+_{\beta,-\gamma}.$ Figure~\ref{fig:visualization_param_families} summarizes our construction.

On one hand, consider the set $S_{\alpha, \alpha, \gamma}$, where $\alpha>0$ is fixed and $\gamma>0$ is much smaller than $\alpha.$
Since $S_{\alpha, \alpha, \gamma}$ converges to the set $S_{\alpha} = \{o, u_\alpha, u_{-\alpha} \}$ in $(\mathcal{P}(\mathbb{R}^3), d_\Pi)$ when $\gamma$ goes to zero, Axiom~\ref{axi:class_cont} implies $\lim_{\gamma\to 0} f(S_{\alpha, \alpha, \gamma})(u_\alpha) = f(S_{\alpha})(u_\alpha).$
Moreover, $S_{\alpha}$ converges in turn to the symmetric set $S_0 = \{o, u_0\}$, 
and Axioms~\ref{axi:sym} and~\ref{axi:class_cont} together imply that $\lim_{\alpha\to 0} f(S_\alpha)(u_\alpha) = f(S_0)(u_0) /2 = 1/4.$ 
Combining these two results, we  get
\begin{equation*}
\lim_{\alpha\to 0} \lim_{\gamma\to 0} f(S_{\alpha, \alpha, \gamma})(u_\alpha) =1/4.
\end{equation*}

On the other hand, consider the set $S_{\alpha, \alpha/2, \sqrt{3} \alpha/ 2 }.$ Note that the points $u_\alpha$, $v^+_{\alpha/2, \sqrt{3} \alpha/ 2}$ and $v^-_{\alpha/2, \sqrt{3} \alpha/ 2}$ form an equilateral triangle centered in $u_0$
and orthogonal to the origin, hence by Axiom~\ref{axi:sym} they must receive similar weights. 
As $S_{\alpha, \alpha/2, \sqrt{3} \alpha/ 2 }$ also converges to the symmetric $S_0$ when $\alpha$ goes to zero, Axioms~\ref{axi:sym} and~\ref{axi:class_cont} finally imply 
\begin{equation*}
\lim_{\alpha\to 0} f \big(S_{\alpha, \alpha/2, \sqrt{3} \alpha/ 2 }\big) (u_\alpha) = 1/6.
\end{equation*}

We hence constructed two arbitrarily close sets of the same cardinality whose individual weightings differ. Moreover, 
$u_\alpha$ and $v^+_{\alpha,\gamma}$ receive vastly different weights in the limit $\alpha\to 0, \gamma\to 0$, although their distance tends to zero, i.e.,
$$\frac{1}{4} = \lim_{\alpha\to 0} \lim_{\gamma\to 0} f(S_{\alpha, \alpha, \gamma})(u_\alpha) \neq \lim_{\alpha\to 0} \lim_{\gamma\to 0} f(S_{\alpha, \alpha, \gamma})(v^+_{\alpha,\gamma}) = \frac{1}{8} ,$$
and Axiom~\ref{axi:clone_fair_uni} breaks.

\end{example}

Given this incompatibility, we preserve Axiom~\ref{axi:clone_fair_uni} and opt for the restriction to sets of similar cardinality in Axiom~\ref{axi:indiv_cont}. 
However, note that Axiom~\ref{axi:indiv_cont},  unlike Axiom~\ref{axi:class_cont}, does not directly address the addition of clones. To account for this, we introduce another weakening of Axiom~\ref{axi:class_cont}, essentially requiring continuity of weighting everywhere except in the vicinity of the newly added clone. This relaxation allows for greater flexibility in how the mass is redistributed locally.

\begin{axiom}[Uniform $\alpha$-Locality under Addition of Clones]\label{axi:alpha_clone_locality}
The addition of a clone only changes the weights of points in the $\alpha$-neighborhood of the clone, i.e.,
for all $\varepsilon >0$, there exists $\delta>0$ such that, for each finite subset $S \in \mathcal{P}(E)$ and elements $x\in S$ and $x' \in E \setminus S$ satisfying $d(x,x') \leq \delta$, we have for all $z\in S$ such that $d(x,z)\geq \alpha$ that $\vert f(S)(z) - f(S\cup\{x'\})(z) \vert \leq \varepsilon.$
\end{axiom}

Note that Axioms~\ref{axi:clone_fair_uni} and~\ref{axi:alpha_clone_locality} provide orthogonal restrictions in the presence of clones: the former dictates how to shift weights around the recently introduced clone, while the latter ensures weights do not change away from it.
We further discuss the relationship between these axioms in Appendix~\ref{sec:metric_cont&axioms_disc}, and refer the interested reader to Appendix~\ref{sec:metric_id} for the treatment of perfect clones.

\smallbreak

We finally denote by $\mathcal{R}_\alpha(E,d) $ the set of weighting functions on $(E,d)$ satisfying Axioms~\ref{axi:pos}, \ref{axi:sym}, \ref{axi:clone_fair_uni}, \ref{axi:indiv_cont} and~\ref{axi:alpha_clone_locality} with parameter $\alpha>0.$ 
The burning question is now this: does the set $\mathcal{R}_\alpha(E,d)$ contain any elements at all?

\paragraph{Back to the Motivating Example.}

Let's now examine the guarantees provided by our axioms in the context of multi-task benchmark aggregation. First, Axiom~\ref{axi:pos} ensures that each task has a contribution to the final aggregated score (in fact, we show a stronger guarantee for $g_r$ in the proof of Theorem~\ref{thm:local_vote_rep_func}).  
Axioms~\ref{axi:sym} and~\ref{axi:clone_fair_uni} both reflect principles of anonymity in social choice. Axiom~\ref{axi:sym} ensures some level of \emph{isotropy} in the embedding space, meaning that tasks which are indistinguishable under the symmetry of the embedding space receive equal weights. Axiom~\ref{axi:clone_fair_uni} guarantees that tasks that are deemed similar based on their embedding will be assigned similar weights.  
Axiom~\ref{axi:indiv_cont} addresses the \emph{continuity of the weighting}. It ensures that small changes in a task --such as adding a few elements to a test set-- should not lead to large changes in the assigned weight.  
Finally, Axiom~\ref{axi:alpha_clone_locality} enforces \emph{weight sharing for similar tasks}. It ensures that tasks with sufficiently different characteristics remain unaffected by the introduction of new yet partially redundant tasks to the benchmark.

Philosophically, our framework offers insight into the assumptions about the space of tasks that are baked into a given choice of weights.
For instance, giving equal weight to each task could reflect various possible scenarios: it could suggest that tasks are either all equally dissimilar, all similar, or simply symmetric within the space. The continuity in Axiom~\ref{axi:indiv_cont} is also of philosophical significance. While the construction of a benchmark is inherently a discrete process, one may hope that small changes around the benchmark (such as adding noise) do not result in completely different outcomes.

\section{Local Voting Approach}\label{sec:local_voting}

To construct weighting functions that satisfy our axioms, the first step is to identify invariant objects under the addition of clones: we argue that the open balls around an element $x \in E$, that is $B_r (x) := \{y \in E \mid d(x,y) < r \}$ for some radius $r>0$, are natural invariants for our problem. 
Indeed, they are stable under the addition of clones, in the sense that for some \emph{$\delta$-clones} $x,y$ in  $E$ satisfying $d(x,y) \leq \delta$, the triangle inequality ensures that $B_r(x) \subseteq B_r(x) \cup B_r(y) \subseteq B_{r+\delta}(x).$ 
If we then equip our space with a measure $\mu$ defined on the open balls of the space\footnote{I.e., on the Borel $\sigma$-algebra.} 
and associate with each finite subset $X\subseteq E$ its neighborhood $B_r(X):= \bigcup_{x\in X} B_r(x)$, we obtain a map invariant under clone addition. Indeed, for each neighboring  finite set $Y \subseteq B_\delta(X)$ with $r>\delta>0$, we have $\mu\big(B_{r- \delta}(X)\big) \leq \mu\big(B_r(Y)\big) \leq \mu\big(B_{r+ \delta}(X)\big)$ and the map $X\in \mathcal{P}(E) \mapsto \mu\big(B_r(X)\big)$ is continuous with respect to the  distance $d_\Pi$.\footnote{At least when $\mu$ is locally finite.}

Note however that further requirements are needed to satisfy the symmetry in Axiom~\ref{axi:sym}, essentially regarding the homogeneity and the isotropy of the underlying measure space.
For this reason, we focus on Euclidean spaces $(\mathbb{R}^n, d_2)$ for the remainder of the section,  where $d_2^2(x,y) = \sum_{i=1}^n (x_i - y_i)^2$ for all $x=(x_i)_{1\leq i\leq n}$ and $y=(y_i)_{1\leq i\leq n}$ in $\mathbb{R}^n.$
We will discuss in Section~\ref{sec:discu} how to adapt our approach to more general metric spaces.

Based on the above invariant, we construct a weighting function as a local voting scheme. For a fixed $r>0$ and finite subset $S \subseteq \mathbb{R}^n$, we consider each element of $B_r(S)$ as a voter that approves only of the candidates in $S$ close to him, and as such spreads his voting power equally among them. 
Formally, we define the grade that each voter $z$ in $B_r(S)$ attributes to a candidate $x$ in $S$ as follows 
\begin{equation*}
    g_{r,S,x} (z) = \frac{\mathbb{1}_{B_r(x)}(z) }{\sum_{y \in S} \mathbb{1}_{B_r(y)}(z)} .
\end{equation*}

We then aggregate the ballots with the Lebesgue measure $\mu$ and finally define the weighting function $g_r$, for each finite subset $S \subseteq \mathbb{R}^n$, i.e.,
\begin{equation*}
    g_r(S) : x \in S \mapsto  \int_{B_r(S)} \frac{g_{r,S,x}(z)}{\mu\big(B_r(S)\big)} \:d\mu(z).
\end{equation*}

As illustrated in Figure~\ref{fig:illustration_g_r_depth_cells},  the weighting function $g_r$ computes a weighted average of the inverse depth of each cell, with the depth defined as the number of intersecting balls forming the cell and the weights based on the cell's size.

\begin{figure}[!tbh]
    \centering
    \begin{tikzpicture}[scale=1.5]

        \definecolor{PaperBlue}{HTML}{1f77b4}
        \definecolor{PaperOrange}{HTML}{ff7f0e}
        \definecolor{PaperGreen}{HTML}{2ca02c}
        \definecolor{PaperRed}{HTML}{d62728}

        \coordinate (X) at (0,0);        
        \coordinate (W) at (-0.5,0.1); 
        \coordinate (Y) at (1.0,0.7);   
        \coordinate (Z) at (1.7,-0.3);  
        \fill[PaperBlue] (X) circle (1pt) node[above] {$x$};
        \fill[PaperOrange] (W) circle (1pt) node[below left] {$w$};
        \fill[PaperGreen] (Y) circle (1pt) node[right] {$y$};
        \fill[PaperRed] (Z) circle (1pt) node[above right] {$z$};

        
        \fill[PaperBlue, opacity=0.09] (X) circle (1.5cm);

        \path[name path=Wcircle] (W) circle (1.5cm);
        \path[name path=Ycircle] (Y) circle (1.5cm);
        \path[name path=Zcircle] (Z) circle (1.5cm);
        \path[name path=Xcircle] (X) circle (1.5cm);

        \draw[thin, dotted, PaperOrange] (W) circle (1.5cm);
        \fill[PaperOrange, opacity=0.06] (W) circle (1.5cm);
        \draw[thin, dotted, PaperGreen] (Y) circle (1.5cm);
        \fill[PaperGreen, opacity=0.06] (Y) circle (1.5cm);
        \draw[thin, dotted, PaperRed] (Z) circle (1.5cm);
        \fill[PaperRed, opacity=0.06] (Z) circle (1.5cm);

        \begin{scope}
            \path[name intersections={of=Xcircle and Wcircle, by={XWa, XWb}}];
            \draw[ PaperOrange] (XWa) arc[start angle=429.1, end angle=269, radius=1.5cm];
            
            \path[name intersections={of=Xcircle and Ycircle, by={XYa, XYb}}];
            \draw[PaperGreen] (XYa) arc[start angle=149, end angle=281, radius=1.5cm];
        
            \path[name intersections={of=Xcircle and Zcircle, by={XZa, XZb}}];
            \draw[ PaperRed] (XZa) arc[start angle=115, end angle=224.5, radius=1.5cm];
        \end{scope}

        \draw[PaperBlue] (X) circle (1.5cm);

        \node at (-0.8,0.6) { $\frac{1}{2}$};  
        \node at (-0.6,-0.62) { $A_r(\{w,x\})$}; 
        \node at (0,0.8) { $\frac{1}{3}$}; 
        \node at (0.6,-0.1) {$\frac{1}{4}$};   
        \node at (0.44,-0.87) {$\frac{1}{3}$};
        \node at (0.3,-1.31) { $1$};  
        \node at (0.8,-1.0) {$\frac{1}{2}$};  
        \node at (1.2,-0.05) { $\frac{1}{3}$};  
        \node at (0.75,1.13) {$\frac{1}{2}$}; 
        \node at (-1.2,1.8) { $g_r(S)(x) \simeq 0.19$};

    \end{tikzpicture}
    \caption{
    Computation of $g_r(S)(x)$ in the two-dimensional Euclidean space $(\mathbb{R}^2, d_2)$, where the set $S = \{w,x,y,z\}$ contains four elements. 
    A cell $A_r(U)$ is uniquely defined by the subset $U\subseteq S$ as the possibly empty intersection of the balls around each element in $U$ and the complement of the balls of each element of $S$ absent from $U$ (c.f. Appendix~\ref{sec:proof}). 
    For each subset $U$ containing $x$, the grading function $g_{r,S,x}$ is constant on the cell $A_r(U)$ and equal to the inverse depth of the cell, i.e., $g_{r,S,x}(z) = 1 / U$ for all $z$ in $A_r(U).$
   The weight of $x$ in $S$
   is then equal to the weighted average of $g_{r,S,x}$ on the ball centered in $x$, where the weight of each cell corresponds to its area normalized by the total area of the balls' union. 
   We estimated the value $g_r(S)(x) \simeq 0.19$ via Monte Carlo sampling, c.f. Algorithm~\ref{alg:gr_estimate}.
    }
    \Description{Cell decomposition induced by overlapping balls and computation of g_r(S)(x).The figure illustrates the arrangement of four Euclidean balls of equal radius centered at S={w,x,y,z} in R^2. The plane is partitioned into cells A_r(U) each determined by a subset U of S: a cell consists of points lying inside the balls centered at elements of U and outside the balls centered at elements not in U.
    For subsets U containing x, the local grading function g_r(S)(x)is constant on each cell A_r(U) and equals the reciprocal of the depth ∣U∣. The bold arcs highlight the boundary contributions of the ball centered at x within intersections of increasing depth.
    The value g_r(S)(x) is obtained as the area-weighted average of contribution of each cell, where each cell contributes proportionally to its relative area in the union of balls. In this configuration, Monte Carlo estimation yields g_r(S)(x)≈0.19.}
    \label{fig:illustration_g_r_depth_cells}
\end{figure}

We next show that this class of weighting functions satisfies the desirable axioms introduced above.

\begin{theorem}\label{thm:local_vote_rep_func}
For $r>0$, the weighting function $g_r$ is well-defined and belongs in $\mathcal{R}_{2r}(\mathbb{R}^n, d_2).$
\end{theorem}

The detailed proof of Theorem~\ref{thm:local_vote_rep_func} is included in Appendix~\ref{sec:proof}, and we provide here a sketch of the proof.
As depicted in Figure~\ref{fig:illustration_g_r_depth_cells}, the weight $g_r(S)(x)$ is in fact a weighted average of positive elements, hence it is positive and Axiom~\ref{axi:pos} trivially holds.
Showing that $g_r$ is symmetric (Axiom~\ref{axi:sym})
is also relatively straightforward after observing the following two properties of Euclidean spaces: first, the fact that one can uplift any self-isometry $\sigma_S$ on a finite subset $S$ to an isometry on the entire space $\mathbb{R}^n$, c.f. Appendix~\ref{sec:self-iso_eucli}; second, the fact that the Lebesgue measure is invariant under \emph{translations}, \emph{rotations} and \emph{reflections}, which generate the group of Euclidean isometries \cite{gallian_contemporary_2020}.
The most challenging aspect of the proof is verifying that $g_r$ satisfies 
Axioms~\ref{axi:clone_fair_uni}, \ref{axi:indiv_cont} 
and~\ref{axi:alpha_clone_locality}.
While these proofs are technically intricate, they fortunately follow a similar structure. We illustrate our approach
by focusing on the simpler case of Axiom~\ref{axi:clone_fair_uni} below. 
In order to bound the difference of 
weightings $\vert g_r(S)(x) - g_r(S)(y)\vert$  between two approximate clones $x$ and $y$ in a given set $S$, we first show that the grading functions  $g_{r,S,x}$ and $g_{r,S,y}$ are equal outside of a thin spherical shell parametrized by $\delta$, the distance between $x$ and $y$ (c.f. Figure~\ref{fig:ball_inclusion}).
This allows us to obtain a difference of Lebesgue measure $\mu(B_{r}(x)) - \mu(B_{r-\delta}(x))$, which we then bound in terms of $\delta$ by taking the limit of this difference as $\delta$ approaches zero (c.f. Figure~\ref{fig:minkowski_content}). 
The formalization of this argument relies 
on tools from \emph{geometric measure theory}, particularly the $n-1$-dimensional Minkowski content. 
These arguments are illustrated in Figure~\ref{fig:illustration_proof_thm1}.

\begin{figure}[!tbh]
    \centering
    \begin{subfigure}[t]{0.48\columnwidth}
        \centering
        \begin{tikzpicture}[scale = 0.8]
        
            \definecolor{PaperBlue}{HTML}{1f77b4}
            \definecolor{PaperRed}{HTML}{d62728}
            
            \coordinate (x) at (0,0); 
            \coordinate (y) at (0.5,0); 
            
            \draw[thick, PaperBlue] (x) circle (2); 
            \draw[ PaperBlue, dashed] (x) circle (1.5); 
            \draw[thick, PaperRed] (y) circle (2); 
            \fill[PaperBlue, opacity =0.1] (x) circle (1.5);

            \draw[<->, thin] (0,-0.4) -- (0.5, -0.4) node[midway, below] {\(\delta\)};
            \draw[thin, dotted] (0,-0.4) -- (x);
            \draw[thin, dotted] (0.5,-0.4) -- (y);

            \fill[PaperBlue] (x) circle (2pt) node[above left] {\(x\)};
            \fill[PaperRed] (y) circle (2pt) node[above right] {\(y\)};
    
            \node[PaperBlue] at (-1.8, -1.8) { $B_r(x)$};  
            \node[PaperRed] at (2.3,1.8) { $B_r(y)$};
            \node[PaperBlue] at (0,1.) { $B_{r-\delta}(x)$};
        \end{tikzpicture}
        \caption{The ball $B_{r-\delta}(x)$ belongs in the intersection of $B_r(x)$ and $B_r(y)$, hence $x$ and $y$ receive the same grade from every voter on $B_{r-\delta}(x).$
        }
        \label{fig:ball_inclusion}
    \end{subfigure}
    \hfill
    \begin{subfigure}[t]{0.48\columnwidth}
         \centering
         \begin{tikzpicture}[scale=0.8] 
            \definecolor{PaperBlue}{HTML}{1f77b4}
            \definecolor{PaperHatch}{HTML}{d3d3d3}
            \definecolor{PaperOrange}{HTML}{ff7f0e}
            \definecolor{PaperGreen}{HTML}{2ca02c}
            \definecolor{PaperRed}{HTML}{d62728}
            
            \coordinate (x) at (0,0); 
            \fill[PaperGreen, opacity =0.1] (x) circle (2); 
            
            \fill[white] (x) circle (1.495); 
            
            \draw[PaperBlue] (x) circle (2); 
            
            \draw[PaperBlue,  dashed] (x) circle (1.5); 

            \draw[>-<, thin, black] (-1, 1) -- (-1.47, 1.47);
            \draw[thin, black] (-0.9, 0.9) -- (-1.6, 1.6);
            \node at (-1.87, 1.5) {\(\delta\)};
            
            \fill[PaperBlue] (x) circle (2pt) node[below] {$x$};
            
            \node[PaperBlue] at (-1.8, -1.8) {$B_r(x)$};  
            \node[PaperBlue] at (0,1.) { $B_{r-\delta}(x)$};
            \node[PaperGreen] at (1.2, -1.2) { $\Delta$};
        \end{tikzpicture}
         \caption{
         The $n$\textsuperscript{th}-dimensional volume of the green set $\mu(\Delta)$ can be approximated as $\delta$ times $S^{n-1}_r$, the $n-1$\textsuperscript{th}-dimensional surface of a ball of radius $r.$
         }
         \label{fig:minkowski_content}
     \end{subfigure}
     
    \caption{Key steps in demonstrating that $g_r$ satisfies Axioms~\ref{axi:clone_fair_uni}, \ref{axi:indiv_cont} and~\ref{axi:alpha_clone_locality}.
    }
    \Description{Subfigure (a) illustrates that two neighboring balls $B_r(x)$ and $B_r(y)$ get the same grade in a large portion of their respective support. Subfigure (b) depicts the annular region $\Delta = B_r(x) \setminus B_{r-\delta}(x)$ with its $n$-dimensional volume approximated by $\delta \cdot S_r^{n-1}$, where $S_r^{n-1}$ is the $(n-1)$-dimensional surface area of a ball of radius $r$.}

    \label{fig:illustration_proof_thm1}
\end{figure}

Since $\mathcal{R}_{2\alpha}(E,d)$ is monotonically increasing in positive $\alpha$,
Theorem~\ref{thm:local_vote_rep_func} actually ensures that the whole collection $\{ g_r\}_{\alpha \geq r>0}$ belongs in $\mathcal{R}_{2\alpha}(\mathbb{R}^n,d_2).$ 
Moreover, it is relatively straightforward to see that $\mathcal{R}_{2\alpha}(\mathbb{R}^n,d_2)$ is a convex set and, as such, contains all finite convex combinations of $\{ g_r\}_{\alpha \geq r>0}.$
Since the weighting functions $g_r$ are well-behaved, we generalize this result as follows.

\begin{theorem}\label{thm:cont_convex_combi_gr}
    Let $\nu$ be a probability density function over $[0,\alpha].$
    Then the weighting function $f_\nu : S \in \mathcal{P}\big(\mathbb{R}^n \big) \mapsto \int_0^{\alpha} \nu(r) g_r(S) ~dr$ belongs in $\mathcal{R}_{2\alpha}(\mathbb{R}^n,d_2).$
\end{theorem}
The detailed proof of Theorem~\ref{thm:cont_convex_combi_gr}, provided in Appendix~\ref{sec:proof}, relies on inequalities derived for the proof of Theorem~\ref{thm:local_vote_rep_func}.

\section{Computational Considerations}\label{sec:computational_cons}

Our focus thus far has been on identifying weighting functions with theoretically desirable properties. However, from a practical standpoint, such tools are of limited utility if they cannot be computed efficiently. This concern is encapsulated in the following principle.

\begin{axiom}[Exact Computability]\label{axi:comput}
The weighting of any given subset is efficiently computable, i.e.,
for any subset $S\in \mathcal{P}(E)$, the probability distribution $f(S)$ can be exactly computed in time polynomially bounded by the cardinality $\vert S\vert$ of the subset, and the dimension $n$ of the space when $E = \mathbb{R}^n.$
\end{axiom}

It is worth noting that the weighting functions introduced in Section~\ref{sec:local_voting} are unlikely to meet this criterion.
For example, computing $g_r(S)(x)$ would a priori involve averaging $g_{r,S,x}$ over as many as $O(2^{\vert S\vert})$ disjoint cells, making the approach computationally infeasible. 
Even the simpler task of evaluating the volume $\mu(B_r(S))$ of the union of Euclidean balls becomes increasingly challenging in higher dimensions (see \cite{cazals_computing_2011} for the case $n=3$). While hardness results for this exact problem are not readily available,  related geometric problems --such as computing the volume of the union of axis-aligned boxes-- are known to be \#P-hard \cite{bringmann_approximating_2010}, suggesting that exact computation of $g_r(S)$ is unlikely to admit a polynomial-time solution in general.

In light of these challenges, we may want to relax Axiom~\ref{axi:comput} and settle for efficient \emph{approximate} evaluations of $g_r(S)$ and $f_\nu(S)$.
Monte Carlo sampling techniques \cite{bringmann_approximating_2010,mitchell_fast_2018} offer a natural route in this direction: by sampling random points in metric neighborhoods of the elements of $S$, one can estimate the relevant integrals within a prescribed accuracy $\varepsilon>0$ and confidence level $1-\delta$. 

We next formalize this idea and introduce a simple Monte Carlo sampling procedure---namely Algorithm~\ref{alg:gr_estimate}---which provides a consistent and asymptotically unbiased estimator of $g_r(S).$ The following theorem establishes explicit $(\varepsilon,\delta)$-style guarantees on its accuracy.

\begin{algorithm}[!bth]
\caption{ Naive Monte Carlo Estimation of $g_r(S)$}
\label{alg:gr_estimate}
\begin{algorithmic}[1]
\State \textbf{Input: } Finite set $S \subset \mathbb{R}^n$, radius $r > 0$, sample size $k \in \mathbb{N}.$

\vspace{0.5ex}
\State \textbf{Precomputation:} For each $x \in S$, compute the closed neighborhood $B_{2r}(x)\cap S.$
\vspace{0.5ex}
\For{each $x \in S$}
    \State Initialize $\widehat N_x \gets 0.$
    \For{$i=1,\dots,k$}
        \State Sample $z_i \sim \mathrm{Unif}(B_r(x)).$
        \State Compute  count $c_i := \big\vert \{ y \in B_{2r}(x)\cap S \mid d_2(z_i,y) \leq r \}\big\vert.$
        \State Update $\widehat N_x \gets \widehat N_x + \frac{1}{c_i}.$
    \EndFor
    \State Set $\widehat N_x \gets \widehat N_x / k.$
\EndFor
\vspace{0.5ex}
\State Compute normalization factor $\widehat D := \sum_{x \in S} \widehat N_x.$
\State \textbf{Output: } Estimates $\widehat g_{r}(S)(x) \; :=\; \widehat N_x / \widehat D$ for each $x \in S.$
\end{algorithmic}
\end{algorithm}

\begin{theorem}\label{thm:gr_estimate}
    Algorithm~\ref{alg:gr_estimate} yields a consistent and asymptotically unbiased estimate of $g_r(S).$ 
    \noindent Moreover, for any  target accuracy $\varepsilon >0$ and confidence  level $\delta \in (0,1)$, setting the number of samples to satisfy $$k\geq \frac{( \vert S \vert ^2 - 1)^2}{2\varepsilon^2 \vert S \vert ^2} \ln\!\frac{2 \vert S \vert}{\delta}$$
    guarantees  with probability at least $1 -\delta$ that
    $$ \big\vert \;\widehat{g}_r(S)(x)  - g_r(S)(x) \;\big\vert \leq \varepsilon.$$
\end{theorem}

The proof of Theorem~\ref{thm:gr_estimate}, included in Appendix~\ref{sec:add_comput_res}, proceeds by separately estimating the unnormalized numerators---the average grade of a random voter $z$ sampled uniformly from $B_r(x)$---and the common normalization factor in the denominator---the relative volume of the union of balls $\mu(B_r(S))/\mu(B_r(x))$. 
Standard concentration results---specifically Hoeffding’s inequality---are then applied to control the deviation of these empirical averages with high probability, and the final error bound is obtained by carefully propagating this deviation through the ratio of the two estimators $\widehat N_x/\widehat D$.

The resulting accuracy guarantee follows the typical Monte Carlo rate $\varepsilon \propto 1/\sqrt{k}$, which is optimal up to constant factors for independent sampling schemes. 
The total runtime of Algorithm~\ref{alg:gr_estimate} is $O(nk \vert S \vert^2)$, since each $c_i$ can be computed in $O(n \vert S \vert)$ time, and the algorithm produces an estimate for each of the $\vert S \vert$ elements in $S.$ 
Importantly, the required number of samples per estimate, $k$, is \emph{independent of the dimension} $n$ of the Euclidean space $\mathbb{R}^n.$ 
Note however that $k$ scales (up to a logarithmic factor) as $k\propto \vert S\vert^2$. 
This quadratic dependency arises because the errors on the numerators accumulate when estimating the denominator, which must therefore be approximated to precision $\varepsilon /\vert S\vert$  to achieve an overall accuracy of $\varepsilon.$
More generally, a $\lvert S \rvert^2$ scaling is essentially unavoidable whenever we seek a simultaneous high-probability accuracy guarantee for all $x \in S$. Indeed, controlling the joint deviation of all $\lvert S \rvert$ estimates via a union bound forces each individual estimate to be of order $O(\varepsilon / \lvert S \rvert)$ in order to achieve a global accuracy of $\varepsilon$.

If the $O(\lvert S \rvert^2)$ cost is prohibitive and only a single weight $g_r(S)(x)$ is needed, one can instead consider Algorithm~\ref{alg:gr_approxunion}, which leverages the \textsc{ApproxUnion} algorithm from \cite{bringmann_approximating_2010} as a subroutine to directly approximate the volume of the union of balls $\mu(B_r(S)).$

\begin{algorithm}[!tbh]
\caption{Estimation of $g_r(S)(x)$ with \textsc{ApproxUnion}}
\label{alg:gr_approxunion}
\begin{algorithmic}[1]
\Require Finite set $S \subset \mathbb{R}^n$, element $x\in S$, radius $r > 0$,  accuracy parameter $\varepsilon > 0$, and failure probability $\delta \in (0,1)$.
\Ensure Output $\widehat g_r(S)(x)$ verifies $ \vert \widehat g_r(S)(x) -g_r(S)(x) \vert \leq \varepsilon$ with probability at least $1-\delta.$

\vspace{0.5ex}
\State Set sample size $k \gets \big\lceil \frac{8 (\vert S\vert -1)^2}{\varepsilon^2 \vert S\vert^2} \ln\!\frac{4}{\delta} \big\rceil $

\State Sample $z_1, \dots, z_k \stackrel{\text{i.i.d.}}{\sim} \mathrm{Unif}(B_r(x))$.
\State Compute empirical average
    $$\widehat N_x \;\gets\; \frac{1}{k}\sum_{i=1}^k \frac{1}{\big \vert\{y \in S  \mid d_2(z_i,y) \leq r\}\big\vert}.$$
\vspace{0.5ex}
\State Set amplification parameter $t \gets \big\lceil \ln\frac{2}{\delta} \big\rceil$.
\For{$j = 1,\dots,t$}
    \State $\widehat U_j \gets \textsc{ApproxUnion}(\{B_r(x) : x \in S\}, \varepsilon/4)$.
\EndFor
\State Set $\widehat U \gets \operatorname{median}(\widehat U_1,\dots,\widehat U_t)$.
\State Compute $\widehat D \gets \widehat U / \operatorname{vol}(B_r(x))$.
\vspace{0.5ex}

\State \textbf{Output: } Return $\widehat g_r(S)(x) := \widehat N_x / \widehat D$.
\end{algorithmic}
\end{algorithm}

A proof that Algorithm~\ref{alg:gr_approxunion} achieves the stated high-probability accuracy guarantee ---quantified in terms of target precision $\varepsilon$ and confidence level $\delta$---can be found in Appendix~\ref{sec:add_comput_res}. 
In our setting, all geometric oracles required by \textsc{ApproxUnion}---membership testing within a ball, volume computation, and uniform sampling---can be implemented efficiently, each in time $O(n)$ with full numerical precision. Consequently, a single execution of $\textsc{ApproxUnion}\big(\{B_r(x):x\in S\},\varepsilon/4\big)$ requires at most $T = 128\ln(8)\,(1+\varepsilon/4)\,\vert S\vert / \varepsilon^2$ random samples. 
This yields an overall runtime for Algorithm~\ref{alg:gr_approxunion} of $O\big(n\,\vert S\vert\,\varepsilon^{-2}\ln \delta^{-1}\big)$. 
Rather than using $\textsc{ApproxUnion}$ merely as a subroutine to estimate $\mu(B_r(S))$, one could instead implement a refined backtracking procedure throughout its execution to simultaneously produce an estimate of $\widehat N_x\,\mu(B_r(x)).$ 
Such an approach could, in principle, produce a multiplicative $\varepsilon$-approximation of $g_r(S)(x)$, improving upon the additive accuracy guarantee obtained in the current formulation.

We now turn to the estimation of $f_\nu(S)$ and introduce Algorithm~\ref{alg:fnu_estimate}, a simple two-stage Monte Carlo sampling method that relies on Algorithm~\ref{alg:gr_estimate} as a subroutine to estimate $g_{r_i}(S)$ for multiple radii $r_i$ drawn independently from $\nu$. 
The theoretical guarantees of this procedure, including explicit confidence and precision bounds, are established in Theorem~\ref{thm:fnu_estimate}.

\begin{algorithm}[!bth]
\caption{Naive Monte Carlo Estimation of $f_\nu(S)$}
\label{alg:fnu_estimate}
\begin{algorithmic}[1]

\State \textbf{Input: } Finite set $S \subset \mathbb{R}^n$, distribution $\nu$ on $[0,\alpha]$,   outer and inner-sample sizes $M,k \in \mathbb{N}$, respectively.
\vspace{0.5ex}
\State Sample $M$  radii $r_1,\dots,r_M \stackrel{\text{i.i.d.}}{\sim} \nu$.
\For{$j = 1,\dots,M$}
    \State Run Algorithm~\ref{alg:gr_estimate} with radius $r_j$ and $k$ samples.
    \State Obtain estimates $\hat g_{r_j}(S)(x)$ for all $x \in S$.
\EndFor
\vspace{0.5ex}
\State For each $x \in S$, compute the average
\[
   \hat f_\nu(S)(x) \;=\; \frac{1}{M}\sum_{j=1}^M \hat g_{r_j}(S)(x).
\]
\State \textbf{Output:} Estimates $\hat f_\nu(S)(x)$ for all $x \in S$.
\end{algorithmic}
\end{algorithm}

\begin{theorem}[Naive Monte-Carlo Estimation for $f_\nu$]\label{thm:fnu_estimate}
    For any target accuracy $\varepsilon >0$ and confidence  level $\delta \in(0,1)$, setting the  outer and inner sample sizes in Algorithm~\ref{alg:fnu_estimate} such that 
    $$ M \;\geq\; \frac{8}{\varepsilon^2}\ln\!\frac{4|S|}{\delta},
    \qquad
    k \;\geq\;  \frac{2( \vert S \vert ^2 - 1)^2}{\varepsilon^2 \vert S \vert ^2} \ln\!\frac{4 \vert S \vert M}{\delta}, $$
    ensures with probability at least $1-\delta$ the following bound for every $x\in S$, i.e.,
    $$\big \vert \widehat f_\nu(S)(x) - f_\nu(S)(x)\big\vert \le \varepsilon,$$ and the estimator is consistent.
    Moreover, if $k=k(M)=\Omega(\log M)$ (e.g.\ $k(M)$ chosen as above), then Algorithm~\ref{alg:fnu_estimate} also yields an asymptotically unbiased estimator, i.e.,
    $\lim_{M\to\infty}\mathbb E[\widehat f_\nu(S)(x)] = f_\nu(S)(x).$ 
\end{theorem}

The proof of Theorem~\ref{thm:fnu_estimate}, given in Appendix~\ref{sec:add_comput_res}, is standard and proceeds by separately controlling the errors of the inner and outer Monte Carlo estimates. 
Critically, the runtime of Algorithm~\ref{alg:fnu_estimate} is $M$ times that of Algorithm~\ref{alg:gr_estimate}, that is, $O(Mnk \lvert S \rvert^2 )$. However, $M$ must scale as $\propto \varepsilon^{-2}$ to control the outer error, which results in an overall runtime $\propto \varepsilon^{-4}$ to achieve an additive accuracy of $\varepsilon > 0$.
More sophisticated statistical techniques can further improve on this naive approach: as shown in Appendix~\ref{sec:add_comput_res}, separating the inner error into deviation and bias terms already reduces the total runtime to $\propto\varepsilon^{-3}$.

However, we identify an alternative direction to reduce the total runtime: reusing samples across multiple radii. 
This approach relies on two key observations. First, by sampling points from balls in decreasing order of radius, each point also serves as a uniform sample for all smaller balls it falls into. 
Second, once the radii are sorted, the depth of a point across all radii can be computed efficiently by performing a dichotomic search for the smallest radius such that the point lies within the corresponding ball, which then determines membership for all larger-radius balls. 
A concrete algorithm implementing these ideas is described in Appendix~\ref{sec:add_comput_res}, c.f. Algorithm~\ref{alg:fnu_reuse}.

In the worst case, Algorithm~\ref{alg:fnu_reuse} may still need to draw $k$ fresh samples for every radius and every center, for a total of $k M \vert S \vert$ samples. Each sample then requires $O(n\vert S \vert)$ operations to compute distances to all centers in $S$, as well as an extra $O(\vert S \vert \log M)$ operations for the binary search across radii. Altogether, this brings the total worst-case computational cost to $O( k M \vert S \vert^2 (n + \log M)) .$

In expectation, however, a substantial reduction in computation may be achieved through sample reuse. 
Suppose the radii are ordered in increasing order, i.e., $r_1 \leq \dots \leq r_M$. For any $j \leq i \in [M]$, the probability that a sample $z$ drawn uniformly from $B_{r_i}(x)$ also lies within $B_{r_j}(x)$  is exactly $(r_j / r_i)^n$. 
Let $T_j$ denote the (random) number of new samples drawn at radius $r_j$, after having already drawn samples for all $i>j.$
Conditioning the expectation on the realization $r_1,\dots,r_M$, we then have the recursive relation
$$
\mathbb{E}[T_j ] 
= k - \sum_{i>j} \mathbb{E}[T_i] \, (r_j/r_i)^n .
$$
In general, the closed-form solution will depend on the joint distribution of the ordered radii.
If we consider the case where $\nu$ is uniform on $[0, \alpha]$, however, then the expected value of the $j$-th order statistic becomes $\mathbb{E}[r_j] = \frac{j}{M+1} \alpha.$
Approximating the realized order $r_1 \leq \dots \leq r_M$  by their respective quantiles $r_j \approx \frac{j}{M+1} \alpha$ and substituting them in the above relation, we then show with a strong induction that 
$\mathbb{E}[T_j] \approx k \big( 1 - ( \tfrac{j}{j+1} )^n \big)$ for all $j < M$, with $\mathbb{E}[T_M] = k$. 
The total expected number of samples per element of $S$ is therefore
$$
\sum_{j=1}^M \mathbb{E}[T_j] 
\approx k \bigg( M - \sum_{j=1}^{M-1} \!\big( \tfrac{j}{j+1} \big)^n \bigg),
$$
which is equivalent to $ k ( n \log M + O(1) )$ in the limit $M\to \infty.$
Only the computation of the $M \vert S \vert$ averages incurs a cost of $O(k M \vert S \vert)$, so the total expected runtime of this sample-reuse algorithm would scale as $O\big(k \vert S \vert (M + \vert S \vert n^2 \log M)\big)$. For large $\vert S \vert$, this constitutes a substantial improvement over the $O(k \vert S \vert^2 M n)$ runtime of Algorithm~\ref{alg:fnu_estimate}.

While the above discussion primarily provides intuition for the potential benefits of sample reuse, a rigorous analysis of Algorithm~\ref{alg:fnu_reuse} is required and left for future work.
All in all, the algorithms introduced in this section demonstrate that efficient computation of the proposed weighting functions is already feasible in practice.
Besides, substantial efficiency gains are still achievable by combining ideas from Algorithms~\ref{alg:gr_approxunion} and~\ref{alg:fnu_reuse}, and incorporating more advanced variance-reduction techniques.

\section{Discussion}\label{sec:discu}

We gather in this section different remarks on our results as well as possible extensions of our work.

\paragraph{Extension to Perfect Clones.}

The framework we considered until now only allows for $\delta$-clones with $\delta>0$, but not perfect clones, i.e., with $\delta =0.$  The appropriate analytical tool to handle this is to consider a \emph{pseudo-metric space} $(E,d)$ instead of a metric one, where the pseudo-metric $d$ verifies \emph{non-negativity}, \emph{symmetry}, \emph{triangle inequality} like a full-fledged metric, but only verifies \emph{identity} instead of \emph{separability}. 
This exactly means that two different elements $x\neq y$ in $E$ may be perfect clones, i.e., $d(x,y) =0.$

The axioms used in the definition of $\mathcal{R}_\alpha(E,d)$ directly extend to a pseudo-metric space $(E,d)$; this is also the case for the representation functions $f_\nu$ in Theorem~\ref{thm:cont_convex_combi_gr}, when the space induced by the vanishing of the pseudo-metric is $\mathbb{R}^n.$ 
We refer the interested reader to Appendix~\ref{sec:metric_id} for more details.

\paragraph{Extension beyond Euclidean Spaces.}

While the solution proposed in Section~\ref{sec:local_voting} is restricted to Euclidean spaces, similar ideas could be applied in more general metric spaces. Using a Radon measure $\mu$, one could define the weighting functions $g_r$ in full generality and show similarly as in Theorem~\ref{thm:local_vote_rep_func} that Axioms~\ref{axi:pos}, \ref{axi:clone_fair_uni}, \ref{axi:indiv_cont} and~\ref{axi:alpha_clone_locality} hold. The real challenge however is to satisfy Axiom~\ref{axi:sym}. 

Indeed, our proof relies on two convenient properties of Euclidean spaces: first, the uplifting of self-isometry $\sigma_S$ to the entire space; second the invariance of the Lebesgue measure under translations, rotations and reflections.
What can be done without these properties?
The first problem could be entirely shunned by arguing that only full-fledged isometries should be considered in the definition of Axiom~\ref{axi:sym}. The second issue is however tougher to ward off. 
To extend invariance by translation beyond vector spaces, one should consider \emph{uniformly distributed measures}, i.e., measures that give the same weight to all balls of the same radius. However, such measures turn out to be very rigid objects and are uniquely defined up to a multiplicative constant in most metric spaces.

\begin{lemma}[From \cite{christensen_measures_1970}]\label{lem:unif_measure}
Let  $(E,d)$ be a locally compact metric space. There exists a Radon measure $\mu$ defined on the Borel $\sigma$-algebra of $E$ that is uniformly distributed, i.e., it verifies $0 < \mu(B_r(x)) = \mu(B_r(y)) < \infty $ for all $r>0$ and $x,y$ in $E.$ 
Moreover, this measure is unique up to a multiplicative constant if $E$ is separable.
\end{lemma}

As a particular example, this essentially implies that the Lebesgue measure is the only Borel measure invariant by translation on $\mathbb{R}^n$. This 
indicates that our approach is doomed even in the simple case of $\mathbb{R}^n$ endowed with the $L^1$ distance $d_1(x,y) = \sum_{i=1}^n \vert x_i - y_i \vert$, as illustrated in Figure~\ref{fig:l1_balls_intersection}.

\begin{figure}[!tbh]
    \centering
    \begin{tikzpicture}[scale=1.2]

        \definecolor{PaperBlue}{HTML}{1f77b4}
        \definecolor{PaperOrange}{HTML}{ff7f0e}
        \definecolor{PaperGreen}{HTML}{2ca02c}
        \definecolor{PaperRed}{HTML}{d62728}
        \definecolor{PaperPurple}{HTML}{9467bd}
        \definecolor{PaperBrown}{HTML}{8c564b}
        \definecolor{PaperPink}{HTML}{e377c2}
        \definecolor{PaperGray}{HTML}{7f7f7f}
        \definecolor{PaperOlive}{HTML}{bcbd22}
        \definecolor{PaperCyan}{HTML}{17becf}
        
        \coordinate (A) at (0,0);
        \coordinate (B) at (2,0);
        \coordinate (C) at (-1,1);

        \draw[ PaperRed, rotate around={45:(A)}] (A) ++(-1,-1) rectangle ++(2,2);
        \fill[PaperRed, opacity=0.07, rotate around={45:(A)}] (A) ++(-1,-1) rectangle ++(2,2);
        \draw[ PaperBlue, rotate around={45:(B)}] (B) ++(-1,-1) rectangle ++(2,2);
        \fill[PaperBlue, opacity=0.07, rotate around={45:(B)}] (B) ++(-1,-1) rectangle ++(2,2);
        \draw[ PaperGreen, rotate around={45:(C)}] (C) ++(-1,-1) rectangle ++(2,2);
        \fill[PaperGreen, opacity=0.07, rotate around={45:(C)}] (C) ++(-1,-1) rectangle ++(2,2);

        \draw[dashed, rotate around={45:(A)}] (A) ++(-1.414,-1.414) rectangle ++(2.828,2.828);

        \draw[->, dotted] (-2.8,0) -- (3.8,0) ;
        \draw[->, dotted] (0,-2.2) -- (0,2.6) ; 

        \fill[PaperRed] (A) circle (1pt) node[above right] {$x$};
        \fill[PaperBlue] (B) circle (1pt) node[below right] {$y$};
        \fill[PaperGreen] (C) circle (1pt) node[above left] {$z$};

    \end{tikzpicture}
    \caption{The weighting function $g_r$ does not satisfy Axiom~\ref{axi:sym} in $(\mathbb{R}^2, d_1).$ As illustrated by the dashed $L^1$ ball centered in $x$, points $y$ and $z$ are indeed at the same distance of $x$, thus belong in a common isometry class in $S=\{x,y,z\}$ and should receive similar weights under Axiom~\ref{axi:sym}. Note however that the Lebesgue measure, i.e., the area, of the intersection between the red and the green ball differs from that of the intersection between the red and the blue ball, hence $g_r(S)(y) \neq g_r(S)(z).$ }
    \Description{Asymmetry of $g_r$ in the $(\mathbb{R}^2,d_1)$ metric. 
    The figure shows three points $S=\{x,y,z\}$ together with their $L^1$-balls of equal radius, represented as diamonds (rotated squares). 
    The dashed diamond centered at $x$ indicates that $y$ and $z$ lie at the same $L^1$-distance from $x$, hence belong to the same isometry class in $S$ and should receive identical weights under Axiom~\ref{axi:sym}. 
    However, the Lebesgue measure of the intersection between the balls centered at $x$ and $y$ differs from that between the balls centered at $x$ and $z$. 
    Since $g_r(S)(\cdot)$ depends on these area overlaps, this geometric imbalance implies $g_r(S)(y)\neq g_r(S)(z)$, illustrating the failure of the symmetry axiom in $(\mathbb{R}^2,d_1)$.}
    \label{fig:l1_balls_intersection}
\end{figure}

Such metric spaces thus require developing techniques different from the ones introduced in this work. 
\emph{Topologically independent} weighting functions would provide an elegant solution to this issue, i.e., functions that do not rely on the topological properties of $(E,d)$, but rather solely depend on the distance matrices associated with each finite set.

\begin{axiom}[Topological Invariance]\label{axi:topo_inv}
Weighting only depends on the distance matrix associated with each finite set, i.e., 
there exists a family $(h_n)_{n\geq1}$ with $h_n: \mathbb{R}^{n\times n} \mapsto \Delta(n)$ such that, for all $S \in \mathcal{P}(E)$ of cardinality $\vert S\vert =n$, we have $f(S) = h_n(M(S))$, where $M = (d(x,y))_{x,y \in S} \in \mathbb{R}^{n\times n}$ denotes the distance matrix associated to $S$ and $d$, unique up to permutations.
\end{axiom}

Identifying weighting functions within $\mathcal{R}_\alpha(E,d)$ that adhere to Axiom~\ref{axi:topo_inv} is a promising direction for future work.

\section*{CRediT Author Statement}
\balance
\textbf{Damien Berriaud}:  Conceptualization, Methodology, Formal Analysis, Investigation, Writing – Original Draft.
\textbf{Roger Wattenhofer}: Writing – Review \& Editing, Supervision, Funding Acquisition.

\newpage

\begin{acks}
A preliminary version of this paper was presented at GAIW'25; we thank the workshop attendees for their interesting questions and remarks. We thank the anonymous reviewers for their constructive feedback and useful suggestions, which helped improve the paper.
\end{acks}

\bibliographystyle{ACM-Reference-Format}
\bibliography{Cloneproofness}

\appendix

\section{Further Related Work}\label{sec:further_related_works}

The framework introduced in this work may prove useful in a range of settings; we introduce hereafter two additional potential applications.

(i) In machine learning, one may want to tackle class imbalance in multi-label classification problems, e.g., by giving weights to the individual contribution of each sample to the loss.
More generally, reweighing a loss function
according to a clone-proof weighting
can be thought of as a distribution-agnostic importance sampling technique. As such, it could be beneficial whenever the most informative samples become increasingly difficult to obtain,
either because of the \emph{computational cost} associated with computing the associated true label, e.g., for protein folding \cite{jumper_highly_2021}, Graph Neural Networks-based SAT solvers \cite{wang_neuroback_2024}, climate modeling \cite{eyring_pushing_2024};
or simply because of their \emph{relative rarity}, e.g., for rare disease diagnosis in medical image analysis \cite{banerjee_machine_2023},
fraud detection in financial systems \cite{motie_financial_2024}, low-resources languages in Natural Language Processing \cite{hedderich_survey_2021}.

(ii) Clone-proof weightings offer a novel set of tools to mitigate Sybil attacks in distributed systems. While reputation mechanisms are designed to promote cooperation, significant research has focused on preventing adversaries from exploiting these systems by creating fake identities that mutually reinforce one another \cite{resnick_sybilproof_2009,seuken_sybil-proof_2014,stannat_achieving_2021}.
Similarly, researchers have sought to design mechanisms that incentivize information diffusion within social networks without encouraging the creation of fraudulent identities \cite{babaioff_bitcoin_2016,chen_sybil-proof_2023,zhang_sybil-proof_2020}. In both scenarios, clone-proof representations could help regulate the influence of Sybils once detected, effectively shifting the challenge to identifying these artificially generated identities.

We now highlight connections with several related bodies of literature.

\paragraph{Domain Adaptation and Samples Reweighting.}

In traditional learning setups, training and testing data are assumed to follow the same distribution. \emph{Domain adaptation} \cite{wang_deep_2018}, however, addresses scenarios where this assumption is violated, such as in the presence of class imbalance or label noise \cite{torralba_unbiased_2011}.

One approach to handle biases in training datasets involves assigning weights to individual samples and minimizing a weighted loss. Classical algorithms, such as AdaBoost \cite{freund_decision-theoretic_1997}, hard-negative mining \cite{chang_active_2018},
self-paced learning \cite{jiang_self-paced_2015}, adapt these weights dynamically during training based on the observed training loss. In contrast, the \emph{meta-learning} framework \cite{jamal_rethinking_2020,ren_learning_2019,shu_meta-weight-net_2019}
iteratively optimizes the weighting 
to minimize loss on a small, unbiased validation dataset. Diverging from these methods, we consider a one-shot scenario, where the sample weighting is determined a priori and remains fixed.

In the context of imbalanced classification and long-tailed datasets, reweighting techniques have been explored extensively \cite{cao_learning_2019,dong_class_2017,gebru_fine-grained_2017}. These methods typically assign weights inversely proportional to the number of instances in each class. Recent approaches, such as those proposed in \cite{cui_class-balanced_2019}, go further by accounting for data overlap. They suggest weighting samples based on the effective number of samples, under the intuition that the marginal benefit of adding a new sample diminishes as the sample count grows. Specifically, they expand each data point to include its surrounding neighborhood and define the informativeness of a sample as the additional coverage 
it contributes compared to the scenario where the sample is excluded.
Our theoretical framework, in particular the locality axiom (see Axiom~\ref{axi:alpha_clone_locality}) and the volume-based construction 
(see Section~\ref{sec:local_voting}), draws a close parallel to their total volume of sampled data, but extends this intuition beyond simple classification problems.

More generally, reweighting is a key technique in addressing \emph{covariate shift} within domain adaptation. Covariate shift occurs when the input distribution differs between training and evaluation datasets, i.e., $P_{train}(x) \neq P_{test}(x)$, but the conditional distribution $P_{train}(y \vert x) = P_{test}(y \vert x)$ remains consistent. Originating in importance sampling -- a technique commonly used to reduce variance in Monte Carlo estimation -- different methods \cite{y_covariate_2019} tackle covariate shift by reweighting samples with the ratio $P_{test} (x) / P_{train}(x).$ Approaches such as \emph{Kernel Density Estimation} \cite{hardle_nonparametric_2004} approximate these distributions using Gaussian kernels but suffer from the curse of dimensionality. Alternatively, \emph{Kernel Mean Matching} \cite{gretton_covariate_2009} minimizes the discrepancy between training and test distributions by aligning their means in a reproducing kernel Hilbert space, effectively estimating $P_{test} (x) / P_{train}(x) $ directly.
In contrast to these statistical methods, our approach does not rely on assumptions about the stochasticity of the sampling process. Instead, we adopt an axiomatic framework that provides robustness guarantees even when samples are adversarially selected. This makes our method more resilient to challenges like dataset poisoning \cite{carlini_poisoning_2024}.

\paragraph{Metric Learning and Hierarchical clustering.}

Metric learning and clustering techniques adopt fundamentally opposing philosophies. On the one hand, metric learning generally assumes access to ground-truth labels of similarity, e.g., whether points belong to the same class, and seeks to derive a distance metric, often within the class of generalized Mahalanobis metrics, that best separates dissimilar points while bringing similar ones closer \cite{ghojogh_spectral_2022}. 
On the other hand, clustering assumes some ground truth distance metric and aims to recover a notion of class by grouping similar points into clusters.
Our framework aligns more closely with clustering techniques through its shared starting assumption -- the availability of an informative distance metric. However, it diverges in its objective, focusing instead on the implications of similarity for 
unbiased weighting of 
data points.

Still, clustering techniques may represent a useful step toward this goal: one could first group points into clusters, assign clusters equal weights, and then share these weights uniformly within each cluster.
Such ``hard'' clusters may however lack the smooth properties we aim for. These limitations can be mitigated by adopting hierarchical clustering techniques, where points can belong to multiple clusters arranged in a tree structure (dendrogram) \cite{murtagh_algorithms_2017,ran_comprehensive_2023}. This approach enables contributions across different scales, akin to the role of the probability distribution $\nu$  in Theorem~\ref{thm:cont_convex_combi_gr}.

One may even consider more flexible structures, such as \emph{Fuzzy Hierarchical Clustering (FHC)} or \emph{Overlapping Hierarchical Clustering (OHC)}.
FHC \cite{varshney_pifhc_2022} 
allows points to have partial membership in several clusters at once, with the sum of memberships normalized across clusters. 
OHC \cite{berthold_overlapping_2020}, on the other hand, constructs directed acyclic graphs of clusters (quasi-dendrograms) instead of traditional trees, enabling a soft merging process. This approach allows points to have full membership in multiple clusters simultaneously, letting clusters overlap without the need for fuzzy memberships.

In any case, transitioning from clusters of points to individual weights becomes a non-trivial task for ``soft'' clusters. In this work, we move away from the concept of clusters and allow for non-transitive similarity relations.

\section{Preliminaries \& Notations}\label{sec:def}

In this section, we provide an overview of the notations used throughout the paper and introduce the key tools necessary for our demonstrations.

\paragraph{Metric spaces and Transport norm.}

Let $E$ be a set and let $d: E\times E \mapsto \mathbb{R}_{\geq0}$ be a metric on $E$, that is an operator satisfying for all $x,y,z \in E$
\begin{enumerate}
    \item (\emph{Non-negativity}) $d(x,y) \geq 0$ ;
    \item (\emph{Symmetry}) $d(x,y) = d(y,x)$;
    \item (\emph{Triangle inequality}) $d(x,z) \leq d(x,y) + d(y,z)$ ;
    \item (\emph{Separability}) $d(x,y) =0 \iff x =y.$ 
\end{enumerate}

For $k \in \mathbb{N}$, let $\mathcal{P}_k(E) := \{ S \subseteq E \mid \vert S \vert = k \}$ denote the powerset of subsets of $E$ of cardinality $k$; we further denote by $\mathcal{P}(E) := \bigcup_{k\geq1} \mathcal{P}_k(E)$ the set of finite subsets of $E.$ In particular,  $\mathcal{P}(E)$ does not contain the empty-set.

\smallskip
We equip $\mathcal{P}(E)$ with the \emph{transport distance} $d_\Pi$, defined for two finite subsets $X, Y$ in $\mathcal{P}(E)$ with cardinality $\vert Y \vert \geq \vert X\vert$ as
$$d_\Pi(X,Y)= d_\Pi(Y,X) = \min_{\pi \in \mathrm{Surj} (Y, X)} \max_{y \in Y} d(y, \pi(y)),$$
where $\mathrm{Surj}(Y, X)$ denotes the set of surjections from $Y$ over $X.$
We similarly denote by 
$$\Pi(Y,X)= \Pi(X,Y) = \argmin_{\pi \in\mathrm{Surj}(Y, X)} \max_{y \in Y} d(y, \pi(y))$$ the set of \emph{minimal transport maps} from $Y$ to $X$ that achieve the distance in $d_\Pi.$
We prove in Section~\ref{sec:metric_cont&axioms_disc} that $d_\Pi$ constitutes a metric on $\mathcal{P}(E).$

A metric space is finally defined as an ordered pair where the first element is a set and the second is a metric on this set; $(E,d)$ and $(\mathcal{P}(E), d_H)$ constitute the two most prominent examples of metric spaces in this work.

\paragraph{Isometries and self-isometries.}

In the metric space $(E,d)$, an isometry is defined as a map $\sigma: E \mapsto E$ that preserves distances, i.e., that verifies for all $x,y$ in $E$, 
$$d(\sigma(x),\sigma(y)) = d(x,y).$$

For $X \in \mathcal{P}(E)$ a finite subset of $E$, we moreover refer to an isometry in the subspace induced by $X$ as a \emph{self-isometry on $X$}, i.e., a permutation $\sigma_X: X\mapsto X$ such that $d(\sigma_X(x),\sigma_X(y)) = d(x,y)$ for all $x,y$ in $X.$

\smallskip
While restricting an isometry $\sigma$ to a subset $X \subseteq E$ gives a self-isometry on $X$, note that the converse may not hold in general. For example, consider the discrete metric space $(E,d_\text{path})$, where $E =\{v_1, v_2,v_3\}$ and  $d_\text{path}$ is the shorest path distance on the path $P = (v_1, v_2, v_3).$ Note that the transposition $\tau_S$ that swaps $v_1$ and $v_2$ is a self-isometry on  $S = \{v_1,v_2\}$, but its extension to $E$ is not an isometry.

However, we show in Appendix~\ref{sec:self-iso_eucli} that the converse holds in the particular case of Euclidean spaces, and that a self-isometry $\sigma_X$ on $X \subseteq E$ may be extended to a full-fledged isometry $\sigma$ on the entire space $E.$

\paragraph{Topology.}

In a metric space $(E,d)$, we denote by $B_r(x)$ the open ball of radius $r>0$ centered in $x \in E$, that is the set $B_r(x) := \{y \in E \mid d(x,y) < r\}.$ When $X$ is a set of points in $E$, we extend this definition and write $B_r(X) := \bigcup_{x\in X} B_r(x)$ for the union of the open balls of radius $r>0$ centered at each element of $X.$

\smallbreak
A set $X\subseteq E$ is then said to be \emph{open} if it contains a ball of positive size centered in each of its elements, i.e., for all $x \in X$, there exists $r>0$ such that $B_r(x) \subseteq X.$ 
On the contrary, we say that a set $X$ is \emph{closed} when its complement $X^c = E \setminus X$ is open. 
We next define the \emph{closure} $\overline{X}$ of a set $X$ as the smallest closed set that contains $X$, and its \emph{interior} $\mathring{X}$ as the largest open set contained within $X$.
We finally define the \emph{interior} of a set $X$ as $\partial X:= \overline{X} \setminus \mathring{X}.$

\paragraph{Hausdorff measure.}

A \emph{$\sigma$-algebra on $E$} is a non-empty collection of subsets of $E$ closed under complement, countable union and countable intersections. In particular, the \emph{Borel $\sigma$-algebra} $\Sigma$ is the smallest $\sigma$-algebra by set inclusion containing all open sets of $E.$

We refer to the ordered pair $(E, \Sigma)$ as a \emph{measurable space}, and define a \emph{measure} as a function $\mu : \Sigma \mapsto \mathbb{R}\cup\{ \pm \infty \}$ that verifies 
\begin{enumerate}
    \item \emph{Non-negativity:} $\mu(X) \geq 0$ for all $X\in \Sigma$;
    \item  $\mu(\emptyset) =0$;
    \item \emph{Countable additivity:} $\mu\big ( \bigcup_{k\in\mathbb{N}} X_k \big) = \sum_{k\in\mathbb{N}} \mu(X_k)$ for all countable collection $\{X_k\}_{k\in\mathbb{N}}$ of pairwise disjoint sets in $\Sigma.$
\end{enumerate}

The \emph{Hausdorff measure} is particularly important, we recall its definition hereafter.
For $U$ a subset of $E$, we first define the \emph{diameter of $U$} as 
$$\operatorname{diam}(U) = \sup \{ d(x,y) \mid x,y, \in U \}.$$
We moreover adopt the convention $\operatorname{diam}(\emptyset)=0.$
For two positive real number $\delta$ and  $m$, as well as a subset $X \in \Sigma$, we further write
\begin{equation*}
\begin{aligned}
    \mathcal{H}^m_\delta(X) = \inf \Bigg\{ \sum_{k\in\mathbb{N}} \operatorname{diam}(U_k)^m \mid &X\subseteq \bigcup_{k\in\mathbb{N}} U_k, 
    \operatorname{diam}(U_k)<\delta \Bigg\},
\end{aligned}
\end{equation*}
where the infimum is taken over countable collections $\{U_k\}_{k\in\mathbb{N}}$ that cover $X$ with sets of diameter smaller than $\delta.$
The \emph{$m$-dimensional Hausdorff measure} is then finally defined as 
$$ \mathcal{H}^m(X) = 2^{-m} \alpha_m \cdot  \lim_{\delta \to 0}  \mathcal{H}^m_\delta(X) ,$$
where $\alpha(m) = \frac{\pi^{m/2}}{\Gamma(m/2+1)} $ denotes the volume of the unit $m$-ball, and $\Gamma$ represents Euler's gamma function. 

When $m$ is an integer, the scaling $2^{-m} \alpha_m$ ensures that the $m$-dimensional Hausdorff measure coincides with the classical Lebesgue measure on the Borel sets of an $m$-dimensional Euclidean space.

\smallbreak
For an Euclidean space $(E,d) = (\mathbb{R}^n, d_2) $, we can therefore write the $n-1$-dimensional surface of a ball of radius $r>0$  as 
$$ S^{n-1}_r = \mathcal{H}^{n-1} \big(B_r(x) \big),$$
where $x$ may be any point in $E$ since the Lebesgue measure, hence also the Hausdorff measure, are invariant by translation.
Denoting by $V^n_r$ the $n$-dimensional volume of a ball of radius $r>0$, we moreover have $V^n_r = \alpha(n) r^n $ as well as the relation $S^{n-1}_r = \frac{n}{r} V^n_r. $

\paragraph{Geometric measure theory.}

For two subsets $X,Y$ of $E = \mathbb{R}^n$, we define the \emph{Minkowski sum $X+Y$} as follows, i.e., $X+Y = \{ x+y \mid x \in X, y \in Y\}.$

\smallbreak
Let $o = (0,\dots,0)$ denote the origin and let $m$ be an integer verifying $0\leq m \leq n.$ We next define the \emph{$m$-dimensional Minkowski content of $X$} as
$$\mathcal{M}^{*m} (X) = \limsup_{\delta \to 0} \frac{ \mu\big( X + B_\delta(o) \big)}{\alpha(n-m) \delta^{n-m}},$$
where $\mu$ denotes the $m$-dimensional Lebesgue measure, and $X + B_\delta(o))$ is a Minkowski sum.
The \emph{$m$-dimensional lower Minkowski content } $\mathcal{M}_*^{m} (X)$ is similarly defined by replacing the $\sup$ by an $\inf$ in the definition of  $\mathcal{M}^{*m} (X).$

When the upper and lower $m$-dimensional Minkowski contents are equal, we call their common value $\mathcal{M}^{m} (X)$ the \emph{ $m$-dimensional Minkowski content of $X.$}

In particular for $m=n-1$, we have $\alpha(1)=2$ and we get
$$\mathcal{M}^{n-1} (X)  = \lim_{\delta \to 0} \frac{ \mu \big( X + B_\delta(o) \big)}{2\delta}.$$

A set $X$ is said to be \emph{$m$- rectifiable} if and only if there exists a Lipschitz function mapping some bounded subset of $\mathbb{R}^m$ onto $X.$ 
\smallbreak
With these definitions in place, we can now present a fundamental result in geometric measure theory that establishes a connection between the Minkowski content of well-behaved sets and their Hausdorff measure.

\begin{theorem}[From \cite{federer_geometric_1996}, Thm 3.2.39]
    If $X$ is a closed $m$-rectifiable set of $\mathbb{R}^n$, then 
    $\mathcal{M}^{m} (X) = \mathcal{H}^m(X).$
\end{theorem}

\section{Proofs of the Main Results}\label{sec:proof}

This section gathers the proofs of Theorems~\ref{thm:local_vote_rep_func} and~\ref{thm:cont_convex_combi_gr}.
Most of the technical tools required for the demonstrations are introduced in Appendix~\ref{sec:def}. Before delving into the proof of Theorem~\ref{thm:local_vote_rep_func}, let us first recall its formulation.

\begin{reptheorem}{thm:local_vote_rep_func}
For $r>0$, the weighting function $g_r$ is well-defined and belongs in $\mathcal{R}_{2r}(\mathbb{R}^n, d_2).$
\end{reptheorem}

\begin{proof}
     Let $r>0$ be a positive radius, we first verify that $g_r$ is a well-defined weighting function. 
    Let $S$ be a finite subset of $\mathbb{R}^n$, and $x$ be an element of $S$, we show that $ g_{r,S,x}$ is measurable with respect to the Lebesgue measure $\mu$. 
    Indeed, consider the following partition of its domain
    \begin{equation*}
        \begin{aligned}
            B_r(S) 
            &= \bigcup_{U \subseteq S} A_r(U),
        \end{aligned}
    \end{equation*}
    where $A_r(U) = \bigcap_{u\in U} B_r(u) \bigcap_{v \in S\setminus U} \big( B_r(S) \setminus B_r(v) \big)$ belongs in the Borel $\sigma$-algebra for all choice of $U \subseteq S$. 
    Note that this partition is finite since there are at most $2^{\vert S \vert }$ choices for the subset $U.$ 

    Moreover,  $g_{r,S,x}$ is null on $B_r(S)  \setminus B_r(x)$ and, for all $U \subseteq S$ with $x\in U$ and all $y$ in $A_r(U)$, we have $g_{r,S,x}(y) = \frac{1}{\vert U \vert}.$ Using the above partition, we rewrite 
    \begin{equation}\label{equ:grade_simple_f}
        g_{r,S,x}: y \in B_r(S) \mapsto \sum_{ \{x\} \subseteq U \subseteq S } \frac{\mathbb{1}_{A_r(U)}(y) } { \vert U \vert},
    \end{equation}
     and recognize a simple non-negative function. As such,  $g_{r,S,x}$ is both measurable and integrable, and $g_r(S)$ is well-defined. 
     Moreover, we verify that $g_r(S)$ is normalized.
    \begin{equation*}
        \begin{aligned}
            \sum_{x\in S} g_r(S)(x)
            &= \sum_{x\in S} \int_{B_r(S)} \frac{g_{r,S,x}}{\mu(B_r(S))} \:d\mu ,\\
            &= \int_{B_r(S)} \frac{1}{\mu(B_r(S))}    \sum_{x\in S} \frac{ \mathbb{1}_{B_r(x)}(y) }{\sum_{z \in S} \mathbb{1}_{B_r(z)}(y) }\: d\mu(y), \\
            &= \int_{B_r(S)} \frac{ 1}{\mu(B_r(S))} \: d\mu =1.
        \end{aligned}
    \end{equation*}
    
    Since $g_r(S)$ is also non-negative, it is a probability distribution over $S$, hence $g_r$ is indeed a well-defined weighting function over $\mathcal{P}(\mathbb{R}^n)$.

    \medskip
     \textbf{Axiom~\ref{axi:pos}.} Let $S$ be a finite subset of $\mathbb{R}^n$, and $x$ be an element of $S.$ Using the expression of $g_{r,S,x}$ in Equation~\eqref{equ:grade_simple_f}, we get for each $y$ in $B_r(x)$ that $g_{r,S,x}(y) \geq \frac{1}{\vert S\vert}.$ Since $\mu\big(B_r(S)\big) \leq \vert S \vert \cdot \mu\big(B_r(x)\big) $ by countable additivity and uniformity of $\mu$, we finally obtain
     \begin{equation}\label{equ:proof_gr_pos}
         \begin{aligned}
             g_r(S)(x) 
             &\stackrel{(a)}{=} \int_{B_r(x)} \frac{g_{r,S,x}}{\mu\big(B_r(S)\big)} \:d\mu, \\
             &\geq \frac{1}{\vert S \vert} \int_{B_r(x)} \frac{1}{\mu\big(B_r(S)\big)} \:d\mu, \\
             &= \frac{1}{\vert S \vert} \frac{\mu\big(B_r(x)\big)}{\mu\big(B_r(S)\big)} \geq \frac{1}{\vert S \vert^2}.
         \end{aligned}
     \end{equation}
     Equality $(a)$ holds since  $g_{r,S,x}$ is null on $B_r(S)  \setminus B_r(x).$
      Hence we have $g_r(S)(x) >0$ for all $x \in S$, and Axiom~\ref{axi:pos} holds.
    
    \medskip
    \textbf{Axiom~\ref{axi:sym}.} Let $S\subset \mathbb{R}^n$ be a finite subset,  $\sigma_S :S \mapsto S$ be a self-isometry on $S$, and $x$ be an element of $S.$ 
    Since $g_{r,S,x}$ is a simple function, we deduce from Equation~\eqref{equ:grade_simple_f} the following, i.e., 
    \begin{equation*}
        \begin{aligned}
            g_r(S)(\sigma(x))
            &= \frac{1}{\mu\big(B_r(S)\big)} \sum_{\{\sigma(x)\} \subseteq U \subseteq S} \frac{\mu\big(A_r(U)\big) } { \vert U \vert},\\
            &\stackrel{(a)}{=}  \frac{1}{\mu\big(B_r(S)\big)} \sum_{\{x\} \subseteq V \subseteq S } \frac{\mu\big(A_r(\sigma_S(V))\big) } { \vert \sigma_S(V) \vert },\\
            &\stackrel{(b)}{=}  \frac{1}{\mu\big(B_r(S)\big)} \sum_{ \{x\} \subseteq V \subseteq S} \frac{\mu\big(T\circ L(A_r(V))\big) } { \vert V \vert},\\
            &\stackrel{(c)}{=}  \frac{1}{\mu\big(B_r(S)\big)} \sum_{\{x\} \subseteq V \subseteq S} \frac{\mu\big(A_r(V))\big) } { \vert V \vert },\\
            &= g_r(S)(x).
        \end{aligned} 
    \end{equation*}
    Equality $(a)$ holds by writing $V = \sigma_S^{-1}(U) = \{ \sigma_S^{-1}(u) \mid u \in U\}$ and noting that $V$ contains $ \{x\}$.
    Equality $(b)$ uses that $\vert \sigma_S(V) \vert = \vert V \vert $ and the decomposition of the isometry $\sigma_S$ as a linear transformation $L: x\mapsto Qx$ and a translation $T:x \mapsto x + t$, where $Q$ and $t$ are the orthogonal matrix and the vector of Lemma~\ref{lem:isometry_rigid_transformation}. Using that $T\circ L$ is an isometry on $\mathbb{R}^n$, we then rewrite $A_r(\sigma_S(V)) = T\circ L (A_r(V)).$
    
    Finally, Equality $(c)$ follows from the invariance of the Lebesgue measure by translation $\mu\big( T\circ L (A_r(V))\big)= \mu\big( L (A_r(V))\big)$, as well as its behavior under linear transformation $\mu\big( L (A_r(V))\big) = \vert \text{det}(Q)\vert \cdot \mu\big( A_r(V)\big) = \mu\big( A_r(V)\big) $, where we used that $Q$ is orthogonal.
    Hence Axiom~\ref{axi:sym} holds.

    \medskip
    Since the proofs of Axioms~\ref{axi:clone_fair_uni}, \ref{axi:indiv_cont} and~\ref{axi:alpha_clone_locality} use similar techniques, we first introduce the necessary tools in the more complex case of Axiom~\ref{axi:indiv_cont}, and later use similar arguments to show Axiom~\ref{axi:alpha_clone_locality}.

    \medskip 
    \textbf{Axiom~\ref{axi:indiv_cont}.}

    Let $k \in \mathbb{N}$ be a positive integer and let $\delta$ satisfy $r > \delta > 0.$
    Let $X,Y$ be two subsets of $\mathbb{R}^n$ of cardinality $\vert X \vert = \vert Y \vert = k$ such that $d_\Pi(X,Y) \leq \delta$, and let $\pi \in \Pi(X,Y)$; note that we have $d(y, \pi(y)) \leq \delta$  for all $y$ in $Y.$ 

    For $U$ a subset of $X$, we denote by $\partial A_r(U)$ the boundary of $A_r(U).$ We moreover associate with $A_r(U)$ its ``thick interior'' $A_r^{-\delta}(U) = \big\{ x \in S \mid d(x, \partial A_r(U)) > \delta \big\} $, as well as its ``thick closure'' $A_r^{+\delta}(U) = \big\{ x \in \mathbb{R}^n \mid d(x, A_r(U)) \leq \delta \big\}.$ Note that these definitions allow to obtain the ``thick boundary'' of $A_r(U)$ by set difference, i.e., $ \partial A_r(U) + \overline{B}_\delta(o) = A_r^{+\delta}(U) \setminus A_r^{-\delta}(U) .$

     \smallskip
     We next show that $A_r^{-\delta}(U) \subseteq  A_r(\pi(U))$; let $z$ be an element of $A_r^{-\delta}(U).$
     \begin{itemize}
         \item  Let $u$ be an element of $U$ and let $z'$ be the projection of $z$ on the boundary of $B_r(u).$ Note in particular that $z$ belongs to the segment $[z',u].$
        This implies that $d(z,u) = d(z',u) - d(z,z') < r -\delta$ by definition of $A_r^{-\delta}(U).$ By triangle inequality, we then get $d(z,\pi(u)) \leq d(z,u) + d(u,\pi(u)) \leq \delta.$  Note that this holds for all choices of $u$ in $U$, hence also for all $\pi(u)$ in $ \pi(U).$ 

        \item Now let $v$ be an element of $X\setminus U$ and let $z'$ be the projection of $z$ on the boundary of $B_r(v).$ Note this time that $z'$ belongs in the segment $[z,v]$ and we similarly obtain $d(z,v) = d(z',z) + d(z',v) > r+\delta > r-\delta$ using the definition of  $A_r^{-\delta}(U).$ The triangle inequality again gives $d(z,\pi(v)) \geq  d(z,v) - d(v,\pi(v)) \geq \delta$ for all $\pi(v) \in Y \setminus \pi(U).$ 
     \end{itemize}
     Together, this implies that $z$ lies in  $A_r(\pi(U))$ and we conclude that $A_r^{-\delta}(U) \subseteq  A_r(\pi(U)).$

     \smallskip
    Since $A_r^{-\delta}(U)$ is also a subset of $A_{r}(U)$,
    Equation~\eqref{equ:grade_simple_f} implies, for all $x$ in $X$ and $z$ in $A_r^{-\delta}(U)$, that $g_{r,X,x}(z) = \mathbb{1}_U(z) /  \vert U \vert  = g_{r,Y,\pi(x)}(z).$
    Since this holds for all choices of $U \subseteq X$, we finally get that $g_{r,X,x}  = g_{r,Y,\pi(x)}$ on $A_r^{-\delta,X} := \bigcup_{U \subseteq X} A_r^{-\delta}(U).$

    \smallskip
    We now bound the difference of weighting between $x$ in $X$ and $\pi(x)$ in $Y$ using the partition of $B_r(X \cup Y)$ induced by $A_r^{-\delta,X}.$

    \begin{equation}\label{equ:bound_rep_diff_X&Y}
        \begin{aligned}
            &\big\vert g_r(X)(x) - g_r(Y)(\pi(x)) \big\vert \\
            &= \bigg\vert \int_{B_r(X)} \frac{g_{r,X,x}} {\mu\big(B_r(X)\big)} d\mu  - \int_{B_r(Y)} \frac{g_{r,Y,\pi(x)}} {\mu\big(B_r(Y)\big)} d\mu \bigg\vert ,\\
            &\stackrel{(a)}{\leq} \bigg \vert \frac{\mu\big(B_r(X)\big)}{\mu\big(B_r(Y)\big)} -  1 \bigg\vert \cdot \bigg 
             \vert\int_{A_r^{-\delta,X}} \frac{g_{r,X,x}} {\mu\big(B_r(X)\big)} d\mu \bigg\vert \\
            &+  \int_{ B_r( X\cup Y )\setminus A_r^{-\delta,X}} \frac{1} {\mu\big(B_r(x)\big)}  d\mu , \\
            &\stackrel{(b)}{\leq}   \frac{\mu\big(B_{r + \delta}(X)\big) - \mu\big(B_{r}(X)\big)} {\mu\big(B_r(\pi(x))\big)}  
            +  \frac{ \mu\big(B_{r + \delta}( X )\big) -   \mu\big(A_r^{-\delta,X}\big)}{\mu\big(B_r(x)\big)} ,\\
            &\stackrel{(c)}{\leq}  \frac{ 2 \cdot \mu\big( \partial B_{r}(X) + \overline{B}_\delta(o)\big) }{\mu\big(B_r(x)\big)}  ,\\
            &\stackrel{(d)}{\leq}   \frac{ 2 \cdot \vert X \vert \cdot \mu\big( \partial B_{r}(o) + \overline{B}_\delta(o)\big) }{V^n_r}   .\\
        \end{aligned}
    \end{equation}
    Inequality $(a)$ follows from the functional equality $g_{r,X,x}  = g_{r,Y,\pi(x)}$ on $A_r^{-\delta,X}$ and the fact that the two functions have their image in $[0,1]$ otherwise; it also uses that both $\mu\big(B_r(X)\big)$ and $\mu\big(B_r(Y)\big)$ are greater than $\mu\big(B_r(x)\big).$
    We establish inequality $(b)$ by noting that $B_r(\pi(x)) \subseteq B_r(Y) \subseteq B_{r+\delta}(X)$ and that the integral in the first term is bounded by one.
    Inequality $(c)$ uses the invariance of the Lebesgue measure by translation, its non-negativity and countable additivity, as well as the inclusion $ B_r(X) \supseteq A_r^{-\delta,X}$ and the set equality $\overline{B}_{r+\delta}(X) \setminus A_r^{-\delta, X} = \partial B_r(X) + \overline{B}_\delta(o).$
    Finally, inequality $(c)$ holds by writing $V_r^n$ for the volume of the $n$-dimensional Euclidean ball of radius $r$, and using Boole's inequality with $ \partial B_r(X) + \overline{B}_\delta(o) = \bigcup_{x\in X} \Big(\partial B_r(x) + \overline{B}_\delta(o) Big)$, as well as the translation invariance of the Lebesgue measure.

     Remark that the boundary $ \partial B_{r}(o)$ is a smooth $n-1$ manifold, therefore it is $n-1$ rectifiable. 
    By \cite[Theorem 3.2.39]{federer_geometric_1996}, the Minkovski content of the boundary $ \mathcal{M}^{n-1}\big(\partial B_{r}(o)\big)$ exists and is equal to its $n-1$ dimensional Hausdorff measure $ \mathcal{H}^{n-1}\big(\partial B_{r}(o) \big)$, i.e.,

    \begin{equation}\label{equ:Minkowski}
    \begin{aligned}
        \mathcal{M}^{n-1}\big(\partial B_{r}(o) \big)
         &=\lim_{\delta \to 0} \frac{ \mu\big(\partial B_{r}(o) + B_\delta(o)\big)  }{2 \delta} 
         = \mathcal{H}^{n-1} \big(\partial B_{r}(o)\big).
    \end{aligned}
    \end{equation}
    
    We can hence take $\delta$ small enough, but independent of $X$, so as to satisfy

    \begin{equation}\label{equ:cont_bound_diff_with_Hausdorff}
        \mu\big( \partial B_{r}(o) + \overline{B}_\delta(o)\big) \leq 4 \delta \mathcal{H}^{n-1}(\partial B_{r}(o)).
    \end{equation}

    With such a choice of $\delta$, Equation~\eqref{equ:bound_rep_diff_X&Y} then becomes, i.e.,
    \begin{equation}\label{equ:proof_cont_final}
        \begin{aligned}
            \big\vert g_r(X)(x) - g_r(Y)(\pi(x)) \big\vert
            &\leq  \frac{8 \delta \vert X\vert \mathcal{H}^{n-1}(\partial B_{r}(o)) }{ V^n_r} 
            \stackrel{(a)}{=}  \frac{8\delta k n}{ r} ,
        \end{aligned}
    \end{equation}
    where equality $(a)$ uses the relation between the surface and the volume of an $n$-dimensional sphere of radius $r$, i.e.,
    $\mathcal{H}^{n-1} \big(\partial B_{r}(o) \big)= S^{n-1}_r = \frac{n}{r} V^n_r.$

    \smallskip
    For a fixed radius $r > 0$ and an arbitrary $\varepsilon > 0$, we can finally choose $\delta \leq \min \{r, \varepsilon r / \big( 8kn \big)\}$ small enough but independent of $X$, and Axiom~\ref{axi:indiv_cont} holds.

    \medskip
    \textbf{Axiom~\ref{axi:alpha_clone_locality}.} 
    Let $\alpha \geq 2r$, let $\delta$ satisfy $r> \delta >0$ and let $S \subseteq \mathbb{R}^n$ be a finite subset.
    Consider $x\in S$ and $x' \in \mathbb{R}^n \setminus S$ such that $d(x,x') \leq \delta.$
    We moreover denote by $S'$ the union $S \cup \{x'\}.$

    \smallskip
    The proof follows the same structure as that of Axiom~\ref{axi:indiv_cont}: for $y$ an element of $S$ such that $d(x,y) \geq  \alpha \geq 2r$, we first show that $ g_{r,S,y}$ and $g_{r,S',y}$ are equal on a carefully chosen set.
    Indeed for $z$ an element of $B_{r - \delta}(y)$, the triangle inequality gives $d(z,x') \geq d(x,y) - d(x,x') - d(x',y) > 2r - \delta - (r- \delta) = r $, and we obtain that $ g_{r,S,y}(z) = g_{r,S',y}(z).$ 

    \smallskip
    We then bound the difference of weighting between $S$ and $S'$ in a similar fashion as in Equation~\eqref{equ:bound_rep_diff_X&Y}.
    \begin{equation}\label{eq:bound_rep_diff_local_simple}
        \begin{aligned}
             &\big\vert g_r(S)(y) - g_r(S')(y) \big\vert \\
             &\stackrel{(a)}{\leq}  \bigg(\frac{\mu\big(B_r(S')\big)}{\mu\big(B_r(S)\big)} -  1 \bigg) \cdot \bigg\vert  \int_{B_{r - \delta}(y)} \frac{g_{r,S',y}} {\mu\big(B_r(S')\big)} d\mu \bigg\vert ,\\
             &+  \int_{B_r(y)\setminus B_{r - \delta}(y)} \frac{1} {\mu\big(B_r(S)\big)}  d\mu ,\\
             &\stackrel{(b)}{\leq}   \frac{\mu\big(B_{r}( x') \big) - \mu\big( B_r(x) \big)}{\mu\big(B_r(S)\big)}  + \frac{ \mu\big(B_r(y)\big) - \mu\big(B_{r - \delta}(y)\big)}{\mu\big(B_r(S)\big)} ,\\
             &\stackrel{(c)}{\leq}   \frac{ \mu\big(B_{r+\delta}(o)\big) - \mu\big(B_{r-\delta}(o)\big)} {V^n_r} 
             \stackrel{(d)}{\leq} \frac{4 \delta n}{  r} .
        \end{aligned}
    \end{equation}
    Inequalities $(a)$ follows from similar arguments as Equation~\eqref{equ:bound_rep_diff_X&Y}$(a)$, and inequality $b$ uses the countable additivity and the non-negativity of $\mu$ with $B_r(S') = B_r(S) \cup B_r(x') \setminus B_r(S)$ and $B_r(S) \supset B_r(x).$
    Inequality $(c)$ holds since the Lebesgue measure is translation invariant and because $\mu\big(B_r(S)\big) \geq \mu\big(B_r(o)\big) = V^n_r.$
    Finally, inequality $(d)$ holds for $\delta$ small enough, but independent of $S$, by combining arguments from Equations~\ref{equ:cont_bound_diff_with_Hausdorff} and~\ref{equ:proof_cont_final}.

    For a fixed radius $r>0$ and an arbitrary $\varepsilon>0$, we can finally choose $\delta \leq \min\{r, \varepsilon r / (4n)\}$ small enough but independent of $S$, and Axiom~\ref{axi:alpha_clone_locality} holds with $\alpha = 2r.$

    \medskip
    \textbf{Axiom~\ref{axi:clone_fair_uni}.} 
    Let $\delta$ be a positive number satisfying $r> \delta >0$, and $S$ be a finite subset of $\mathbb{R}^n.$ We moreover let $x,y$ be two elements of $S$ such that $d(x,y) \leq \delta.$

    \smallskip
    Similarly as for the proof of Axiom~\ref{axi:indiv_cont}, we first show that $B_{r-\delta}(x) \subset B_{r}(x) \cap B_r(y).$ Indeed, for $z$ an element of $B_{r-\delta}(x)$, the triangle inequality gives $d(z,y) \leq d(z,x) + d(x,y) < r -\delta + \delta = r.$ Equation~\eqref{equ:grade_simple_f} then implies the functional equality $g_{r,S,x} = g_{r,S,y}$  on $B_{r-\delta}(x).$

    \smallskip
    We then bound the difference of weighting between $x$ and $y$ in a similar fashion as in Equation~\eqref{equ:bound_rep_diff_X&Y}.
    \begin{equation}\label{equ:bound_clone_fair}
        \begin{aligned}
             \big\vert g_r(S)(x) - g_r(S)(y) \big\vert
             &\stackrel{(a)}{\leq}   \int_{B_r(x)\setminus B_{r - \delta}(x)} \frac{1} {\mu\big(B_r(S)\big)}  d\mu ,\\
             &\stackrel{(b)}{\leq}  \frac{ \mu\big(B_r(x)\big) - \mu\big(B_{r - \delta}(x)\big)}{\mu\big(B_r(x)\big)}  \\
             &\stackrel{(c)}{\leq} \frac{4 \delta S_r^{n-1}}{  V^n_r} = \frac{4 \delta n}{r},
        \end{aligned}
    \end{equation}
    where inequality $(a)$ and $(b)$ use similar arguments as for Equation~\eqref{equ:bound_rep_diff_X&Y}$(a)$ and $(b)$ respectively; inequality $(c)$ follows from combining the arguments of Equations~\eqref{equ:cont_bound_diff_with_Hausdorff} and~\eqref{equ:proof_cont_final}. 
    In particular, note that 
    $ \mu\big(B_{r}(x)\big) - \mu\big(B_{r -\delta}(x)\big) \leq \mu\big(B_{r+\delta}(x)\big) - \mu\big(B_{r -\delta}(x)\big) \leq 4 \delta S_r^{n-1} $
    holds for a value of $\delta$ small enough but independent of $S$ since the Lebesgue measure is invariant by translation.
    
    \smallskip
    Then for an arbitrary $\varepsilon>0$, a choice of $\delta \leq  \min\{r, \varepsilon  r / (4n) \}$ small enough ensures that Axiom~\ref{axi:clone_fair_uni} holds.

   \medskip \textbf{Conclusion.} 
   Since $g_r$ is a weighting function on $(\mathbb{R}^n, d_2)$ satisfying Axioms~\ref{axi:pos}, \ref{axi:sym}, \ref{axi:clone_fair_uni}, \ref{axi:indiv_cont} and \ref{axi:alpha_clone_locality} with $\alpha =2r$, we conclude that $g_r$ belongs in $\mathcal{R}_{2r}(\mathbb{R}^n,d_2).$

\end{proof}

We now recall the formulation of Theorem~\ref{thm:cont_convex_combi_gr} before attacking its proof.

\begin{reptheorem}{thm:cont_convex_combi_gr}
    Let $\nu$ be a probability density function over $[0,\alpha].$
    Then the weighting function $f_\nu : S \in \mathcal{P}\big(\mathbb{R}^n \big) \mapsto \int_0^{\alpha} \nu(r) g_r(S) ~dr$ belongs in $\mathcal{R}_{2\alpha}(\mathbb{R}^n,d_2).$
\end{reptheorem}

\begin{proof}
Let $\nu$ be a probability density function over $[0,\alpha]$, that is a non-negative Lebesgue-integrable function satisfying $\int_0^{\alpha} \nu(r) dr =1$. Let $S$ be a finite subset of $\mathbb{R}^n$ and $x$ be an element of $S.$

\smallskip
First, note that $r \mapsto g_r(S)(x)$ is  non-negative and bounded by one over $\mathbb{R}_{>0}$, hence  $r \mapsto \nu(r) \cdot g_r(S)(x)$ is Lebesgue-integrable and $f_\nu(S)(x)$ is non-negative. Moreover, we have 
\begin{equation*}
\begin{aligned}
    \sum_{x \in X} f_\nu(S)(x) 
    &=  \int_{(0,\alpha]} \nu(r) \sum_{x \in X} g_r(S)(x) ~dr 
    \stackrel{(a)}{=}  \int_{(0,\alpha]} \nu(r) ~dr =1 ,
\end{aligned}
\end{equation*}
where equality $(a)$ uses that $g_r$ is a weighting function.
This ensures that $f_\nu$ is  also a weighting function of $(\mathbb{R}^n, d_2).$
 
\smallskip

We now verify that Axioms~\ref{axi:pos}, \ref{axi:sym}, \ref{axi:clone_fair_uni}, \ref{axi:indiv_cont} and \ref{axi:alpha_clone_locality} are implied for $f_\nu$ from the fact that $g_r$ belongs in $\mathcal{R}_{2\alpha}(\mathbb{R}^n,d_2)$ for all $\alpha\geq r>0.$
On the one hand, Axiom~\ref{axi:pos} is directly implied from Equation~\eqref{equ:proof_gr_pos}, i.e,
\begin{equation*}
\begin{aligned}
    f_\nu(S)(x) 
    &=  \int_{(0,\alpha]} \nu(r)  g_r(S)(x) ~dr 
    \geq \int_{(0,\alpha]} \frac{\nu(r)}{|S|^2}  dr 
    = \frac{1}{|S|^2} >0.
\end{aligned}
\end{equation*}
Furthermore, for an isometry $\sigma_S : S\mapsto S$, we have
\begin{equation*}
\begin{aligned}
    f_\nu(S)(\sigma_S(x)) 
    &=  \int_{(0,\alpha]} \nu(r)  g_r(S)(\sigma_S(x)) ~dr ,\\
    &\stackrel{(a)}{=} \int_{(0,\alpha]} \nu(r)  g_r(S)(x) ~dr 
    = f_\nu(S)(x),
\end{aligned}
\end{equation*}
where equality $(a)$ uses that Axiom~\ref{axi:sym} holds for $g_r$, hence it also holds for $f_\nu.$
\smallbreak
On the other hand, Axioms~\ref{axi:clone_fair_uni}, \ref{axi:indiv_cont} and~\ref{axi:alpha_clone_locality} require a little more work; we hereafter focus on Axiom~\ref{axi:indiv_cont}.
Let $k\in \mathbb{N}$ be a positive integer and let $\varepsilon>0$ be an arbitrary constant.
Let  $\mathcal{V}(\cdot)$ denote the cumulative distribution associated with the density $\nu(\cdot)$; by continuity of $\mathcal{V}$, there exists a positive $C \in (0,\alpha)$ such that $\mathcal{V}(C) = \varepsilon /2.$ 
Let $\delta$ verify $C>\delta>0$, and let $X,Y$ be two subsets of $\mathbb{R}^n$ of cardinality $\vert Y \vert = \vert X \vert = k$ such that $d_\Pi(X,Y) \leq \delta.$ Take a bijection $\pi \in  \Pi(X,Y)$, note that $d(y, \pi(y)) \leq \delta$  holds for all $y$ in $Y.$ 

For $x\in X$, the following then holds, i.e.,
\begin{equation*}
    \begin{aligned}
        &\big\vert f_\nu(X)(x) - f_\nu(Y)(\pi^{-1}(x)) \big\vert \\
        &\stackrel{(a)}{\leq} \int_{(0,\alpha]} \nu(r) \big\vert g_r(X)(x) - g_r(Y)(\pi^{-1}(x)) \big\vert dr,\\
        &\stackrel{(b)}{\leq} \int_0^C \nu(r)  dr + \int_{C}^{\alpha} 2  \vert X \vert  \nu(r) \frac{  \mu\big( \partial B_{r}(o) + \overline{B}_\delta(o)\big) }{V^n_r} dr.\\
        &\stackrel{(c)}{\leq} \mathcal{V}(C) + 4  k \delta \sup_{r\in [C,\alpha]} h(r,\delta) .
    \end{aligned}
\end{equation*}
Inequality $(a)$ uses the non-negativity of $\nu$ and the triangle inequality;
inequality $(b)$ holds by bounding the difference $\big\vert g_r(X)(x) - g_r(Y)(\sigma(x)) \big\vert $ by one in the first term, and using the arguments of Equation~\eqref{equ:bound_rep_diff_X&Y} for the second term. 
We finally define the function $h(r,\delta) := \mu\big( \partial B_{r}(o) + \overline{B}_\delta(o)\big) / (2 \delta V^n_r) = (V^n_{r+\delta} - V^n_{r-\delta} ) / (2 \delta V^n_r)$, and we establish inequality $(c)$ by bounding $h(r, \delta)$ by its supremum on $[C,\alpha]$, which is finite since the function $h(\cdot, \delta): r\in [C,\alpha] \mapsto h(r,\delta)$ is continuous on the closed interval $[C,\alpha].$

For a fixed $r\in [C,\alpha]$, we perform a Taylor-Expansion of $h(r,\delta)$ in $\delta = 0$ as follows, i.e.,
\begin{equation*}
    \begin{aligned}
        h(r,\delta) 
        &=  \frac{1}{2 \delta} \bigg( 2 n\frac{\delta}{r} + o \Big(\frac{\delta^2}{r^2}\Big)\bigg) 
        = \frac{n}{r} + o \Big(\frac{\delta}{r^2}\Big)
        = \frac{n}{r} + o \Big(\frac{\delta}{C^2}\Big),
    \end{aligned}
\end{equation*}
where the last equality uses that $r$ belongs to $[C,\alpha].$ This shows that $\lim_{\delta \to 0} h(r,\delta) = n/r$ uniformly in $r$.
Moreover,  we have the point-wise convergence $\lim_{r \to C} h(r,\delta)  = (V^n_{C+\delta} - V^n_{C-\delta} ) / (2 \delta V^n_C)$ for each $C/2>\delta >0.$
Moore-Osgood's Theorem then implies that the double limit exists and we have, i.e.,
\begin{equation*}
    \begin{aligned}
        \limsup_{\delta \to 0} \sup_{r\in [C,\alpha]} h(r,\delta) 
        =\sup_{r\in [C,\alpha]} \limsup_{\delta \to 0} h(r,\delta)
        = \sup_{r\in [C,\alpha]} \frac{n}{r} = \frac{n}{C},
    \end{aligned}
\end{equation*}

This implies that there exists $\delta_0(C) >0$ small enough and independent of $X$ such that $\sup_{r\in [C,\alpha]} h(r,\delta) \leq \frac{2n}{C}$ holds for all $C>\delta_0(C)\geq\delta > 0.$ For such a $\delta$, we finally get, i.e.,
\begin{equation*}
    \begin{aligned}
        \big\vert f_\nu(X)(x) - f_\nu(Y)(\pi^{-1}(x)) \big\vert 
        \leq \frac{\varepsilon}{2} + \frac{8  k \delta}{C} \leq \varepsilon,
    \end{aligned}
\end{equation*}
where the last inequality holds for $\delta \leq \min \{ \delta_0(C), \varepsilon C / (16 k n) \}$ small enough but independent of $X$, hence $f_\nu$ verifies Axiom~\ref{axi:indiv_cont}.

\medskip
Similar arguments combined with Equations~\eqref{eq:bound_rep_diff_local_simple} and~\eqref{equ:bound_clone_fair} respectively show that  $f_\nu$ verifies Axioms~\ref{axi:alpha_clone_locality} and~\ref{axi:clone_fair_uni}.
We then conclude that $f_\nu$ belongs in $\mathcal{R}_{2\alpha}(\mathbb{R}^n,d_2).$

\end{proof}

\section{Computational Proofs and Supplements}\label{sec:add_comput_res}

This section presents the proofs of Theorems~\ref{thm:gr_estimate} and~\ref{thm:fnu_estimate}, the high-probability accuracy guarantee for Algorithm~\ref{alg:gr_approxunion} (c.f. Lemma~\ref{lem:gr_approxunion}), a refined sample-complexity analysis of Algorithm~\ref{alg:fnu_estimate} (c.f. Lemma~\ref{lem:fnu_bias_bound}), and introduces Algorithm~\ref{alg:fnu_reuse}, discussed in Section~\ref{sec:computational_cons}, which estimates $f_\nu$ by reusing samples across radii.

We start with recalling  Theorem~\ref{thm:gr_estimate} before attacking its proof.

\begin{reptheorem}{thm:gr_estimate}
    Algorithm~\ref{alg:gr_estimate} yields a consistent and asymptotically unbiased estimate of $g_r(S).$ 
    \noindent Moreover, for any  target accuracy $\varepsilon >0$ and confidence  level $\delta \in (0,1)$, setting the number of samples to satisfy $$k\geq \frac{2(\vert S\vert-1)^2}{\varepsilon^2} \log\frac{2 \vert S \vert}{\delta}$$
    guarantees  with probability at least $1 -\delta$ that
    $$ \big\vert \;\widehat{g}_r(S)(x)  - g_r(S)(x) \;\big\vert \leq \varepsilon.$$
\end{reptheorem}

\begin{proof}
First, the precomputation step is valid. If $z_i\in B_r(x)$ and some $y\in S$ contributes to the local grading $g_{r,S,x}(z_i)$, then
$d_2(y,x)\leq d_2(y,z_i)+d_2(z_i,x)\leq 2r$, so $y\in B_{2r}(x)\cap S$ and
indeed every contributing $y$ is examined in count $c_i$; thus $g_{r,S,x}(z_i) = \frac{1}{c_i}.$

\smallbreak
\noindent\textbf{Consistency and Asymptotic Unbiasedness of the  estimator.}
For a fixed $x \in S$, let $Z_1,\dots Z_k$ be $k$ i.i.d. random variables sampled from $B_r(x)$ , $C_i = \big\vert \{ y \in B_{2r}(x)\cap S \mid d_2(Z_i,y) \leq r \}\big\vert$ be the count associated with $Z_i$ for $i\in[k]$, and let $\widehat N_x^k = \frac{1}{k}\sum_{i=1}^k \frac{1}{C_i}$ be the  empirical average of the reciprocals $1/ C_i.$
Note that $\widehat N_x$ is an unbiased estimator of the unnormalized $g_r(S)(x)$, i.e.,
\begin{equation*}
    \begin{aligned}
    \mathbb{E}[\widehat N_x^k] &= \mathbb{E}[\tfrac{1}{ C_i}] 
    = \frac{1}{\mu(B_r(x))}\int_{B_r(x)} g_{r,S,x}(z) \, d\mu(z),\\
    &= \frac{\mu(B_r(S))}{\mu(B_r(x))} g_r(S)(x).
    \end{aligned}
\end{equation*}
Now consider the estimator $\widehat D^k = \sum_{x\in S} \widehat N_x^k$. By linearity of expectation, we have, i.e.,
\begin{equation*}
    \begin{aligned}
    \mathbb{E}[\widehat D^k] &= \sum_{x\in S} \tfrac{\mu(B_r(S))}{\mu(B_r(x))} g_r(S)(x)= \tfrac{\mu(B_r(S))}{\mu(B_r(x))}.
    \end{aligned}
\end{equation*}

Thus both $\widehat N_x^k$ and $\widehat D^k$ are unbiased estimators of their respective unnormalized quantities. Moreover, since $\widehat N_x$ and $\widehat D$ are averages of bounded i.i.d. random variables, the strong law of large numbers implies that
\begin{equation*}
    \widehat N_x^k \xrightarrow[k\to \infty]{\;\; a.s.\;\;} \tfrac{\mu(B_r(S))}{\mu(B_r(x))}\,g_r(S)(x),
\qquad 
\widehat D^k \xrightarrow[k\to \infty]{\;\;a.s.\;\;} \tfrac{\mu(B_r(S))}{\mu(B_r(x))} >0.
\end{equation*}
Slutsky’s theorem then ensures that $\widehat N_x^k/\widehat D^k$ is a \emph{consistent} estimator of $g_r(S)(x)$, i.e.,
\begin{equation*}
    \frac{\widehat N_x^k}{\widehat D^k} \xrightarrow[k\to \infty]{\;\;p\;\;}
\frac{ \tfrac{\mu(B_r(S))}{\mu(B_r(x))}\, g_r(S)(x) }
     { \tfrac{\mu(B_r(S))}{\mu(B_r(x))} }
= g_r(S)(x).
\end{equation*}
Moreover, the estimator $\widehat N_x^k / \widehat D^k$ is uniformly bounded, i.e., it holds for all $k \in \mathbb{N}$ that $0 \leq \widehat N_x^k / \widehat D^k \leq 1.$
Hence, the bounded convergence theorem ensures that the expectation commutes with the limit, and we have
\begin{equation*}
    \lim_{k\to\infty}\, \mathbb{E}\!\left[\tfrac{\widehat N_x^k}{\widehat D^k}\right]
= g_r(S)(x),
\end{equation*}
i.e., the estimator is also \emph{asymptotically unbiased}.

\smallbreak
\noindent\textbf{Concentration of $\widehat N_x^k$ and $\widehat D^k$.}
Note that $C_1, \dots, C_k$ are independent random variables which verify $\frac{1}{ \vert S\vert} \leq \frac{1}{ C_i} \leq 1$ almost surely. Hoeffding's inequality then gives the following concentration bound for $\widehat N_x^k$ and any $t>0$, i.e.,
\begin{equation}\label{equ:N_x_hoeffding_bound}
    \Pr\Big[ \big\vert \widehat N_x^k - \mathbb{E}[\widehat N_x^k] \big\vert \geq t\Big] 
    \leq 2 \exp\!\bigg(-\frac{2k t^2 |S|^2}{(|S|-1)^2}\bigg)
\end{equation}

Moreover, note that the event $\big\vert \widehat D^k - \mathbb{E}[\widehat D^k] \big\vert \geq t \vert S \vert$ requires that at least one of the terms $\big\vert \widehat N_y^k - \mathbb{E}[\widehat N_y^k] \big\vert $ for $y \in S$ exceeds $t$, hence Boole's inequality yields, i.e.,
\begin{equation}
\label{equ:hoeff_D^k}
\begin{aligned}
     \Pr\Big[ \big\vert \widehat D^k - \mathbb{E}[\widehat D^k] \big\vert \geq t \cdot\vert S \vert \Big] 
     &\leq \Pr\bigg[ \bigcup_{x\in S} \big\{\big\vert \widehat N_x^k - \mathbb{E}[\widehat N_x^k] \big\vert \geq t \big\}\bigg] ,\\
    &\leq 2 \vert S \vert \exp\!\bigg(-\frac{2k t^2 \vert S\vert^2}{(\vert S\vert-1)^2}\bigg).
\end{aligned}
\end{equation}

\smallbreak
\noindent\textbf{Propagating errors to the quotient $\widehat N_x^k / \widehat D^k$.}
Assuming that the deviations 
$\big \vert\widehat N_y^k - \mathbb{E}[\widehat N_y^k]\big \vert \leq t$ hold simultaneously for all $y\in S$, we relate the ratio $\widehat N_x^k / \widehat D^k$ with the corresponding ratio of expectations and bound their difference.

\begin{equation}
\label{equ:final_gr_estimation}
    \begin{aligned}
        \bigg \vert \frac{\widehat N^k_x}{\widehat D^k} - g_r(S)(x) \bigg\vert
        &= \bigg \vert \frac{\widehat N^k_x}{\widehat D^k} - \frac{ \mathbb{E} [\widehat N^k_x] }{ \mathbb{E}[\widehat D^k]}  \bigg\vert, \\
        &\stackrel{(a)}{\leq}  \frac{ \big \vert \widehat N^k_x - \mathbb{E} [\widehat N^k_x] \big \vert }{\widehat D^k }  +  \frac{ \big \vert
         \mathbb{E}[\widehat D^k] - \widehat D^k  \big\vert \cdot \mathbb{E} [\widehat N^k_x]}
        {\widehat D^k \cdot \mathbb{E}[\widehat D^k]} ,\\
        &\stackrel{(b)}{\leq} t (\vert S \vert +1) := \varepsilon.
    \end{aligned}
\end{equation}
Inequality $(a)$ follows from  adding and substracting the terms $ \mathbb{E} [\widehat N^k_x] / \mathbb{E} [\widehat D^k]$, and using the triangle inequality with the non-negativity of $\widehat N^k_x$, $ \mathbb{E} [\widehat N^k_x]$, $ \widehat D^k$ and $ \mathbb{E} [\widehat D^k]$.
Inequality $(b)$ holds by noting that the estimators $\widehat N_x^k $ (and their expectation) are confined to the interval $[1/ \vert S \vert, 1].$ Consequently, $\widehat D^k$ and its expectation lie in $[1, \vert S \vert]$, and the denominator in the previous expression is greater than one. Finally, it uses the assumption that the event $\big\vert \widehat N_y^k - \mathbb{E}[\widehat N_y^k] \big\vert \leq t $ holds simultaneously for all $y \in S$, which guarantees that $\big\vert \mathbb{E}[\widehat D^k] - \widehat D^k  \big\vert \leq t \vert S \vert.$

Substituting $t = \frac{\varepsilon}{  \vert S \vert +1}$ in Equation~\ref{equ:hoeff_D^k} and identifying the right term as $\delta >0$, we finally conclude that Equation~\ref{equ:final_gr_estimation} holds if 
$$ k\geq \frac{( \vert S \vert ^2 - 1)^2}{2\varepsilon^2 \vert S \vert ^2} \ln \!\frac{2 \vert S \vert}{\delta}.$$

\end{proof}

We now prove the high-probability guarantees for Algorithm~\ref{alg:gr_approxunion}.

\begin{lemma}\label{lem:gr_approxunion}
Let $S \subseteq \mathbb{R}^n$, $x\in S$, $r>0$, and $\varepsilon, \delta \in (0,1)$. 
Then, Algorithm~\ref{alg:gr_approxunion} outputs an estimate $\widehat g_r(S)(x)$ that satisfies with probability at least  $1-\delta$
$$\big\vert \widehat g_r(S)(x) - g_r(S)(x) \big\vert \leq \varepsilon ,$$
provided that $t$, the number of independent runs of $\textsc{ApproxUnion}$ used for median amplification, and $k$ the number of samples used to estimate $N_x^k$, satisfy
$$ t \ge 24 \ln\frac{2}{\delta} 
\quad \text{and} \quad 
k \ge \frac{8 (\vert S\vert -1)^2}{\varepsilon^2 \vert S\vert^2} \ln\!\frac{4}{\delta}.$$
 
\end{lemma}
\begin{proof}
By Theorem 2 in~\cite{bringmann_approximating_2010}, the algorithm
$\textsc{ApproxUnion}(\{B_r(x) : x \in S\}, \varepsilon/4)$
returns, with probability at least $3/4$, an estimate of $\mu(B_r(S))$ with multiplicative error at most $\varepsilon/4$.

\smallskip
\noindent \textbf{Confidence Amplification.} 
To boost the success probability, we run $\textsc{ApproxUnion}$ independently $t$ times. Let
$$ X_j = 
\begin{cases} 
1 & \text{if } \vert \widehat U_j - \mu(B_r(S)) \vert \le (\varepsilon/4) \mu(B_r(S)),\\
0 & \text{otherwise,}
\end{cases} \quad j = 1,\dots,t,
$$
and define $X = \sum_{j=1}^t X_j$. Let $\widehat U = \operatorname{median}(\widehat U_1,\dots,\widehat U_t)$.

Since the median fails only if $X \leq t/2$, we have
\begin{equation*}
\begin{aligned}
    &\Pr \big[ \vert \widehat U - \mu(B_r(S)) \vert > \varepsilon \mu(B_r(S)) \big]
    \leq \Pr \Big [ X \leq t/2 \Big],\\
    &= \Pr \Big [ X \leq \Big(1 - \frac{1}{3}\Big) \frac{3}{4}t \leq \Big(1 - \frac{1}{3}\Big) \mathbb{E}[X] \Big],\\
    &\stackrel{(a)}{\leq} \exp\Big(- \frac{(1/3)^2 (3t/4)}{2} \Big) = \exp(-t/24),
\end{aligned}
\end{equation*}
 where inequality $(a)$ follows from Chernoff's lower tail bound  with parameter $0< 1/3 <1$, as well as the inequality $\mathbb{E}[X] \geq 3 t /4.$
 
Thus, to ensure failure probability at most $\delta/2$, it suffices to choose
\begin{equation*}
    t \geq 24 \ln\frac{2}{\delta}.
\end{equation*}

\smallskip
\noindent \textbf{Propagation to the Quotient.} 
Equation~\ref{equ:N_x_hoeffding_bound} implies that, for
$$
k \geq \frac{8 ( \vert S \vert-1)^2}{\varepsilon^2 \vert S\vert ^2} \ln\frac{4}{\delta},
$$
the estimator $N^k_x$ lies within additive error $\varepsilon/4$ of its expectation with probability at least $1 - \delta/2$.

By a union bound, both the amplification step for $\widehat U$ and the approximation of $N^k_x$ succeed simultaneously with probability at least $1 - \delta$. Then, similarly to Equation~\ref{equ:final_gr_estimation}, we obtain
\begin{equation*}
\begin{aligned}
    &\bigg\vert \frac{\widehat N_x \, \mu(B_r(x))}{\widehat U} - g_r(S)(x) \bigg\vert 
    = \bigg\vert \frac{\widehat N_x \, \mu(B_r(x))}{\widehat U} - \frac{\mathbb{E}[\widehat N_x] \, \mu(B_r(x))}{\mu(B_r(S))} \bigg\vert ,\\
    &\leq \frac{\mu(B_r(x))}{\widehat U} \left( \big\vert \widehat N_x - \mathbb{E}[\widehat N_x] \big\vert + \frac{\big\vert \mu(B_r(S)) - \widehat U \big\vert}{\mu(B_r(S))} \mathbb{E}[\widehat N_x] \right), \\
    &\leq \frac{\mu(B_r(x))}{\mu(B_r(S))(1 - \varepsilon/4)} \left( \frac{\varepsilon}{4} + \frac{\varepsilon}{4} \mathbb{E}[\widehat N_x] \right) ,\\
    &\leq 2 \left( \frac{\varepsilon}{4} + \frac{\varepsilon}{4} \right) = \varepsilon.
\end{aligned}
\end{equation*}

Here, the first inequality uses the triangle inequality, the second uses that, with probability at least $1 - \delta$, both
\begin{equation}
    \big\vert \widehat N_x - \mathbb{E}[\widehat N_x] \big\vert \le \varepsilon/4 
    \quad \text{and} 
    \quad \frac{\big\vert \mu(B_r(S)) - \widehat U \big\vert}{\mu(B_r(S))} \le \varepsilon/4
\end{equation}
hold simultaneously. 
The third inequality finally holds for $\varepsilon < 2$ since then $1/ (1 - \varepsilon/4) \leq 2$, and it uses  $\mathbb{E} [\widehat N_x] \leq 1$, as well as $\mu(B_r(x)) \leq \mu(B_r(S)).$

\end{proof}

We next turn to the proof of Theorem~\ref{thm:fnu_estimate}.

\begin{reptheorem}{thm:fnu_estimate}
    For any target accuracy $\varepsilon >0$ and confidence  level $\delta \in(0,1)$, setting the  outer and inner sample sizes in Algorithm~\ref{alg:fnu_estimate} such that 
    $$ M \;\geq\; \frac{8}{\varepsilon^2}\ln\!\frac{4|S|}{\delta},
    \qquad
    k \;\geq\; \frac{8(|S|-1)^2}{\varepsilon^2}\ln\!\frac{4|S|M}{\delta}, $$
    ensures with probability at least $1-\delta$ the following bound for every $x\in S$, i.e.,
    $$\big \vert \widehat f_\nu(S)(x) - f_\nu(S)(x)\big\vert \le \varepsilon,$$ and the estimator is consistent.
    Moreover, if $k=k(M)=\Omega(\log M)$ (e.g.\ $k(M)$ chosen as above), then Algorithm~\ref{alg:fnu_estimate} also yields an asymptotically unbiased estimator, i.e.,
    $\lim_{M\to\infty}\mathbb E[\widehat f_\nu(S)(x)] = f_\nu(S)(x).$ 
\end{reptheorem}

\begin{proof}
For any fixed $x\in S$, we decompose the estimation error as follows, i.e.,
\begin{equation*}
    \begin{aligned}
         \big\vert  f_\nu(S)(x) - \widehat f_\nu(S)(x) \big \vert
         &= \bigg\vert E_{r\sim\nu}[g_r(S)(x)] - \frac{1}{M}\sum_{j=1}^M \widehat g_{r_j}(S)(x) \bigg \vert,\\
         &\leq \underbrace{ \bigg\vert \mathbb E_{r\sim\nu}[g_r(S)(x)] - \frac{1}{M}\sum_{j=1}^M g_{r_j}(S)(x) \bigg\vert}_{(I)},\\
         &\;+\; \underbrace{ \bigg\vert \frac{1}{M}\sum_{j=1}^M \big(\widehat g_{r_j}(S)(x) - g_{r_j}(S)(x)\big) \bigg\vert}_{(II)}.
    \end{aligned}
\end{equation*}

We first bound $(I)$, the outer Monte-Carlo term. Since $r_1,\dots, r_M$ are independent random variables, so are the $g_{r_1}(S)(x), \dots, g_{r_M}(S)(x).$ Since each $g_{r_j}(S)(x)$ belongs in $[0,1]$, Hoeffding's inequality gives for any \(\varepsilon_1>0\),
\begin{equation*}
    \Pr\bigg(\bigg \vert \frac{1}{M}\sum_{j=1}^M g_{r_j}(S)(x) - \mathbb E_{r\sim\nu}[g_r(S)(x)]\bigg\vert \geq \varepsilon_1\bigg) 
    \leq 2\exp(-2 M \varepsilon_1^2).
\end{equation*}
Applying a union bound over the $\vert S \vert$ choices of $x$ shows that, with $M\geq \frac{1}{2\varepsilon_1^2}\ln\!\frac{2 \vert S\vert}{\delta_1}$, 
the following inequality holds for all $x\in S$ simultaneously with probability at least $1-\delta_1$, i.e.,
\begin{equation*}
    \bigg\vert \mathbb E_{r\sim\nu}[g_r(S)(x)] - \frac{1}{M}\sum_{j=1}^M g_{r_j}(S)(x) \bigg\vert \leq \varepsilon_1.
\end{equation*}

We next turn to bounding $(II)$, the inner per-radius errors. We apply Theorem~\ref{thm:gr_estimate} with accuracy $\varepsilon_2>0$ and confidence $\delta_2/M \in (0,1)$ to each of the terms in the average  $(II).$
Then, with probability greater than $1 - M \delta_2 / M = 1 - \delta_2$, the average is itself smaller than $\varepsilon_2$, and it holds simultaneously for all $x\in S$ that
\begin{equation*}
    \bigg\vert \frac{1}{M}\sum_{j=1}^M \big(\widehat g_{r_j}(S)(x) - g_{r_j}(S)(x)\big) \bigg\vert \leq \varepsilon_2,
\end{equation*}
provided that $k\geq \frac{( \vert S \vert ^2 - 1)^2}{2\varepsilon^2 \vert S \vert ^2} \ln\!\frac{2 \vert S \vert M}{\delta_2}.$

We now combine the two failure probabilities $\delta_1,\delta_2$ by a union bound:
with probability at least $1-(\delta_1 + \delta_2)$ both bounds hold simultaneously,
so for every $x\in S$, we have
$$ \Big\vert  f_\nu(S)(x) - \widehat f_\nu(S) (x)\Big\vert \leq \varepsilon_1+\varepsilon_2.$$

Choosing the symmetric split $\varepsilon_1=\varepsilon_2=\varepsilon/2$ and
$\delta_1=\delta_2=\delta/2$ yields the displayed sufficient sample sizes:
$$ M \geq \frac{2}{\varepsilon^2}\ln\!\frac{4|S|}{\delta},
\qquad
k \geq \frac{2( \vert S \vert ^2 - 1)^2}{\varepsilon^2 \vert S \vert ^2} \ln\!\frac{4 \vert S \vert M}{\delta}, $$
and the uniform guarantee with probability at least $1-\delta.$

Finally, in the joint limit where $M \to \infty$ and $k(M)=\Omega(\log M)$, we directly get for any $\varepsilon>0$ that
\begin{equation*}
    \Pr\Big[\Big\vert  f_\nu(S)(x) - \widehat f_\nu(S) (x)\Big\vert > \varepsilon \Big] = O\Big(\frac{1}{M}\Big),
\end{equation*}
thus $\widehat f_\nu^{(M,k(M))}(S) \xrightarrow[M\to \infty]{\;p\;} f_\nu(S)$ and the estimator $f_\nu(S)$ is consistent.

Moreover, the estimators $\widehat f_\nu^{(M,k(M))}(S)$ are bounded in $[0,1]$ for all $M\in\mathbb{N}$, hence the expectation commutes with the limit by the bounded convergence theorem and we have
\begin{equation*}
    \lim_{M\to\infty}\, \mathbb{E}\!\left[\widehat f_\nu^{(M,k(M))(S)(x)}\right] = f_\nu(S)(x),
\end{equation*}
i.e., the estimator is \emph{asymptotically unbiased}.

\end{proof}

We refine the previous sample complexity bound for Algorithm~\ref{alg:fnu_estimate}, and show it only requires $k = O(\varepsilon^{-1})$ inner samples to achieve additive error $\varepsilon$ --an improvement over the $k \propto \varepsilon^{-2}$ dependence in Theorem~\ref{thm:fnu_estimate}.

\begin{lemma}\label{lem:fnu_bias_bound}
    There exists a function $C:\mathbb{N}\to\mathbb{R}_{>0}$  such that the following holds. For any accuracy $\varepsilon\in(0,1)$ and confidence $\delta\in(0,1)$, if the outer and inner sample sizes $M,k$ satisfy
    $$
    M \;\ge\; \frac{9}{2\varepsilon^2}\ln\!\frac{4|S|}{\delta},
    \qquad
    k \;\ge\; \frac{3\big(1 + \sqrt{|S|} + C(\vert S\vert)\big)}{4\varepsilon},
    $$
    then Algorithm~\ref{alg:fnu_estimate} produces estimates $\widehat f_\nu(S)(x)$ that satisfy simultaneously for every $x \in S$
    $$\big\vert \widehat f_\nu(S)(x) - f_\nu(S)(x)\big\vert \leq \varepsilon, $$ 
    with probability at least $1-\delta$,
\end{lemma}
\begin{proof}

For any fixed $x\in S$, we decompose the estimation error as follows, i.e.,
\bigskip
\begin{equation*}
    \begin{aligned}
         \big\vert  f_\nu(S)(x) - \widehat f_\nu(S)(x) \big \vert
         &= \bigg\vert E_{r\sim\nu}[g_r(S)(x)] - \frac{1}{M}\sum_{j=1}^M \widehat g_{r_j}(S)(x) \bigg \vert,\\
         &\leq \underbrace{ \bigg\vert \mathbb E_{r\sim\nu}[g_r(S)(x)] - \frac{1}{M}\sum_{j=1}^M g_{r_j}(S)(x) \bigg\vert}_{(I)},\\
         &+\; \underbrace{ \bigg\vert \frac{1}{M}\sum_{j=1}^M \big( g_{r_j}(S)(x) - \mathbb{E}[\widehat g_{r_j}(S)(x) ]\big) \bigg\vert}_{(II)},\\
         &+\; \underbrace{ \bigg\vert \frac{1}{M}\sum_{j=1}^M \big(\mathbb{E}[\widehat g_{r_j}(S)(x) ] - \widehat g_{r_j}(S)(x)\big) \bigg\vert}_{(III)}.
    \end{aligned}
\end{equation*}

\smallskip
\noindent\textbf{Bound on the outer Monte-Carlo term $(I)$.} 
Since $r_1,\dots, r_M$ are independent random variables, so are the $g_{r_1}(S)(x), \dots, g_{r_M}(S)(x).$ Since each $g_{r_j}(S)(x)$ belongs in $[0,1]$, Hoeffding's inequality gives for any $\varepsilon_1>0$,
\begin{equation*}
    \Pr\bigg(\bigg \vert \mathbb E_{r\sim\nu}[g_r(S)(x)] - \frac{1}{M}\sum_{j=1}^M g_{r_j}(S)(x) \bigg\vert \geq \varepsilon_1\bigg) 
    \leq 2\exp(-2 M \varepsilon_1^2).
\end{equation*}
Applying a union bound over the $\vert S \vert$ choices of $x$ shows that, with $M\geq \frac{1}{2\varepsilon_1^2}\ln\!\frac{2 \vert S\vert}{\delta_1}$, 
the following inequality holds for all $x\in S$ simultaneously with probability at least $1-\delta_1$, i.e.,
\begin{equation*}
    \bigg\vert \mathbb E_{r\sim\nu}[g_r(S)(x)] - \frac{1}{M}\sum_{j=1}^M g_{r_j}(S)(x) \bigg\vert \leq \varepsilon_1.
\end{equation*}

\smallskip
\noindent\textbf{Bound on the bias term $(II)$.}

We next construct a bound on the bias 
$$b_{j,x} :=   g_{r_j}(S)(x) - \mathbb{E}[\widehat g_{r_j}(S)(x)] := \mathbb{E}[\widehat N_{x,j}] / \mathbb{E}[\widehat D_j] - \mathbb{E}[\widehat N_{x,j} / \widehat D_j]$$
that is uniform across all radii $r_1,\dots, r_M$ and $x\in S.$ 

Consider the function $f(a,b)=a/b$ and perform a second-order Taylor expansion of $f$ around the point $(\mathbb{E}[\widehat N_{x,j}],\mathbb{E}[\widehat D_j])$, treating the random pair $(\widehat N_{x,j}, \widehat D_j)$ as a perturbation of this mean point.
Denoting partial derivatives by subscripts and writing $\Delta N = \widehat N_{x,j} - \mathbb{E}[\widehat N_{x,j}]$ as well as $\Delta D = \widehat D_j - \mathbb{E}[\widehat D_j]$, we get
\begin{equation*}
    \begin{aligned}
        b_{j,x} &= f_N\,\Delta N + f_D\,\Delta D,\\
                &+ \tfrac12\big( f_{NN}\,\Delta N^2 + 2 f_{ND}\,\Delta N\Delta D + f_{DD}\,\Delta D^2\big)
                + R_3,
    \end{aligned}
\end{equation*}
where the derivatives evaluated at $(\mathbb{E}[\widehat N_x], \mathbb{E}[\widehat D])$ are
$$f_N = \frac{1}{\mathbb{E}[\widehat D_j]},\;
f_D = -\frac{\mathbb{E}[\widehat N_{x,j}]}{\mathbb{E}[\widehat D_j]^2},\;
f_{NN}=0,\;
f_{ND}=-\frac{1}{\mathbb{E}[\widehat D_j]^2},$$
as well as $
f_{DD}=\frac{2\mathbb{E}[\widehat N_{x,j}]}{\mathbb{E}[\widehat D_j]^3}, $ and $R_{x,j}$ denotes the third-order remainder.

Taking expectations and using that $\mathbb{E}[\Delta N]=\mathbb{E}[\Delta D]=0,$ the first-order terms vanish and we obtain 
\begin{equation}\label{equ:bias_second_order}
    b_{j,x} = -\frac{\operatorname{Cov}(\widehat N_{x,j},\widehat D_j)}{\mathbb{E}[\widehat D_j]^2}
    + \frac{\mathbb{E}[\widehat N_{x,j}]}{ \mathbb{E}[\widehat D_j]^3}\,\operatorname{Var}(\widehat D_j) 
    + \mathbb{E}[R_{x,j}].
\end{equation}

We now bound each term on the right-hand side. Since each  $\widehat N_{x,j}$ is an average of $k$ i.i.d. entries, each bounded in $(0,1)$, Popovicu's inequality directly gives
$$ \operatorname{Var}(\widehat N_{x,j}) \leq \frac{1}{4k}.$$

Since the $\widehat N_{x,j}$ are independent across $x\in S$, we similarly get for the variance of $\widehat D_j$
$$ \operatorname{Var}(\widehat D_j) = \sum_{x\in S} \operatorname{Var}(\widehat N_{x,j}) \leq \frac{ \vert S \vert}{4k} $$
We next bound the covariance with Cauchy-Schwartz, i.e.,
$$ \vert \operatorname{Cov} ( \widehat N_{x,j}, \widehat D_j ) \vert 
\leq \sqrt{ \operatorname{Var}(\widehat N_{x,j}) \operatorname{Var}(\widehat D_j) } 
\leq \frac{\sqrt{\vert S\vert}}{4k}$$

We finally control the remainder $\mathbb{E}[R_{x,j}].$ Since all $\widehat N_{x,j}$ and $\widehat D_j$ are deterministically bounded, all partial derivatives are also uniformly bounded on the neighborhood of $(\mathbb{E}[\widehat N_{x,j}],\mathbb{E}[\widehat D_j]).$
By Taylor's theorem, the third-order remainder can then be bounded via a finite linear combination of expectations of cubic monomials in $\Delta N$ and $\Delta D$ with coefficients depending only on $\mathbb{E}[\widehat N_{x,j}]$ and $\mathbb{E}[\widehat D_j]$, hence that can be bounded in terms of $\vert S\vert$ uniformly across all $j\in[M].$
Moreover, these moments can be uniformly bounded across radii via the standard Marcinkiewicz–Zygmund moment inequality, and there exists a constant $C(\vert S\vert) \in \mathbb{R}$ such that 
$$\vert \mathbb{E}[R_{x,j}]\vert \leq  C(\vert S\vert) k^{-3/2}.$$

Combining these bounds with Equation~\eqref{equ:bias_second_order}, using that $\mathbb{E}[\widehat D_j] \geq 1 \geq \mathbb{E}[\widehat N_{x,j}]$, we finally obtain
\begin{equation*}
    \vert b_{j,x} \vert \leq \frac{1 + \sqrt{\vert S\vert}}{4k} + C(\vert S\vert) k^{-3/2}
\end{equation*}

Since this bound is uniform across all radii $r_j$, we can finally bound the term $(II)$, i.e.,
\begin{equation*}
\begin{aligned}
    \bigg\vert \frac{1}{M}\sum_{j=1}^M \big( g_{r_j}(S)(x) - \mathbb{E}_{r\sim\nu}[\widehat g_{r_j}(S)(x) ]\big) \bigg\vert
    &\leq \frac{1}{M}\sum_{j=1}^M \vert b_{j,x} \vert,\\
    &\leq \frac{1 + \sqrt{\vert S\vert}}{4k} + \frac{C(\vert S\vert)}{ k^{3/2}}.
\end{aligned}
\end{equation*}

Hence setting $k\geq \frac{1 + \vert S \vert^{1/2} + C(\vert S\vert)}{4\varepsilon_2}$
guarantees an error smaller than $\varepsilon_2 >0.$

\smallskip
\noindent\textbf{Bound on the deviation term $(III)$.} 
For all $j \in[M]$, define the random variable $Y_{j,x} = \widehat g_{r_j}(S)(x) - \mathbb E[\widehat g_{r_j}(S)(x)] .$ Note that each $Y_{j,x}$ is zero-mean and verify $\vert Y_{j,x}\vert \leq 1$ since each $\widehat g_{r_j}(S)(x)$ belongs in $(0,1).$ Moreover, the $Y_{1,x}, \dots, Y_{M,x}$ are independent conditioned on the radii $r_1,\dots,r_M$, hence Hoeffding's inequality gives
\begin{equation*}
\begin{aligned}
    \Pr\bigg[\bigg\vert \frac{1}{M}\sum_{j=1}^M Y_{j,x}  \bigg\vert  \geq \varepsilon_3 \bigg]
    \leq 2 \exp(-2 M \varepsilon_3^2)
\end{aligned}
\end{equation*}

By a union bound over $x\in S$, choosing 
$$M \geq \frac{2}{\varepsilon_3^2} \ln\frac{2 \vert S \vert}{\delta_3}$$
guarantees that $\big\vert \frac{1}{M}\sum_{j=1}^M Y_{j,x}  \big\vert  \leq \varepsilon_3$ holds simultaneously for all $x\in S$ with probability at least $1-\delta_3.$

\smallskip
\noindent\textbf{Conclusion.}
Choosing a symmetric split $\varepsilon_1=\varepsilon_2= \varepsilon_3 = \varepsilon/3$ and
$\delta_1=\delta_3=\delta/2$, we finally get sufficient sample sizes:
$$ M \geq \frac{9}{2\varepsilon^2}\ln\!\frac{4|S|}{\delta},
\qquad
k \geq \frac{3(1 + \vert S \vert^{1/2} + C(\vert S\vert))}{4\varepsilon} .$$

Then, with total failure probability smaller than $\delta_1 + \delta_2 = \delta$, the following holds simultaneously for all $x\in S$, i.e.,
$$ \Big\vert  f_\nu(S)(x) - \widehat f_\nu(S) (x)\Big\vert \leq \varepsilon.$$

\end{proof}

We finally introduce Algorithm~\ref{alg:fnu_reuse}, already discussed in Section~\ref{sec:computational_cons}.
Algorithm~\ref{alg:fnu_reuse} provides a Monte Carlo estimator of $f_\nu(S)$ that aims to reduce the total number of samples by reusing points across multiple radii.
The procedure begins by drawing $M$ i.i.d.\ radii from the distribution $\nu$ and sorting them, i.e., $r_1 \leq \dots \leq r_M$.
For each $x \in S$, the algorithm maintains a counter $\mathrm{count}_x[i]$ to ensure that every ball  $B_{r_i}(x)$ ultimately receives exactly $k$ independent samples.
Sampling starts from the largest radius $r_M$: each point drawn uniformly from $B{r_M}(x)$ is also valid for all smaller radii $r_i$ such that $z \in B_{r_i}(x)$, and is therefore reused to update the corresponding counts and local depth estimates.
This reuse proceeds iteratively from larger to smaller radii---never in the reverse direction---since points sampled from smaller balls are not uniformly distributed within larger ones.
To efficiently handle the range updates induced by such reuse, operations of the form $\mathrm{count}x[j:i] !+!= 1$ or $c[j:i] !+!= 1$ can be implemented in amortized $O(1)$ time using a difference-array representation, thereby avoiding explicit $O(i-j)$ loops.
Finally, the empirical averages $\widehat N_x[r_i]$ and their normalizations yield per-radius estimates $\widehat g{r_i}(S)(x)$, which are combined across radii to form $\widehat f_\nu(S)(x)$.
Because samples are reused across radii, the resulting estimators for different $r_i$ are correlated; nonetheless, this parallelization strategy may yield substantial gains in practice.

\begin{algorithm}[!tbh] 
\caption{Monte Carlo Estimation of $f_\nu(S)$ with Sample Reuse Across Radii}
\label{alg:fnu_reuse}
\begin{algorithmic}[1]
\Require Finite set $S \subset \mathbb{R}^n$, sampling distribution $\nu$ on $[0,\alpha]$, 
number of radii $M$, samples per radius $k$
\Ensure Estimate $\widehat f_\nu(S)(x)$ for all $x \in S$
\vspace{0.5ex}

\State Sample $M$ radii i.i.d. $r_i \sim \nu$  and sort them so that $r_1 \leq \cdots \leq r_M.$
\For{each $x \in S$}
    \State Initialize vector $\mathrm{count}_x[1:M] \gets 0$
    \State Initialize  vector $\widehat N_x[1:M] \gets 0$
\vspace{0.5ex}
    \For{$i = M, M-1, \dots, 1$}
        \For{$\ell = 1, 2, \dots, k - count_x[i]$}
            \State Sample $z$ uniformly from $B_{r_j}(x).$
            \State Find smallest $j\leq i$ such that $z\in B_{r_j}(x).$
            \State Update $\mathrm{count}_x[j:i] \gets \mathrm{count}_x[j;i] + 1$.
            \State Initialize vector $c[j:i] \gets 1$
            \For{ $y \in S\setminus\{x\}$}
                \State Find smallest $j \leq m \leq i $ such that $z\in B_{r_m}(y)$.
                \State Update $c[m:i] \gets c[m:i] + 1$.
            \EndFor
            \State Update $\widehat N_x[j:i] \gets N_x[j:i] + 1 / c[j:i].$
        \EndFor
    \EndFor
    \State Compute averages $\widehat N_x[1:M] \gets \widehat N_x[1:M] /k.$ 
\EndFor
\State Compute $\widehat G_x[1:M] \gets \widehat N_x[1:M] \;/  \sum_{x\in S} \widehat N_x[1:M]$ for all $x\in S.$
\State \textbf{Output:} 
 Final estimates $\widehat f_\nu(S)(x) := \frac{1}{M} \sum_{i=1}^M \widehat G_x[i]$ for all $x \in S.$
\end{algorithmic}
\end{algorithm}

\section{Definition of Continuity in Terms of Metric and Discussion of the Axioms}\label{sec:metric_cont&axioms_disc}

In this section, we show that Axioms~\ref{axi:indiv_cont} and~\ref{axi:class_cont} are simply an instance of the definition of continuity between two metric spaces, and we provide a more thorough discussion of the relationship between the different Axioms.

First, recall that we equipped the domain $\mathcal{P}(E)$ of a weighting function $f$ with the \emph{transport distance} $d_\Pi$, defined for two finite subsets $X, Y$ in $\mathcal{P}(E)$ with cardinality $\vert Y \vert \geq \vert X\vert$ as
$$d_\Pi(X,Y) = \min_{\pi \in \mathrm{Surj} (Y, X)} \max_{y \in Y} d(y, \pi(y)) = d_\Pi(Y,X).$$
We now turn to the codomain of $f.$
For $S$ a finite subset of $E$, recall that we denote by $\Delta(S) := \big\{p_S : S \mapsto [0,1] \mid \sum_{x\in S} p_S(x) =1 \big\}$ the set of probability distributions over the elements of $S.$ We further define $\Delta_{\mathcal{P}}(E) := \bigcup_{S\in \mathcal{P}(E)} \Delta(S)$ to be the set of probability distributions over all finite subsets of $E.$
For two finite subsets $X,Y$ of $E$ with $\vert X \vert \leq \vert Y \vert$ and respective probability distributions $p_X$ in $\Delta(X)$ and $p_Y$ in $\Delta(Y)$, we then define the map $d_\Delta : \Delta_{\mathcal{P}}(E) \times \Delta_{\mathcal{P}}(E) \mapsto \mathbb{R}_{\geq 0}$ as follows:

\begin{equation*}
\begin{aligned}
    d_\Delta(p_X, p_Y) 
    &= \min_{\pi \in \mathrm{Surj} (Y, X)} \max \bigg\{ \max_{ y \in Y} d(y, \pi(y)),\\
    & \sum_{x \in X}  \Big \vert p_X(x) - \sum_{y \in \pi^{-1}(x)} p_Y(y) \Big| \bigg\} 
    = d_\Delta(p_Y, p_X).
\end{aligned}
\end{equation*}

\begin{lemma}
     The maps $d_\Pi$ and $d_\Delta$ constitute metrics on $\mathcal{P}(E)$ and $\Delta_{\mathcal{P}}(E)$ respectively.
\end{lemma}
\begin{proof}

Note first that the set $\mathrm{Surj} (Y, X)$ is non-empty since  $\vert X \vert \leq \vert Y \vert$, and both $d_\Pi$ and $d_\Delta$ are well defined.
Moreover, \emph{symmetry} and \emph{non-negativity} clearly hold by definition. We hence focus on \emph{separability} and \emph{triangle inequality.}

\begin{itemize}
    \item \emph{Separability.} Let $X,Y$ be two finite subsets of $E$ satisfying $\vert X \vert \leq \vert Y \vert$, and let $p_X, p_Y$ be probability distribution on the respective sets.
    
    \smallskip
    On one hand, we verify that $d_\Pi(X,X) = d_\Delta(p_X, p_X)=0$: indeed, the choice $\pi =\text{Id}$ is surjective and it renders $d_\Pi(X,X)$ as well as both  terms of $d_\Delta(p_X, p_X)$ null since $d$ and $\Vert\cdot \Vert_1$ both satisfy distinguishability.

    \smallskip
    On the other hand, suppose that $d_\Delta(p_X, p_Y)=0$, and let $\pi: Y \mapsto X$ be the surjective map achieving the minimum. Since the first term is null and $d$ verifies distinguishability, we directly get that $X = Y.$ Since the second term is moreover null, we moreover obtain $p_X =p_Y$, and $d_\Delta$ also verifies distinguishability. A similar argument shows that $d_\Pi(X,Y) =0 \implies X=Y.$

    \item \emph{Triangle inequality.}
Let $X,Y,Z$ be three sets in $\mathcal{P}(E)$ satisfying $\vert X \vert \leq \vert Y \vert \leq \vert Z \vert$, and $p_X, p_Y, p_Z$ be probability distributions on the respective sets. Let $\pi_Y : Z \mapsto Y$ and $\pi_X: Y \mapsto X$ be the two surjective maps that achieve the minimum in $d_\Delta(p_Z,p_Y)$ and in $d_\Delta(p_Y,p_X)$ respectively; we then denote by $\pi$ the surjective map $\pi_X \circ \pi_Y : Z \mapsto X.$

\begin{equation*}
    \begin{aligned}
        d_\Delta(p_X, p_Z)
        &\stackrel{(a)}{\leq} \max \bigg\{ \max_{ z \in Z} d(z, \pi(z)),
        \: \sum_{x \in X}  \Big \vert p_X(x) - \sum_{z \in \pi^{-1}(x)} p_Z(z) \Big\vert \bigg\} ,\\ 
        &\stackrel{(b)}{\leq} \max \bigg\{ \max_{ z \in Z} d(z, \pi_Y(z)) + d(\pi_Y(z),  \pi(z)), \\
        &\sum_{x \in X}  \Big \vert p_X(x) - \sum_{y \in \pi_X^{-1}(x)} p_Y(y) \Big\vert 
        + \sum_{y \in \pi_X^{-1}(x)}  \Big \vert  p_Y(y) - \sum_{z \in \pi_Y^{-1}(y)} p_Z(z) \Big\vert \bigg\} ,\\ 
        &\stackrel{(c)}{\leq} \max \bigg\{ \max_{ z \in Z} d(z, \pi_Y(z)) + \max_{ y \in Y}d(y,  \pi_X(y)), \\
        &\sum_{x \in X}  \Big \vert p_X(x) - \sum_{y \in \pi_X^{-1}(x)} p_Y(y) \Big\vert 
        +  \sum_{y \in Y} \Big \vert  p_Y(y) - \sum_{z \in \pi_Y^{-1}(y)} p_Z(z) \Big\vert \bigg\} ,\\ 
        &\stackrel{(d)}{\leq} d_\Delta(p_X, p_Y) + d_\Delta(p_Y, p_Z)  .
    \end{aligned}
\end{equation*}
Inequality $(a)$ follows from the minimum in the definition of $d_\Delta(p_X, p_Z) $ being smaller than with the particular choice $\pi = \pi_X \circ \pi_Y$; 
inequality $(b)$ use the triangular inequality, as well as the partition $\pi^{-1}(x) = \bigcup_{y \in \pi_X^{-1}(x)} \pi_Y^{-1}(y).$
Inequality $(c)$ holds by taking the maximum over $Y$ in the first term and summing over the whole set $Y$ instead of simply $\pi^{-1}(x)$ in the second term;
inequality $(d)$ finally uses the inequality $\max\{a+b, u+v\} \leq \max \{a,u\} + \max \{b,v\}$, as well as the definitions of $\pi_Y$ and $\pi_X.$
A similar argument also shows that $d_\Pi$ satisfies the triangle inequality.

\end{itemize}
    
\end{proof}

Combining these results, we finally characterize the weighting functions that satisfy Axiom~\ref{axi:class_cont}.

\begin{lemma}
    A weighting function $f$ 
    satisfying Axiom~\ref{axi:class_cont} is precisely a continuous map from $(\mathcal{P}(E), d_\Pi)$ to $(\Delta_{\mathcal{P}}(E), d_\Delta).$
\end{lemma}
\begin{proof}
Note that a weighting function $f$ is indeed a map from $\mathcal{P}(E)$ to $\Delta_{\mathcal{P}}(E)$, we hereafter focus on the relationship between Axiom~\ref{axi:class_cont} and continuity between metric spaces.

\smallbreak
We first prove the direction $\implies.$ Consider an arbitrary small $\varepsilon>0$, and let $X$ be a finite subset of $E.$ By Axiom~\ref{axi:indiv_cont}, there exists $\min\{\underline{d}(X)/2, \varepsilon\} > \delta>0$ such that, for all subset $Y$ satisfying $d_\Pi(X,Y) \leq \delta$, it holds that
$\max_{x\in X} \big\vert f(X)(x) -  \sum_{y \in \pi^{-1}(x)} f(Y)(y) \big\vert \leq \varepsilon / \vert X \vert$, where $\pi:Y\mapsto X \in \Pi(X,Y)$ is a surjection (as illustrated in Figure~\ref{fig:d_Pi}, note that $\delta < \underline{d}(X)/2$ implies $\vert Y \vert \geq \vert X\vert$).
For all such subset $Y$, we then get
\begin{equation*}
\begin{aligned}
    d_\Delta(f(X), f(Y)) 
        &\leq \max \bigg\{ \max_{ y \in Y} d(y, \pi(y)),\\
        & \sum_{x \in X}  \Big \vert f(X)(x) -  f(Y)(\pi^{-1}(x)) \Big\vert \bigg\}.
\end{aligned}
\end{equation*}
Note that the first term is bounded by $\delta \leq \varepsilon$ by definition of $\pi \in \Pi(X,Y)$ and $d_\Pi.$
Moreover, as shown above, each of the $\vert X \vert$ terms of the sum is bounded by $\varepsilon / \vert X \vert.$ We hence conclude that $d_\Delta(f(X), f(Y)) \leq \varepsilon$, and $f$ is indeed a continuous map between the two metric spaces $(\mathcal{P}(E), d_\Pi)$ and $(\Delta_{\mathcal{P}}(E), d_\Delta).$

\smallbreak
We next turn to the direction $\impliedby.$ 
Consider an arbitrary small $\varepsilon>0$, and let $X$ be a finite subset of $E.$ By definition of continuity for $\underline{d}(X)/2 > \varepsilon'>0$, there exists $\varepsilon' > \delta>0$ such that, for all subsets $Y$ in $\mathcal{P}(E)$ satisfying $d_\Pi(X,Y) \leq \delta$, it holds that $d_\Delta(f(X), f(Y)) \leq \varepsilon'.$
Let $\pi$ be the minimizer in $d_\Delta(f(X), f(Y))$; and note that $\pi:Y \mapsto X$ is  a surjective map satisfying $d(y, \pi(y)) \leq \varepsilon'$ for all $y \in Y$, and it is the only one since $\varepsilon'< \underline{d}(X)/2 .$  Hence $\pi$ also belongs to $\Pi(X,Y)$, and we have 
$$\max_{x \in X} \Big\vert f(X)(x) - \sum_{y \in \pi^{-1}(x)} f(Y)(y) \Big\vert \leq d_\Delta(f(X), f(Y)) \leq \varepsilon.$$
    
\end{proof}

We next show that Axiom~\ref{axi:indiv_cont} implies continuity in terms of the \emph{Wasserstein metric}, defined for two probability measures $\lambda $ and $\xi$ on $E$ by
$$ W_1(\lambda, \xi) = \inf_{\gamma \in \Gamma(\lambda, \xi)} \mathbb{E}_{(x,y) \sim \gamma} [ d(x,y)] ,$$
where $\Gamma(\lambda, \xi)$ is the set of transport plans between $\lambda$ and $\xi$. In our setting where $X$ and $Y$ are finite sets, this means that a $\gamma \in \Gamma(f(X), f(Y))$ is a function that associates with $x\in X$ and $y\in Y$ the mass $\gamma(x,y)$ to be moved from $x$ to $y$, under the constraints that $\sum_{x\in X} \gamma(x,y) = f(Y)(y)$ holds for all $y\in Y$, and similarly that $\sum_{y\in Y} \gamma(x,y) = f(X)(x)$ holds for all $x\in X.$

\begin{lemma}
    A weigthing function $f$ satisfying Axiom~\ref{axi:indiv_cont} is a continuous function from each $(\mathcal{P}_k(E), d_\Pi)$, with $k\in \mathbb{N}$, to $(\Delta_{\mathcal{P}}(E), W_1).$
\end{lemma}
\begin{proof}
    Let $\varepsilon >0$ be an arbitrary constant, let $k\in \mathbb{N}$ and let $X$ be a subset of $E$ of cardinality $\vert X \vert = k.$ 
    By Axiom~\ref{axi:indiv_cont}, there exists $\delta >0$ such that, for each subset $Y$ of cardinality $k$ satisfying $d_\Pi (X,Y) \leq \delta$, it holds that $ \max_{x\in X} \vert f(X)(x) - f(Y)(\pi^{-1}(x)) \vert \leq \varepsilon$, where $\pi \in \Pi(X,Y).$

    We then construct a particular transport plan $\gamma \in \Gamma(f(X), f(Y))$ as follows. First, we define $\gamma\big(x,\pi^{-1}(x)\big) = \min \big\{f(X)(x), f(Y)(\pi^{-1}(x) \big\}$ for all $x \in X$, and then we complete all other entries arbitrarily so as to verify the constraints $\sum_{x\in X} \gamma(x,y) = f(Y)(y)$ for all $y\in Y$ and $\sum_{y\in Y} \gamma(x,y) = f(X)(x)$ for all $x\in X.$ 

    Clearly, we have the following, i.e.,
    \begin{equation*}
        \begin{aligned}
            &W_1(f(X), f(Y)) 
            \leq \mathbb{E}_{(x,y)\sim \gamma}[d(x,y)],\\
            &= \sum_{x \in X} d\big(x, \pi^{-1}(x)\big) \gamma\big(x, \pi^{-1}(x)\big) 
            + \sum_{x\in X, y\neq \pi^{-1}(x)} d(x, y) \gamma(x, y), \\
            &\stackrel{(a)}{\leq} \delta   
            + \Big(\max_{z,z'\in X} d(z,z') + \delta \Big) \sum_{x\in X} \Big(\sum_{y \in Y} \gamma(x,y) \Big) - \gamma\big(x, \pi^{-1}(x)\big)\\\
            &\stackrel{(b)}{\leq} \delta   + \big(\overline{d}(X) + \delta \big) k \varepsilon .
        \end{aligned}
    \end{equation*}
    Inequality $a$ uses, for the first term, that $d(x, \pi^{-1}(x))\leq \delta$ by definition of $\pi$, and in the second term, that $d(x,y) \leq \max_{z,z'\in X} d(z,z') + \delta$ for all $y\neq \pi^{-1}(x)$ by triangle inequality. It moreover uses the fact that $\gamma$ is a joint probability distribution and sums up to one.
    Inequality $(b)$ finally uses the fact that, for each $x\in X$, we have $\sum_{y\in Y} \gamma(x,y) = f(X)(x)$ and that $ f(X)(x) - \gamma\big(x, \pi^{-1}(x)\big)  $ is smaller than $\varepsilon.$ We also defined $\overline{d}(X) := \max_{z,z'\in X} d(z,z'). $
    
    For $\varepsilon'>0$, we may then choose $\varepsilon >0$ such that $ \varepsilon' / (4 k \overline{d}(X))\geq \varepsilon $, as well as a corresponding $\min\{ \varepsilon' / 2, \overline{d}(X) \} > \delta> 0$, and we finally get $W_1(f(X), f(Y))  \leq \varepsilon'.$
\end{proof}

We conclude this section with a discussion of Axioms~\ref{axi:clone_fair_uni}, \ref{axi:indiv_cont} and~\ref{axi:alpha_clone_locality}.
First, note that all the axioms are in the spirit of uniform continuity, in the sense that for a desired $\varepsilon$, the choice of $\delta$-neighborhood where the property is satisfied is independent of subset $X.$
Note that this detail is of great importance as the non-uniform equivalent of Axiom~\ref{axi:clone_fair_uni} would always be trivially satisfied for a given finite subset $X$ by choosing $\delta$ strictly smaller than $\underline{d}(X).$
Even in the case of perfect clones in a pseudo-metric space (c.f. Section~\ref{sec:discu} and Appendix~\ref{sec:metric_id}), this weaker form of Axiom~\ref{axi:clone_fair_uni} would not be interesting since perfect clones, being in the same isomorphism class, would already obtain the exact same weighting under Axiom~\ref{axi:sym}.

Note also that Axiom~\ref{axi:indiv_cont} is close to imply Axiom~\ref{axi:clone_fair_uni}: it would be enough to modify Axiom~\ref{axi:indiv_cont} and ask that  $\max_{x\in X} \vert f(X)(x) - f(Y)(\pi^{-1}(x)) \vert \leq \varepsilon$ holds for all bijections $\pi:Y \mapsto X$ such that $\max_{y \in Y} d(y,pi(y))\leq \delta.$

One could also wonder whether stronger versions of our axioms could be considered, e.g., Lipschitz-continuous variants. Axiom~\ref{axi:alpha_clone_locality} could for example be strengthened as follows.

\begin{axiom}[Lipschitz Clone Fairness]\label{axi:clone_fair_lip}
Weighting is fair among $\delta$-clones, i.e.,
there exists $L >0$ such that, for all finite subset $S\in \mathcal{P}(E)$ and $x,y$ in $S$, it holds that $\vert f(S)(x) - f(S)(y) \vert \leq L \cdot d(x,y).$
\end{axiom}

Another possible direction would be to strengthen Axiom~\ref{axi:alpha_clone_locality} by requiring the same to hold also for $\varepsilon = 0.$

\begin{axiom}[Strict $\alpha$-Locality under Addition of Clones]\label{axi:strict_alpha_clone_locality}
The addition of a clone only changes the weighting locally, i.e.,
for a finite subset $S \in \mathcal{P}(E)$, there exists $\delta>0$ such that, for each element $x \in S$ and $\delta$-clone $x'$ satisfying $d(x,x') \leq \delta$, we have for all $z\in S$ such that $d(x,z)\geq \alpha$ that $f(S)(z) = f(S\cup\{x'\})(z).$
\end{axiom}

Note that the family of weighting function $f_\nu$ we introduced in Section~\ref{sec:local_voting} verifies neither of these strengthenings.

\section{Hausdorff Distance and Extension to Perfect Clones}\label{sec:metric_id}

In this section, we expand on the discussion in Section~\ref{sec:discu} regarding the extension of our framework to perfect clones.

\paragraph{Pseudo-Metric and Perfect Clones}

Let $(E,d)$ be a pseudo-metric space, that is an ordered pair where $E$ is a set and $d: E\times E \mapsto \mathbb{R}_{\geq0}$ is a pseudo-metric on $E$ satisfying, for all $x,y,z \in E$, i.e.,
\begin{enumerate}
    \item (\emph{Non-negativity}) $d(x,y) \geq 0$ ,
    \item (\emph{Symmetry}) $d(x,y) = d(y,x)$,
    \item (\emph{Triangle inequality}) $d(x,z) \leq d(x,y) + d(y,z)$ ;
    \item (\emph{Identity}) $d(x,x) =0 .$ 
\end{enumerate}

We next show that the pseudo-metric $d$ implicitly defines an equivalence relation $\sim$ on $E$, which we refer to as the \emph{metric identification.} We denote by $[x] = \{ y \in E \mid x\sim y\}$ the equivalence class of $x$ in $E.$

\begin{lemma}[Metric Identification in $(E,d)$]\label{lem:equiv_rel_E}
The binary relation defined for all $x,y \in E$ by $x\sim y$ if and only if $d(x,y) =0$ is an equivalence relation.
Moreover, for all $x\sim y$ and $z$ in $E$, we have $d(x,z) = d(y,z).$
\end{lemma}
\begin{proof}
The symmetry and reflexivity of $\sim$ are directly implied by the symmetry and the identity of the pseudo-metric $d$; there only remains to verify that $\sim$ is transitive.

\smallbreak
Let $x,y$ be elements of $E$ such that $x \sim y$, i.e., $d(x,y)= d(y,x) =0$, and let $z$ be in $E.$
By triangle inequality, we have on one hand $d(x,z) \leq d(x,y) + d(y,z) = d(y,z).$ 
On the other hand, we also have $d(y,z) \leq d(y,x) + d(x,z) = d(x,z)$, hence we indeed get $d(x,z) = d(y,z).$ 

\smallbreak
Applying this to $z \in E$ such that $y\sim z$ finally implies $d(x,z) =0$, i.e., $x\sim z$, and we verify that $\sim$ is transitive.
\end{proof}

We finally consider the \emph{quotient space} of $E$ by the equivalence relation $\sim$, that is the set $E^* = E / {\mathrel{\sim}}$ of all equivalence classes induced by $\sim$ on $E.$ We may now define the metric $d^* : ([x], [y]) \in E^*\times E^* \mapsto d(x,y) $, and refer to the metric space $(E^*,d^*)$ as the metric space induced by the vanishing  of the pseudo-metric space $(E,d).$

Most of the Axioms in Section~\ref{sec:axioms} directly extend to a pseudo-metric space, with the intuition that perfect clones are a particular example of approximate clones.
Two subtleties are however worth being mentioned. 
First, note that introducing perfect clones may break some of the symmetry classes described in Axiom~\ref{axi:sym} --akin to the effect of introducing approximate clones though. Specifically, $k\in \mathbb{N}$ perfect clones in a set $S$ must remain at zero distance from each other under any self-isometry $\sigma_S$, and can only be symmetric with $k$ other perfect clones.
Second, the set of minimum transport maps $\Pi(Y,X)$ between two subsets $Y$ and $X$ at distance $d_\Pi(X,Y) \leq \delta$ won't be a singleton in the presence of perfect clones. 
However, this can also happen without perfect clones, and Axiom~\ref{axi:indiv_cont} only requires the existence of a
$\pi$ in $\Pi(Y,X)$ such that $\max_{x\in X} \vert f(X)(x) - f(Y)(\pi^{-1}(x)) \vert \leq \varepsilon$ holds.

\paragraph{Undesirability of Hausdorff Norm}

We now explain why the commonly used \emph{Hausdorff distance} is not a good fit in our setting.

First, let us recall its definition. The \emph{Hausdorff distance} $d_H$ is a metric on  $\mathcal{P}(E)$, defined for two finite subsets $X, Y \subseteq E$ by
\begin{equation*}
    d_H(X,Y) := \max\big\{ \max_{x\in X} d(x,Y), \max_{y\in Y} d(X,y) \big\},
\end{equation*}
where $d(a,B) = \min_{b \in B} d(a,b)$ denotes the minimal distance between a point $a \in E$ and the finite set $B \subseteq E.$

We next show that $d_H$ and $d_\Pi$ are tightly related in the absence of perfect clones.
 Specifically, for any $X \in\mathcal{P}$, there exists a $\delta>0$ small enough such that the $\delta$-neighborhoods of $X$ with respect to $d_H$ coincides with that defined by $d_\Pi.$ 
 
Formally, recall the definition of the  \emph{inner diameter of $X$} $\underline{d}(X) := \min_{x \neq x' \in X} d(x,x')$, and we define the \emph{canonical projection on $X$} as the operator $\pi_X : E \mapsto \mathcal{P}(X)$ verifying $\pi_X(y) =\argmin_{x \in X} d(x,y)$ for all $y$ in $E.$ We further associate singleton $\{x'\} = \argmin_{x \in X} d(x,y)$ with $x' \in X$ and denote by $\restr{\pi_X}{Y}$ its restriction to a finite subset $Y\subseteq E.$

\begin{lemma}\label{lem:conv_hausdorff}
     Let $X$ be a finite subset of $E$ and $\delta$ satisfy $\underline{d}(X) /2 >\delta>0.$ 
     A finite subset $Y \in \mathcal{P}(E)$ is at distance $d_H(X,Y) \leq \delta$ if and only if the canonical projection $\restr{\pi_X}{Y}$ is the unique surjective map $\pi: Y \mapsto X$ such that $\max_{y \in Y} d(y,\pi(y)) \leq \delta$ holds.
\end{lemma}
\begin{proof}
Let $X$ be a finite subset in $\mathcal{P}(E)$ and let $\delta$ satisfy $$\max_{x,x' \in X} d(x,x') / 2 >\delta>0.$$

\smallskip
We first show the converse $\impliedby.$ Consider a finite subset $Y$ such that $\pi = \restr{\pi_X}{Y}$ is surjective and verifies $d(y,\pi(y)) \leq \delta.$
We directly have $d(X,y) \leq d(\pi(y), y) \leq \delta$ for all $y\in Y.$
Moreover for $x\in X$, the set $\pi^{-1}(x)$ is non-empty by surjectivity of $\pi$, and we similarly get $d(x,Y) \leq \min_{y \in \pi^{-1}(x)} d(\pi(y), y) \leq \delta.$
Taking the maximum over $Y$ and $X$ respectively, we finally establish that $\max_{y \in Y} d(X,y) \leq \delta$ and $\max_{x \in X} d(x,Y) \leq \delta$, which, combined with the definition of the Hausdorff norm, gives the desired result $d_H(X,Y) \leq \delta.$

\smallskip
We now turn to the direction $\implies.$ Let $Y$ be a finite subset in  $\mathcal{P}(E)$ that verifies $d_H(X,Y) \leq \delta.$ 
We first show that $\pi = \restr{\pi_X}{Y}$ is well-defined and surjective. 
Indeed, consider $x$ in $X$ and let $y$ be an element of $Y$ closest to $x.$ We then directly get $d(x,y) = d(x,Y) \leq d_H(X,Y) \leq \delta < \max_{z,z' \in X} d(z,z') /2.$
Moreover, let $x'$ be an element of $X$ different from $x$:
by triangle inequality, we get
$d(y, x') \geq d(x, x') - d(x,y) \geq \max_{z,z' \in X} d(z,z') - \delta  >  \delta.$
This implies that $x$ is the only element in $X$ at distance at most $\delta$ of $y$, hence we have $\pi(y) =x$ and $\pi$ is well-defined and surjective.

Moreover, any other choice of $\pi(y) = x' \neq x$ will break the property $\max_{y\in Y} d(y, \pi(y))$, and $\restr{\pi_X}{Y}$ is the only choice.

\end{proof}

Note that this implies $\Pi(Y,X) = \{\restr{\pi_X}{Y}\}$ and the $\delta$-neighborhood of $X$ according to $d_\Pi$ coincides with that of $d_H.$

Importantly though, this only holds for a choice of $\delta$ that depends on the subset $X$, and the topologies of $d_H$ and $d_\Pi$ are not uniformly similar. This problem becomes even more apparent when we introduce perfect clones, c.f., Figure~\ref{fig:d_H_perfect_clones}.

\begin{figure}[!tbh]
    \centering
    \begin{tikzpicture}[scale=0.6]
    \def\radius{0.7}
    \definecolor{PaperBlue}{HTML}{1f77b4}
    \definecolor{PaperRed}{HTML}{d62728}

    \coordinate (x1) at (0, 0);
    \coordinate (x2) at (0, 0);        
    \coordinate (x3) at (0, 3);
    \coordinate (x4) at (0, 3);        

    \node at (0, 2) {\(X = \{ \textcolor{PaperBlue}{x_1} , \textcolor{PaperRed}{ x_3}, \textcolor{PaperRed}{x_4} \}\)};
    \node at (0, 1) {\(X' = \{ \textcolor{PaperBlue}{x_1}, \textcolor{PaperBlue}{x_2}, \textcolor{PaperRed}{x_3} \}\)};

    \fill[PaperRed] (x1) circle (2pt);
    \node[PaperRed, below right] at (x1) {$x_3 \sim_d x_4$};

    \fill[PaperBlue] (x3) circle (2pt);
    \node[PaperBlue, above right] at (x3) {$x_1 \sim_d x_2$};

    \end{tikzpicture}
    \caption{Illustration of the difference of topologies between $d_H $and $d_\Pi$ in the presence of  perfect clones $x_1 \sim_d x_2$ and $x_3 \sim_d x_4.$ On one hand we have $d_H(X,X') =0$, on the other hand $d_\Pi(X,X') = \max_{y,z \in X} d(y,z) $ can be arbitrarily large.}
    \label{fig:d_H_perfect_clones}
\end{figure}

One might wonder whether Axiom~\ref{axi:indiv_cont} could be replaced by requiring that the weighting function $f$ be continuous from  $(\mathcal{P}(E), d_H)$ to $(\Delta_{\mathcal{P}}(E), W_1)$, even when $d$, and hence also $d_H$ and $W_1$, are pseudo-metrics -- that is, in the presence of perfect clones. We now show that any such weighting function would violate Axiom~\ref{axi:clone_fair_uni}.

\begin{lemma}
    Let $d$ be a pseudo-metric on $E$, and let $f$ be a weighting function that is continuous from  $(\mathcal{P}(E), d_H)$ to $(\Delta_{\mathcal{P}}(E), W_1)$, where $d_H$ and $W_1$ are defined with respect to $d$. If $f$ also satisfies Axiom~\ref{axi:sym}, then it necessarily violates Axiom~\ref{axi:clone_fair_uni}.
\end{lemma}
\begin{proof}
Let $(E,d)$ be a pseudo-metric space such that the metric space induced by the vanishing of the pseudo-metric is $(E^*, d^*) = (\mathbb{R}^3, d_2).$

We revisit the construction shown in Figure~\ref{fig:visualization_power_two}, now setting  $\gamma =0 $, i.e., we consider that $v_{\alpha,0}^+$ and $v_{\alpha,0}^-$ are perfect clones.
Accordingly, we examine the sets $S_{\alpha,\alpha,0} = \{o, u_\alpha, v^+_{\alpha, 0},  v^-_{\alpha, 0} \}$ and $S_\alpha = \{o, u_\alpha, u_{-\alpha}\}.$

On one hand, Axiom~\ref{axi:sym} implies that  $f(S_{\alpha,\alpha,0})(v^+_{\alpha, 0}) = f(S_{\alpha,\alpha,0})(v^-_{\alpha, 0})$ since $ v^+_{\alpha, 0}$ and $ v^-_{\alpha, 0}$ are symmetric in $S_{\alpha,\alpha,0}.$
Moreover, since $d_H(S_\alpha, S_{\alpha,\alpha,0}) =0$, the continuity of $f$ implies that $W_1(f(S_{\alpha,\alpha,0}),f(S_\alpha)) =0$, and we have in particular $f(S_\alpha)(u_{-\alpha}) = f(S_{\alpha,\alpha,0})(v^-_{\alpha, 0}) + f(S_{\alpha,\alpha,0})(v^-_{\alpha, 0})$, as well as $f(S_\alpha)(u_{\alpha}) = f(S_{\alpha,\alpha,0})(u_{\alpha}).$

On the other hand, the previous arguments still apply to $S_\alpha$ and imply  $\lim_{\alpha \to 0} f(S_\alpha)(u_\alpha) = \lim_{\alpha \to 0} f(S_\alpha)(u_{-\alpha}) = 1/4. $

Combining the two arguments, we finally get, i.e., 
\begin{equation*}
    \lim_{\alpha \to 0} f(S_{\alpha,\alpha,0})(u_\alpha) = \frac{1}{4} \quad \text{and} \quad  \lim_{\alpha \to 0} f(S_{\alpha,\alpha,0})(v^-_{\alpha, 0}) = \frac{1}{8}. 
\end{equation*}
This is a contradiction with Axiom~\ref{axi:clone_fair_uni} since $\lim_{\alpha \to 0} d(u_\alpha, v^-_{\alpha, 0} ) =0.$

\end{proof}

\section{Self-Isometries in Euclidean Space}\label{sec:self-iso_eucli}

In this section, we show that a self-isometry $\sigma_S$ on a finite subset $S \subseteq \mathbb{R}^n$ can be uplifted to a full-fledged isometry on $\mathbb{R}^n$, that is a rigid transformation.

\begin{lemma}
\label{lem:isometry_rigid_transformation}
Let $S=\{x_1, \dots, x_m\}$ be a finite subset of the Euclidean space $\mathbb{R}^n$ and  $\sigma_S:S\mapsto S$ be a self-isometry on $S.$
 There then exists an $n\times n$-orthogonal matrix $Q$ and an $n$-dimensional vector $t$ such that, for all $i\in[m]$, we have $\sigma(x_i) = Q x_i + t.$
\end{lemma}
\begin{proof}
For an index $2 \leq i \leq m$, we define $y_i = x_{i} - x_1$ and $z_i = \sigma(x_{i})  - \sigma(x_1).$ 
We then concatenate the $y_i$ (resp. $z_i$) and define the $n\times(m-1)$ matrix $Y = [y_2, \dots, y_{n}]$ (resp. $Z = [z_2, \dots, z_{m}]$).
For $i,j \in [m-1]$, note that the following holds:
\begin{equation*}
    \begin{aligned}
        (Y^\top Y)_{i,j} 
        &= y_{i+1} \cdot y_{j+1} ,\\
        &= \frac{1}{2} \Big(\Vert y_{i+1} + y_{j+1}\Vert^2 - \Vert y_{i+1} \Vert^2 - \Vert y_{j+1} \Vert^2 \Big),\\
        &= \frac{1}{2} \Big( d(x_{i+1}, x_{j+1})^2 - d(x_{i+1}, x_1)^2 
        - d(x_{j+1}, x_1)^2 \Big),  \\
        &\stackrel{(a)}{=} \frac{1}{2} \Big( d\big(\sigma(x_{i+1}), \sigma(x_{j+1})\big)^2 - d\big(\sigma(x_{i+1}), \sigma(x_1)\big)^2 ,\\
        &- d\big(\sigma(x_{j+1}),\sigma( x_1)\big)^2 \Big) ,\\
        &= z_{i+1} \cdot z_{j+1} = (Z^\top Z)_{i,j}.
    \end{aligned}
\end{equation*}
Equality $(a)$ holds since $\sigma$ is an isometry on $S$.

\smallbreak
By \cite[Theorem 3.7.11]{horn_matrix_2012}, there exists an orthogonal $n\times n$ matrix $Q$ such that $Z=Q Y$, and we obtain for all $i \in[m]$ that $\sigma(x_{i}) = Q (x_{i} - x_1) +  \sigma(x_1)$ (the case $i=1$ holds trivially). 
Rewriting $t =  \sigma(x_1) - Q x_1 $ gives the desired result.

\end{proof}

Note moreover that the Euclidean group $E(n)$, i.e., the group of isometries in Euclidean space,  is exactly the semi-direct product of the orthogonal group $O(n)$ extended by the translational group $T(n).$ In other words, Lemma~\ref{lem:isometry_rigid_transformation} ensures that all self-isometries $\sigma_S: S \mapsto S$ on a finite subset $S \in \mathbb{R}^n$ can be extended to a full-fledged isometry $\sigma:\mathbb{R}^n \mapsto \mathbb{R}^n.$

\section{Comparison with Voronoi Weighing Function}\label{sec:voronoi}

In \cite{procaccia_clone-robust_2025}, the authors propose to weigh elements of a set $S$ according to their cell's size in an associated Voronoi diagram.

Formally, let $E$ be a subset of $\mathbb{R}^n$ that is Lebesgue-measurable with finite measure $\mu(E)<\infty$, and let $S$ be a finite subset of $E.$ 
The Voronoi weighing function is then defined as 
\begin{equation}
    V(S): x\in S \mapsto \int_{z\in E} \frac{\mathbb{1}_{x \in \pi_S(z)}}{\mu(E)} d\mu(z).
\end{equation}

Note that this definition depends heavily on the arbitrary choice of integration subspace $E$, which introduces more subjectivity than the single parameter $r>0$ used to define the integration domain $B_r(S)$ in our method.

Which of the Axioms in Section~\ref{sec:axioms} does the weighting function $V$ verify?

\begin{figure*}[!tbh]
     \centering
     \begin{subfigure}[t]{0.45\textwidth}
         \centering
         \begin{tikzpicture}[scale=4]

        \draw[thick] (0,0) rectangle (1,1);
        
        \fill[blue!30] (0,0) -- (0,1) -- (0.5,0.5) -- (0.5,0) -- cycle;      
        \fill[green!30] (0.5,0) -- (1,0) -- (1,0.5) -- (0.5,0.5) -- cycle;   
        \fill[red!30] (0,1) -- (1,1) -- (1,0.5) -- (0.5,0.5) -- cycle;       
        
        \draw[thick] (0.5,0) -- (0.5,0.5);
        \draw[thick] (0.5,0.5) -- (1,0.5);
        \draw[thick] (0.5,0.5) -- (0,1);
        
        \foreach \x/\y in {0/0, 1/0, 1/1}
            \fill[red] (\x,\y) circle (0.015);
        
        \foreach \x in {0,0.2,0.4,0.6,0.8,1}
            \node[below] at (\x,0) {\small \x};
        
        \foreach \y in {0,0.2,0.4,0.6,0.8,1}
            \node[left] at (0,\y) {\small \y};
        
        \end{tikzpicture}
        \caption{Diagram for $S = \{(0,0), (1,0), (1,1)\}$.}
        \label{fig:voronoi_procaccia_1}
     \end{subfigure}
     \hfill
     \begin{subfigure}[t]{0.45\textwidth}
         \centering
        \begin{tikzpicture}[scale=4]
        
        \draw[thick] (0,0) rectangle (1,1);
        
        \fill[blue!30] (0,0) -- (0,0.92) -- (0.5,0.45) -- (0.5,0) -- cycle;         
        \fill[green!30] (0.5,0) -- (1,0) -- (1,0.5) -- (0.5,0.45) -- cycle;     
        \fill[orange!30] (0, 0.92) -- (0, 1)-- (0.95,1) -- (0.95, 0.5) -- (0.5,0.45) -- cycle;   
        \fill[red!30] (0.95,1) -- (1,1) -- (1,0.5) -- (0.95,0.5) -- cycle;         
        
        \draw[thick] (0.5,0) -- (0.5,0.45);
        \draw[thick] (0,0.92) -- (0.5,0.45);
        \draw[thick] (1,0.5) -- (0.95,0.5);
        \draw[thick] (0.95,0.5) -- (0.95,1);
        \draw[thick] (0.95,0.5) -- (0.5,0.45);
        
        \foreach \x/\y in {0/0, 1/0, 1/1, 0.9/1}
            \fill[red] (\x,\y) circle (0.015);
        
        \foreach \x in {0,0.2,0.4,0.6,0.8,1}
            \node[below] at (\x,0) {\small \x};
        
        \foreach \y in {0,0.2,0.4,0.6,0.8,1}
            \node[left] at (0,\y) {\small \y};
        
        \end{tikzpicture}
        \caption{Diagram for $S \cup \{(0.9,1)\}$.}
        \label{fig:voronoi_procaccia_2}
     \end{subfigure}
     
        \caption{Effect of adding approximate clones to  Voronoi diagrams with $E = [0,1] \times[0,1].$ Illustrations from \protect\cite{procaccia_clone-robust_2025}.}
        \label{fig:voronoi_procaccia}
\end{figure*}

\begin{figure*}[!tbh]
     \centering
     \begin{subfigure}[t]{0.45\textwidth}
         \centering
         \begin{tikzpicture}[scale=4]

            \definecolor{PaperBlue}{HTML}{1f77b4}
            \definecolor{PaperOrange}{HTML}{ff7f0e}
            \definecolor{PaperGreen}{HTML}{2ca02c}
            \definecolor{PaperRed}{HTML}{d62728}

            \draw[thick] (0,0) rectangle (1,1);
    
            \fill[PaperBlue, opacity =0.2] (0,0) -- (0,1) -- (0.45,1) -- (0.45,0) -- cycle;      
            \fill[PaperRed, opacity = 0.2] (0.45,0) -- (0.45,1) -- (0.55,1) -- (0.55,0) -- cycle;   
            \fill[PaperGreen, opacity =0.2] (0.55,1) -- (0.55,0) -- (1,0) -- (1,1) -- cycle;       

            \foreach \x/\y in {0.4/0, 0.5/0, 0.6/0}
                \fill[red] (\x,\y) circle (0.015);
        
        \draw[thick] (0.45,0) -- (0.45,1);
        \draw[thick] (0.55,0) -- (0.55,1);
        
        \foreach \x in {0,0.2,0.4,0.6,0.8,1}
            \node[below] at (\x,0) {\small \x};
        
        \foreach \y in {0,0.2,0.4,0.6,0.8,1}
            \node[left] at (0,\y) {\small \y};
        
        \end{tikzpicture}
        \caption{Diagram for the parametric set $S_\delta = \{ (0.5,0), (0.5 + \delta, 0), (0.5 - \delta, 0) \}.$}
        \label{fig:voronoi_non_unif_1}
     \end{subfigure}
     \hfill
     \begin{subfigure}[t]{0.45\textwidth}
         \centering
        \begin{tikzpicture}[scale=4]
        
            \definecolor{PaperBlue}{HTML}{1f77b4}
            \definecolor{PaperOrange}{HTML}{ff7f0e}
            \definecolor{PaperGreen}{HTML}{2ca02c}
            \definecolor{PaperRed}{HTML}{d62728}

            \draw[thick] (0,0) rectangle (1,1);

            \fill[PaperBlue, opacity =0.2] (0,0) -- (0,0.5) -- (0.45,0.05) -- (0.45,0) -- cycle;      
            \fill[PaperRed, opacity = 0.2] (0.45,0) -- (0.45,0.05) -- (0.55,0.05) -- (0.55,0) -- cycle;   
            \fill[PaperGreen, opacity =0.2] (0.55,0) -- (0.55,0.05) -- (1,0.5) -- (1,0) -- cycle;       
            \fill[PaperOrange, opacity =0.2] (0,1) -- (0,0.5) -- (0.45,0.05) -- (0.55,0.05) -- (1,0.5) -- (1,1) -- cycle;      
            
            \foreach \x/\y in {0.4/0, 0.5/0, 0.6/0, 0.5/0.1}
                \fill[red] (\x,\y) circle (0.015);
            
            \draw[thick] (0.45,0) -- (0.45,0.05);
            \draw[thick] (0.55,0) -- (0.55,0.05);
            \draw[thick] (0.45,0.05) -- (0.55,0.05);
            \draw[thick] (0,0.5) -- (0.45,0.05);
            \draw[thick] (1,0.5) -- (0.55,0.05);
            
            \foreach \x in {0,0.2,0.4,0.6,0.8,1}
                \node[below] at (\x,0) {\small \x};
            
            \foreach \y in {0,0.2,0.4,0.6,0.8,1}
                \node[left] at (0,\y) {\small \y};
            
        \end{tikzpicture}
        \caption{Diagram for $S'_\delta = S_\delta \cup \{(0.5, \delta) \}.$}
        \label{fig:voronoi_non_unif_2}
     \end{subfigure}
     
        \caption{While the continuity of $V$ breaks exactly in the presence of perfect clones,
        it remains non-uniform even when limited to approximate clones: for any $\delta>0$, adding a $\delta$-clone to a set $S_\delta$ can cause individual weights to diverge.}
        \label{fig:voronoi_non_uniform}
\end{figure*}

\begin{itemize}
    \item While $V$ verifies Axiom~\ref{axi:pos}, Figure~\ref{fig:voronoi_procaccia_2} shows that an element may receive arbitrarily small weight. This is to be contrasted with our result in the proof of Theorem~\ref{thm:local_vote_rep_func} (c.f. Section~\ref{sec:proof}, where we show that the individual weights are all greater or equal to $1/\vert S \vert^2.$

    \item The Voronoi weighting function does not satisfy Axiom~\ref{axi:sym} per say, although it should be noted that, for a self-isometry $\sigma: E\mapsto E$ on the entire space and any finite subset $S\subseteq E$, we  have $V(S)(x) = V(S)(\sigma(x)) $ for all $x \in S.$
    Whether this weakening of Axiom~\ref{axi:sym} is sufficient is left to the discretion of the practitioner. 
    It is important to note, however, that this property is particularly fragile: the definition of $V$ depends on an arbitrary choice of the ambient space $E$, and the class of self-isometries --and thus the symmetry behavior of $V$-- varies significantly with this choice.

    \item Figure~\ref{fig:voronoi_procaccia_2} provides a direct illustration that the weighting function $V$ violates Axiom~\ref{axi:clone_fair_uni}.

    \item While Voronoi cells vary continuously as long as approximate clones remain separated by a fixed minimal distance, this continuity breaks down near perfect clones --discontinuous jumps in individual weights occur as soon as approximate clones converge to being exact duplicates (see once again Figure~\ref{fig:voronoi_procaccia}).
    Even without perfect clones, continuity is not uniform (c.f. Figure~\ref{fig:voronoi_non_uniform}), and the Voronoi weighting function $V$ fails Axiom~\ref{axi:indiv_cont}.
    
    \item Finally, the Voronoi weighting function seems to satisfy Axiom~\ref{axi:alpha_clone_locality}, although the extent to which it does so uniformly remains unclear.
\end{itemize}

\end{document}